%% file: neurips_2026.tex
\documentclass{article}

\PassOptionsToPackage{numbers,sort&compress}{natbib}

\usepackage[preprint]{neurips_2026}

\usepackage{fontawesome}
\usepackage{float}
\usepackage[utf8]{inputenc}
\usepackage{mathtools}
\usepackage[T1]{fontenc}
\usepackage{hyperref}
\usepackage{url}
\usepackage{booktabs}
\usepackage{amsmath}
\usepackage{amssymb}
\usepackage{amsthm}
\usepackage{amsfonts}
\usepackage{algorithm}
\usepackage{algpseudocode}
\usepackage{nicefrac}
\usepackage{microtype}
\usepackage{xcolor}
\usepackage{graphicx}
\usepackage{cleveref}
\usepackage{array}
\usepackage{multirow}
\usepackage[most]{tcolorbox}
\usepackage{caption}
\usepackage{subcaption}
\usepackage{wrapfig}
\usepackage{thmtools,thm-restate}
\usepackage{booktabs}
\usepackage{tabularx}
\usepackage[table]{xcolor}
\usepackage{makecell}
\usepackage{tikz}
\usepackage{wrapfig}
\usepackage{adjustbox}

\usepackage{tikz}
\usepackage{pgfplots}
\pgfplotsset{compat=1.18}

\definecolor{espurple}{HTML}{8E44AD}
\definecolor{rougeorange}{HTML}{D95F02}
\definecolor{parablue}{HTML}{1F77B4}

\definecolor{takeawayblue}{HTML}{1F4E79}
\definecolor{takeawaylightblue}{HTML}{EEF6FC}
\definecolor{takeawayborder}{HTML}{8DB7D6}

\newtcolorbox{takeawaybox}[1][]{
  enhanced,
  colback=takeawaylightblue,
  colframe=takeawayblue,
  boxrule=0.65pt,
  arc=2.5pt,
  left=6pt,
  right=6pt,
  top=4pt,
  bottom=4pt,
  before skip=1.7em,
  after skip=0.65em,
  fontupper=\normalsize,
  overlay={
    \node[
      anchor=south west,
      fill=takeawayblue,
      text=white,
      rounded corners=1.8pt,
      inner xsep=6.5pt,
      inner ysep=2.6pt,
      font=\bfseries\normalsize
    ] at ([xshift=6pt,yshift=-4pt]frame.north west) {#1};
  }
}

\newtcolorbox{compactquestionbox}[1][]{
  enhanced,
  colback=white,
  colframe=black!45,
  boxrule=0.45pt,
  arc=2.5pt,
  left=6pt,
  right=6pt,
  top=4pt,
  bottom=4pt,
  before skip=0.45em,
  after skip=0.45em,
  width=0.95\columnwidth,
  center,
  fontupper=\normalsize,
  #1
}

\PassOptionsToPackage{numbers}{natbib}

\title{Fast Unlearning at Scale via Margin Self-Correction}

\author{%
Federico Di Gennaro\thanks{Equal contribution.} \ \thanks{Corresponding author: \texttt{figennaro@ethz.ch}}\\ ETH Zürich 
\And Alexander Shevchenko\footnotemark[1]\\ ETH Zürich 
\And Fanny Yang\\ ETH Zürich }

\newtheorem{definition}{Definition}
\newcommand{\Dforget}{\mathcal{D}_{\mathrm{fg}}}
\newcommand{\Dretain}{\mathcal{D}_{\mathrm{ret}}}
\newcommand{\Dval}{\mathcal{D}_{\mathrm{val}}}

\newcommand{\KL}{\mathrm{KL}}

\newcommand{\one}{\mathbf{1}}

\newcommand{\Lfg}{\mathcal{L}_{\mathrm{fg}}}
\newcommand{\Lret}{\mathcal{L}_{\mathrm{ret}}}

\usepackage{makecell}

\newcommand{\datasetsep}{\specialrule{1.2pt}{1.5pt}{1.5pt}}

\usepackage{minitoc}
\usepackage{etoc}

\doparttoc

\begin{document}

\faketableofcontents

\maketitle

\begin{abstract}
Language-model unlearning updates a trained model to behave as if it had not seen selected training examples, while preserving utility and avoiding costly retraining. Existing approaches typically fine-tune the pretrained model with a fixed training budget and select the final model \emph{afterwards} by evaluating several saved checkpoints on downstream validation data. Two sources of unnecessary computation limit scalability: training beyond the desired forget--retain trade-off, and checkpoint selection that requires extra storage and repeated evaluations.
To address these limitations, we introduce \emph{\underline{MA}rgin \underline{S}elf-\underline{C}orrection} (\underline{MASC}), an efficient unlearning method with an \emph{online} stopping rule that does not require downstream evaluation. Given a text sequence to be forgotten, MASC actively reduces the logit gap between the original next token and the most likely alternatives. It outputs a 
final model once this gap is small on average over a sufficiently large proportion of token positions across all forget sequences. On \texttt{TOFU}, \texttt{MUSE News}, and \texttt{MUSE Books}, MASC achieves a competitive forget--retain trade-off at a fraction of the computational cost of existing baselines. We further observe that as we increase model size (a.k.a. number of parameters), the trade-offs improve for both MASC and SimNPO -- the forget metrics remain comparable while retain utility increases.
\end{abstract}
\vspace{-0.2in}
\section{Introduction}
\label{sec: introduction}
\vspace{-0.1in}
Despite their remarkable success in tasks including code generation \cite{nam2024using,chen2021evaluating}, mathematical reasoning \cite{lewkowycz2022solving}, and scientific discovery \cite{boiko2023autonomous}, Large Language Models (LLMs) are prone to memorizing sensitive training data \cite{carlini2022quantifying}, including private information \cite{meeus2024did,staab2024beyond} and copyrighted content \cite{karamolegkou2023copyright}. This tendency poses significant safety and privacy risks, particularly as LLMs are nowadays increasingly deployed in high-stakes domains \cite{lee2020biobert,wu2023bloomberggpt}. These concerns are also reflected in legal frameworks such as the \emph{California Consumer Privacy Act} (CCPA) \cite{ccpa2018} and the European \emph{General Data Protection Regulation} (GDPR) \cite{hoofnagle2019european}, which establish rights to request the deletion of personal data, often referred to as the \emph{right to be forgotten}.

Machine unlearning \cite{bourtoule2021machine,cao2015towards,liu2025rethinking} provides a computational framework for such a goal. Given a trained model and a collection of examples to be forgotten, the aim is to return a new model that behaves as if those examples had never been used for training, while preserving its performance on the rest of the data. The gold standard is exact retraining, where the model is trained from scratch after removing the forgotten examples from the training corpus. While retraining gives the desired behavior \emph{by definition}, it is computationally prohibitive for modern language models. This motivates \emph{approximate} unlearning methods which fine-tune the existing model. 
Approximate unlearning, however, introduces a delicate forget--retain trade-off: weak procedures may leave the target content reproducible, whereas overly aggressive interventions may damage performance on unrelated data, reminiscent of the broader problem of catastrophic forgetting \cite{liu2025rethinking,luo2025empirical,mccloskey1989catastrophic,kirkpatrick2017overcoming}.
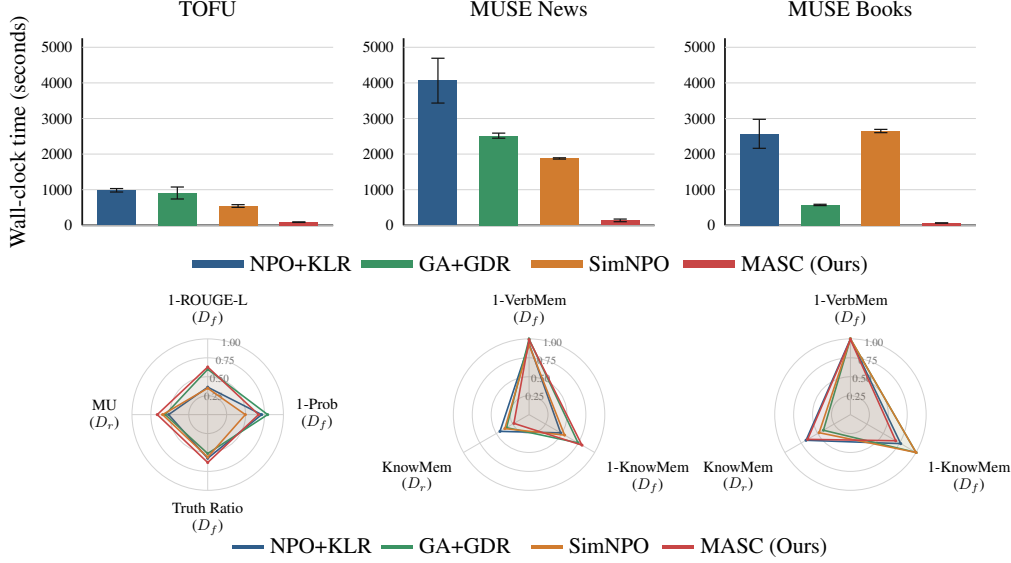
\begin{figure}[t]
\centering
\resizebox{\linewidth}{!}{\input{plots/runtimes.pgf}} \caption{\textbf{Top:} Wall-clock runtimes (in seconds) of methods with similar retain--forget trade-off. \textbf{Bottom}: Forget--retain trade-offs of timed methods. Each metric is in $[0,1]$ (the higher the better), and MASC is competitive (i.e. not Pareto dominated) with the others.} \vspace{-0.2in}
\label{fig:runtimes}
\end{figure}
Existing methods -- including Gradient Ascent (GA) \cite{yao2024large}, NPO \cite{zhang2024negative,fan2026simplicity}, and their regularized variants \cite{chen2023unlearn,maini2024tofu,shi2024muse} -- typically lack an \emph{online} (i.e., during training) model selection rule to identify the optimal (or desired) forget--retain trade-off without relying on costly downstream evaluations. Instead, these algorithms generally run for a predefined and fixed compute budget, which is both inefficient (cf. \Cref{fig:runtimes}) and agnostic to the actual unlearning dynamics. Consequently, practitioners are forced to select a final model retrospectively, evaluating all saved checkpoints only after training is complete. This leads to our first research question:

\begin{compactquestionbox}
\centering
\emph{\textbf{(Q1)} Can we design an efficient unlearning objective with an intrinsic stopping rule that offers a controllable stopping criterion for the forget–-retain trade-off?}
\end{compactquestionbox}

We introduce \textbf{MArgin Self-Correction (MASC)}, an unlearning method whose objective naturally admits an \emph{adaptive} stopping rule. MASC performs gradient updates on a retain-regularized loss that discourages drift from the original model on retain data, while correcting forget-set token predictions that remain too dominant. Each forget continuation is evaluated under teacher forcing: at every token position, MASC computes a \emph{restricted margin}, defined as the logit gap between the original next token (also called \emph{target token}) and a log-sum-exp aggregate of the logits of the model's top-\(k\) alternatives to the original next token. This margin measures how strongly the model still prefers the forgotten token over plausible replacements. MASC then returns the first checkpoint at which the margin condition is satisfied on a sufficiently large fraction of monitored forget-set tokens. We show that this token-level condition theoretically upper-bounds the probability of exactly reproducing the forgotten continuation (cf. \Cref{prop:masc-reproduction}). Thus, the returned checkpoint is selected online using the same condition optimized during unlearning, rather than by running a fixed training budget followed by downstream checkpoint evaluation. Empirically, this yields competitive forget--retain trade-offs at a fraction of the computational cost of existing baselines (see \Cref{fig:runtimes}).

This efficiency advantage is especially relevant at scale, where unlearning costs during both finetuning and evaluation can grow quickly with model size. Beyond computation, however, scale may also affect the forget--retain behavior itself. While prior work has studied how unlearning performance changes with the size of the deletion request \cite{maini2024tofu,shi2024muse}, the role of \emph{model size} remains under-explored. Larger models may internalize target information more strongly during supervised finetuning \cite{morris2025much,carlini2022quantifying,lu2024scaling}, and respond differently when that information is later removed. The second question we aim to answer is therefore: 

\begin{compactquestionbox}
\centering
\emph{\textbf{(Q2)} How does model scale influence knowledge acquisition during learning and its subsequent removal during unlearning?}
\end{compactquestionbox}

In this analysis, we distinguish between two levels of memorization: \emph{exact memorization} \cite{carlini2022quantifying,morris2025much,lu2024scaling}, where the model reproduces target content verbatim, and \emph{knowledge memorization}, where the model recovers the same underlying information under paraphrased prompts. During supervised finetuning, both metrics grow with model size and follow empirical power-law trends in log--log space, with a larger fitted exponent for exact memorization. This indicates that scale amplifies verbatim reproduction more strongly than paraphrase-based recovery. After unlearning, however, forget-side metrics become roughly stable across model sizes, while retain utility increases. This suggests that scale mainly improves the utility side of the post-unlearning trade-off, rather than systematically augmenting residual memorization of the forgotten content.

To summarize, our main contributions are:
\begin{itemize}
    \item We introduce \textbf{MASC}, an efficient unlearning method that suppresses target tokens only when they remain much more likely than an aggregate of the top-$k$ most likely alternatives. We demonstrate on \texttt{TOFU}, \texttt{MUSE News}, and \texttt{MUSE Books} that MASC achieves competitive trade-offs with substantially shorter wall-clock runtime.
    \item We provide a scaling study across the Qwen2.5 family, examining how scale affects different forms of memorization and how it benefits the final forget--retain frontier after unlearning (for both MASC and SimNPO \cite{fan2026simplicity}).
\end{itemize}

\textbf{Datasets.}
We evaluate MASC on three standard LLM unlearning benchmarks: \texttt{TOFU} \cite{maini2024tofu}, \texttt{MUSE News}, and \texttt{MUSE Books} \cite{shi2024muse}. \texttt{TOFU} is a synthetic question-answering benchmark based on fictitious biographies. We use its forget10/retain90 split, where 10\% of examples are assigned to the forget set and the remaining 90\% to the retain set. \texttt{MUSE} provides a more realistic setting based on memorized text from news articles (BBC) and books (Harry Potter series).\footnote{Code available at \href{https://github.com/FedericoDiGennaro/Fast-LLM-Unlearning-MarginSelfCorrection}{\color{orange}{\faGithub \ \url{FedericoDiGennaro/Fast-LLM-Unlearning-MarginSelfCorrection}}}}

\textbf{Notation.} For a finite set $\mathcal{S}$, we denote by \(\Delta(\mathcal{S})=\{p\in\mathbb{R}^{\mathcal{S}}_{+}:\sum_{s\in \mathcal{S}}p_s=1\}\) the probability simplex over it. If \(\mathcal{S}\subseteq\mathbb{R}^d\) and \(r\in\mathbb{N}\) is positive, we write \(S^r\) for the \(r\)-fold Cartesian product of \(\mathcal{S}\). Finally, for \(x\in\mathbb{R}\), we use \([x]_+=\max\{x,0\}\) to denote the positive part of \(x\). For an integer $T\in\mathbb N$, we denote $[T]$ as the set $\{1,...,T\}$.

\section{LLM unlearning and prior work}
\label{sec:LLM-unlearning}
This section introduces the notation for LLM unlearning and provides a non-exhaustive overview (see Appendix \ref{app:additional-related-work} for a more detailed discussion) of well-known unlearning methods that we later use as baselines.

Let \(\mathcal V\) denote the token vocabulary, and let \(\Delta(\mathcal V)\) be the probability simplex over \(\mathcal V\). Then, let \(\mathcal C=\bigcup_{\ell\geq 0}\mathcal V^\ell\) denote the set of finite token contexts. An autoregressive language model with parameters \(\theta\in\mathbb R^d\) is defined as a policy \(\pi_\theta:\mathcal C\to\Delta(\mathcal V)\), where \(\pi_\theta(\cdot\mid c) =\mathrm{softmax}(z_\theta(\cdot\mid c))\) is the next-token distribution over $\mathcal{V}$ given context \(c\in\mathcal{C}\), and \(z_\theta(\cdot\mid c)\) denotes the corresponding \emph{logits}. Given a sample consisting of a prompt \(x\in\mathcal X\) and a continuation \(y=(y_1,\dots,y_T)\in\mathcal V^T\), we evaluate a policy on \((x,y)\) using the probability it assigns to the full continuation, which factorizes as \(\pi_\theta(y\mid x)=\prod_{t=1}^{T}\pi_\theta(y_t\mid c_t)\), where \(c_t=(x,y_{<t})\). This corresponds to \emph{teacher-forced} evaluation: each next-token distribution is conditioned on the reference prefix \(c_t\), rather than on tokens sampled from the model. Evaluating \(\pi_\theta(y\mid x)\) therefore only requires standard forward passes on the given prompt-continuation pair.

Let \(\Dforget\) denote the data\footnote{With a slight abuse of notation, we also use \(\Dforget\) and \(\Dretain\) to denote the empirical distributions obtained by sampling uniformly from the corresponding finite datasets.} to be forgotten, and \(\Dretain\) the data on which the model's behavior should be preserved. An LLM unlearning algorithm \(\mathcal A\) takes as input the weights \(\theta_0\in\mathbb{R}^d\) of a pretrained model together with \(\Dforget\) (and usually also $\Dretain$), and returns updated weights \(\theta_{\mathrm{unl}}\in\mathbb{R}^d\). The goal is for the resulting policy \(\pi_{\theta_{\mathrm{unl}}}\) to behave as if \(\Dforget\) had not been used for training, while preserving performance on \(\Dretain\). The key computational challenge is to avoid retraining the model from scratch. A common approach is to minimize a \emph{forget} loss \(\Lfg(\theta;\Dforget)\). Different unlearning methods correspond to different choices of such a loss, often combined with additional regularization to preserve retain-set behavior. Arguably, the most natural choice for \(\Lfg\) is
\begin{equation}
\label{eq:ga-loss}
\Lfg^{\mathrm{GA}}(\theta;\Dforget)
=
\mathbb{E}_{(x,y)\sim \Dforget}
\left[
\log \pi_\theta(y\mid x)
\right],
\end{equation}
which penalizes policies that assign a high probability to the forget continuations. Equivalently, since standard language-model pretraining minimizes next-token cross-entropy, minimizing \(\Lfg^{\mathrm{GA}}\) by gradient descent performs gradient ascent on the original cross-entropy objective restricted to \(\Dforget\). The resulting update \emph{reverses} likelihood-based learning on the forget data, and is therefore commonly referred to as Gradient Ascent (GA) unlearning. However, GA provides no intrinsic mechanism for stopping this likelihood decrease: continued optimization can keep lowering the probability of the forget continuation and may quickly degrade the model's behavior beyond the forget set. Negative Preference Optimization (NPO) \cite{zhang2024negative} addresses this issue by replacing direct likelihood minimization with a bounded preference-style objective \cite{rafailov2023direct}. Rather than indefinitely pushing down the likelihood of the forget-set continuations, NPO treats it as a negative preference example relative to the original model. In particular,
\begin{equation}
\label{eq:npo-loss}
\Lfg^{\mathrm{NPO}}(\theta;\Dforget)
    =
    -
    \frac{2}{\beta}
    \mathbb{E}_{(x,y)\sim \Dforget}
    \left[
        \log \sigma\!\left(
            -\beta \log \frac{\pi_\theta(y\mid x)}{\pi_{\theta_0}(y\mid x)}
        \right)
    \right],
\end{equation}
where \(\beta>0\) is an inverse-temperature parameter and \(\sigma(u)=(1+e^{-u})^{-1}\) is the sigmoid function. Unlike GA, NPO weakens the forget update once the forget continuation is already much less likely under the current model than under the original one. Indeed, if \(r_\theta=\log \frac{\pi_\theta(y\mid x)}{\pi_{\theta_0}(y\mid x)}\), then the gradient scales as \(\sigma(\beta r_\theta)\), which vanishes for large negative \(r_\theta\). However, because NPO compares the current likelihood to the original model likelihood \(\pi_{\theta_0}(y\mid x)\), the magnitude of the forget update depends on the reference-model score of each example, and can therefore vary with sequence length or reference likelihood. SimNPO \cite{fan2026simplicity} removes this dependence by using a reference-free, length-normalized variant of the NPO objective. Although \Cref{eq:ga-loss,eq:npo-loss} define common forget-side objectives, practical unlearning methods typically combine them with a retain regularizer to improve the forget--retain trade-off. This leads to objectives of the form
\begin{equation}
\label{eq:joint}
    \min_{\theta}
    \quad
    \Lfg(\theta;\Dforget)
    +
    \lambda_{\mathrm{ret}}
    \Lret(\theta;\Dretain,\theta_0),
\end{equation}
where \(\Lfg\) encourages suppression of the forget data, while \(\Lret\) (computed on \(\Dretain\)) discourages unnecessary drift from the original model \(\pi_{\theta_0}\). The retain term is typically implemented as a KL penalty relative to \(\pi_{\theta_0}\) or as a cross-entropy loss on retained examples.

The main limitation of the above forget losses is that they only specify what should be made less likely, not what the model should do instead. The probability \(\pi_\theta(y\mid x)\) can be decreased by concentrating mass on a few alternative continuations, spreading mass broadly, or degrading the next-token distribution more generally. Since these sequence-level objectives do not specify what should happen at each next-token prediction, they also provide no direct criterion for deciding when an original forget-set token has become sufficiently non-dominant relative to its alternatives. Further, although these objectives provide strong and widely used baselines, the computational cost of unlearning can be substantial. In practice, one must either fix the unlearning budget in advance (i.e., the number of finetuning epochs), or periodically evaluate intermediate models to decide which one should be returned. The latter requires downstream forget--retain validation data, since the checkpoint is selected using external metrics rather than a condition monitored during training. Saving many checkpoints and selecting among them after training further adds storage and evaluation overhead. To address these limitations, we introduce MASC, an unlearning objective based on a token-level dominance condition: a forget token should no longer dominate the model's top-\(k\) non-target alternatives under the same reference prefix. This condition defines both the forget loss and the online stopping rule: MASC stops once it is satisfied on a sufficiently large fraction of monitored forget-set tokens, allowing the returned checkpoint to be selected during training without downstream validation.

\section{MASC: MArgin Self-Correction}
\label{sec:method}

This section introduces MASC (MArgin Self-Correction) and derives its objective from first principles. MASC is based on a simple observation: exact reproduction of a forget sequence requires many positions at which the model assigns high probability to the true next token when evaluated under \emph{teacher forcing}, i.e., when conditioned on the true prefix. We now turn this observation into a token-level comparison that will define both the forget loss and the stopping rule.

\subsection{Token dominance and margins}
We first define token-dominance measures, which we use in our unlearning algorithm. Intuitively, on the forget set, we want to lower the probability of the true next-token continuation given the true prefix, while preserving overall utility. Our approach therefore reduces the dominance of the highest-probability token relative to its nearest alternatives, without substantially altering the rest of the distribution.
\paragraph{Restricted token comparison.}
Let \((x,y)\sim\Dforget\) be a forget example, where \(x\in\mathcal X\) and \(y=(y_1,\ldots,y_T)\in\mathcal V^T\). For each position \(t\in[T]\) of the sequence, the teacher-forced context is \(c_t=(x,y_{<t})\) and \(y_t\) is called \emph{target} token. Recall that $\pi_\theta(\cdot\mid c_t)$ denotes the next-token distribution over $\mathcal{V}$, and \(z_\theta(v\mid c_t)\) is the logit of each token \(v\in\mathcal V\). Rather than comparing \(y_t\) to the full vocabulary, we focus on the set of the model's top-\(k\) non-target alternative tokens, denoted as
\begin{equation}
\mathcal{S}_{\theta,k}(c_t) =
\arg\max_{\substack{\mathcal{S}\subseteq\mathcal V\setminus\{y_t\}\\ |\mathcal{S}|=k}}
\sum_{v\in \mathcal{S}}\pi_\theta(v\mid c_t),
\label{eq:topk-set}
\end{equation}
with ties broken arbitrarily.\footnote{We do not differentiate through the top-\(k\) operation; gradients are taken only through the logits appearing in the loss.}
For \(\beta>0\), further define the \emph{restricted} probability\footnote{For \(\beta=1\), this is exactly the probability obtained by restricting \(\pi_\theta(\cdot\mid c_t)\) to \(\{y_t\}\cup \mathcal{S}_{\theta,k}(c_t)\) and renormalizing.} of the target token $y_t$ 
\begin{equation}
    \pi_{\theta}^{(k,\beta)}(y_t\mid c_t)
    =
    \frac{\exp(\beta z_\theta(y_t\mid c_t))}
    {\exp(\beta z_\theta(y_t\mid c_t))
    +
    \sum_{v\in \mathcal{S}_{\theta,k}(c_t)}
    \exp(\beta z_\theta(v\mid c_t))}.
    \label{eq:restricted-probability}
\end{equation}
This restricted probability can be interpreted as a measure of local dominance and gives rise to a natural constraint
\begin{definition}
\label{def:restriced-suppression}
    For a threshold \(\rho\in(0,1)\), we say that the target token $y_t$ is locally suppressed in context \(c_t\) if $\pi_{\theta}^{(k,\beta)}(y_t\mid c_t)\leq \rho $.
\end{definition}
MASC uses this local-dominance measure as a constraint and hence implicitly asks: is the forget token still preferred over plausible replacements? Proposition~\ref{prop:masc-reproduction} shows that controlling this \emph{local dominance} on many positions is enough to control exact reproduction of the whole continuation.

\begin{restatable}{proposition}{propboundexactrepro}
\label{prop:masc-reproduction}
Consider a forget prompt $x$ and the corresponding continuation \(y=(y_1,\ldots,y_T)\), evaluated under teacher forcing. Let \(c_t=(x,y_{<t})\), and let \(\beta=1\). Assume there exists a set \(I\subseteq\{1,\ldots,T\}\) with \(|I|\geq \lceil(1-\alpha)T\rceil\) such that \(\pi_{\theta}^{(k,1)}(y_t\mid c_t)\leq\rho\) for every \(t\in I\). Then
\[
    \pi_\theta(y\mid x)
    =
    \prod_{t=1}^{T}\pi_\theta(y_t\mid c_t)
    \leq
    \rho^{\lceil(1-\alpha)T\rceil}.
\]
\end{restatable}
The proof of \Cref{prop:masc-reproduction} is deferred to Appendix \ref{app:proofs}. The proposition shows that enforcing the condition of \Cref{def:restriced-suppression} on many positions $t \in [T]$ gives a bound on the probability of exactly reproducing the forgotten continuation \(y\) from prompt \(x\). 

\textbf{Remark.} The parameter \(k\) does not appear in the bound because the proposition is conditional on the local constraint being satisfied: for any $k<|\mathcal{V}|$, once \(\pi_{\theta}^{(k,1)}(y_t\mid c_t)\leq\rho\), the full-vocabulary probability also satisfies \(\pi_\theta(y_t\mid c_t)\leq\rho\). Thus, \(k\) only determines how stringent the local comparison is, while the sequence-level bound depends on \(\rho\) and on the number of controlled positions. We defer the discussion of how \(k\) affects the optimization of the proposed algorithm to \Cref{app-subsec:size-k}.

\begin{takeawaybox}[Takeaway 1]
Controlling \emph{local} target-token dominance on many positions gives an upper bound on the probability of reproducing the exact continuation.
\end{takeawaybox}

\paragraph{From probabilities to margins.}
\Cref{prop:masc-reproduction} gives an upper bound on exact reproduction when the condition in \Cref{def:restriced-suppression} is met for many forget tokens in the sequence. A direct surrogate would penalize violations with \([\pi_{\theta}^{(k,\beta)}(y_t\mid c_t)-\rho]_+\). However, this probability-space penalty can have weak gradients when the observed token already dominates the restricted set. In that regime, the softmax probability is close to one, so even large changes in the underlying logit margin produce only small changes in the penalized quantity. 
We therefore propose an alternative logit-space condition based on the following \emph{restricted margin}
\begin{equation}
    m_{\theta}^{(k,\beta)}(x,y,t)
    =
    \beta z_\theta(y_t\mid c_t)
    -
    \log
    \sum_{v\in \mathcal{S}_{\theta,k}(c_t)}
    \exp(\beta z_\theta(v\mid c_t)).
    \label{eq:masc-margin}
\end{equation}
The margin in \eqref{eq:masc-margin} compares the target-token logit with the log-sum-exp\footnote{The log-sum-exp term is a smooth approximation of the maximum alternative logit. Larger \(\beta\) makes the margin closer to a gap against the strongest alternative, while smaller \(\beta\) averages more broadly over the top-\(k\) alternatives.} aggregate of the selected alternative logits. Large margins correspond to strong target-token preference; small margins correspond to competitive alternatives.  
Moreover, because this quantity is defined directly on the logits, increasing dominance of the observed token continues to increase the violation linearly.
In addition, the following lemma establishes that imposing a threshold on the restricted probability as in \Cref{def:restriced-suppression} is equivalent to imposing a corresponding threshold on the margin of \Cref{eq:masc-margin}.

\begin{restatable}{lemma}{lemmaequivalence}
\label{lem:restricted-prob-margin-equivalence}
Fix \(\rho\in(0,1)\) and define \(\tau_\rho=\log(\rho/(1-\rho))\). Then, for any forget position \(t\) and any \(\beta>0\), \(\pi_{\theta}^{(k,\beta)}(y_t\mid c_t)\leq \rho\) if and only if \(m_{\theta}^{(k,\beta)}(x,y,t)\leq \tau_\rho\).
\end{restatable}

\subsection{MASC: Unlearning with Margin Self-Correction}
We are now ready to propose our new unlearning method, MASC, that returns a policy whose restricted margin condition is violated on at most an $\alpha$ fraction of the forget-set tokens. Specifically, the algorithm uses the average $V_\rho$ of the \emph{per-example violation rate} $v_\rho$ over the forget set
\begin{equation}
\label{eq:violation-rate}
v_\rho(\theta;x,y)=\frac{1}{T}\sum_{t=1}^T
\one\{m_{\theta}^{(k,\beta)}(x,y,t)>\tau_\rho\}, \quad
V_\rho(\theta;\Dforget)=
\mathbb{E}_{(x,y)\sim\Dforget}\left[v_\rho(\theta;x,y)\right].
\end{equation}
In words, \(V_\rho\) corresponds to the fraction of forget tokens whose restricted target probability remains above \(\rho\). Our proposed unlearning objective then aims to find models that satisfy a forget constraint of the form \(V_\rho(\theta;\Dforget)\leq\alpha\) for some \(\alpha\in[0,1]\), while behaving similarly to the original model on retain data. In particular, motivated by recent evidence that policy-level KL divergence is closely tied to forgetting dynamics \cite{shenfeld2026rls}, we choose to minimize the KL divergence (averaged over retain continuations) between \(\pi_{\theta_0}\) and the currently optimized policy $\pi_\theta$. All in all, we aim to solve
\begin{equation}
    \min_\theta
    \quad
    \underbrace{\mathbb{E}_{(x,y)\sim\Dretain}
    \left[
    \frac{1}{T}
    \sum_{t=1}^{T}
    \KL\!\left(
        \pi_{\theta_0}(\cdot\mid x,y_{<t})
        \,\middle\|\,
        \pi_\theta(\cdot\mid x,y_{<t})
    \right)
    \right]}_{\Lret^{\KL}(\theta,\theta_0)}
    \quad
    \text{s.t.}
    \quad
    V_\rho(\theta;\Dforget)\leq\alpha .
    \label{eq:masc-constrained}
\end{equation}
In order to solve this optimization problem with gradient-based algorithms, we first replace the indicator in \eqref{eq:violation-rate} with the hinge loss \(\psi_{\rho,\eta}(m)=[m-(\tau_\rho-\eta)]_+/\eta\) as a surrogate loss that satisfies \(\one\{m>\tau_\rho\}\leq\psi_{\rho,\eta}(m)\). The final MASC algorithm then minimizes the following Lagrangian objective with gradient descent:
\begin{equation}
\min_\theta \Lret^{\KL}(\theta,\theta_0) + \lambda \Lfg^{\mathrm{MASC}}(\theta) \:\:\text{where}\:\: \Lfg^{\mathrm{MASC}}(\theta)
    =
    \mathbb{E}_{(x,y)\sim\Dforget}
    \left[
    \frac{1}{T}
    \sum_{t=1}^{T}
    \psi_{\rho,\eta}\!\left(m_{\theta}^{(k,\beta)}(x,y,t)\right)
    \right].
    \label{eq:masc-forget-loss}
\end{equation}

\begin{figure}[ht] 
    \centering 
    \resizebox{\linewidth}{!}{\input{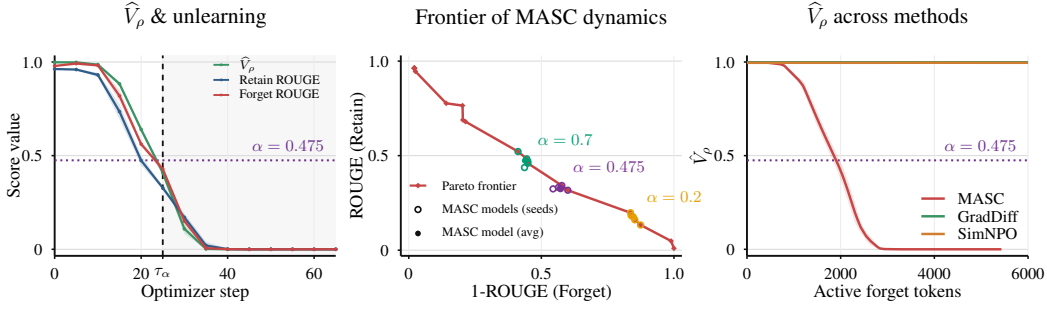}} \caption{\textbf{(Left:)} Correlation between stopping statistic $\hat V_\rho$ and forget--retain trade-off. \textbf{(Middle:)} Pareto frontier (top-right is the ideal) of MASC training dynamics. \textbf{(Right:)} $\hat V_\rho$ versus number of forget tokens seen during training for three different unlearning methods. [Plots refer to TOFU].} 
\label{fig:stopping_panel} 
\vspace{-1em}
\end{figure}
Note that for a fixed set \(\mathcal{S}_{\theta,k}(c_t)\), consisting of tokens without a sufficient slack $\tau_\rho-m_{\theta}^{(k,\beta)}(x,y,t)<\eta$, the gradient $\nabla\Lfg^{\mathrm{MASC}}$ decreases the target next-token logit, and increases the probabilities of the competitive alternatives relative to the target
token (see Appendix~\ref{app:masc-gradient} for the full derivation). On the other hand, tokens with sufficient slack have zero gradient.
Given an initial fine-tuned model $\theta_0$, MASC runs gradient updates with model weights \(\theta^{(s)}\) at step \(s\). Crucially, it terminates using a stopping rule that monitors the average violation rate on a subset \(\Dval\subset\Dforget\)
\[
\widehat V_\rho\big(\theta^{(s)};\Dval\big)
=
\frac{1}{|\Dval|}
\sum_{(x,y)\in \Dval}
\frac{1}{T_y}
\sum_{t=1}^{T_y}
\one\!\left\{
m_{\theta^{(s)}}^{(k,\beta)}(x,y,t)>\tau_\rho
\right\},
\]
where $T_y$ is the length of sequence $y$. In particular, $\Dval$ is drawn at random as a small \emph{off-batch} subset of $\Dforget$ (of size $8$ to $16$ in our experiments) and is employed during training only as a stopping condition. Note that choosing a small $\Dval$ avoids a more costly pass over the entire $\Dforget$ while still providing a usable stopping signal (see \Cref{sec:experiments}). Moreover, $\Dval$ is independent of downstream evaluation\footnote{In many datasets, training data is usually in the form of raw text, while evaluation data is in the form of Q\&A text.} and does not require creating extra datasets (beyond the ones provided by the benchmark suites) that are more \emph{aligned} with the downstream evaluation sets (e.g., \cite{zhong2026duet}). MASC stops at the first instance at which the monitored violation rate falls below tolerance \(\alpha\), i.e. at step
\[
    \tau_\alpha
    =
    \inf\left\{
        s\geq 0:
        \widehat V_\rho\big(\theta^{(s)};\Dval \big)\leq \alpha
    \right\}.
\]
The final policy that MASC outputs is \(\pi_{\text{MASC}} = \pi_{\theta^{(\tau_\alpha)}}\) (cf. \Cref{alg:masc} in Appendix). MASC is closest in spirit to recent logit-level unlearning methods such as UNDIAL~\cite{dong2025undial}, Unilogit~\cite{vasilev-etal-2025-unilogit}, and constrained entropy or logit-flattening approaches~\cite{entesari2026constrained}. Unlike these methods, which distill toward modified full-vocabulary targets or flatten the predictive distribution under a fixed budget, MASC enforces a relative local condition against a small set of model-proposed alternatives and uses the same condition for early stopping.

\paragraph{Monitored violation rate.} Experiments strongly suggest that the monitored violation rate \(\widehat V_\rho\) behaves as intended. First, \Cref{fig:stopping_panel} (Left) shows how, along the MASC optimization trajectory, \(\widehat V_\rho\) (computed on $\Dval\subset\Dforget$) closely tracks the forget--retain trade-off measured by standard downstream evaluation metrics (which are not computed during training). In addition, we observe how the trajectory traces a Pareto frontier on the forget--retain trade-off space and the tolerance \(\alpha\) selects different points on this frontier; see \Cref{fig:stopping_panel} (Middle). A natural question is whether the same stopping statistic could serve as a generic early-stopping criterion for other unlearning objectives. Our experiments suggest that it does not: under the same budget of processed forget tokens, \(\widehat V_\rho\) decreases steadily for MASC but remains close to its initial value for GradDiff and SimNPO; see \Cref{fig:stopping_panel} (Right). This provides strong empirical evidence that unlike MASC, other established unlearning methods would not stop early using the same stopping criterion, as they are not designed to reduce the violation statistic.

\section{Experiments}
\label{sec:experiments}

We now evaluate MASC against several baselines on three well-known LLM unlearning datasets: TOFU (90/10 split), MUSE News, and MUSE Books.

\begin{table}[ht]
\centering

\resizebox{\linewidth}{!}{
\begin{tabular}{lccccc}
\toprule
& \multicolumn{3}{c}{Unlearning Efficacy} 
& \multicolumn{1}{c}{Retain Utility}
& \multicolumn{1}{c}{Efficiency} \\
\cmidrule(lr){2-4}
\cmidrule(lr){5-5}
\cmidrule(lr){6-6}
Method
& \(1-\)ROUGE-L \(\uparrow\)
& \(1-\)Prob \(\uparrow\)
& Truth Ratio \(\uparrow\)
& MU \(\uparrow\)
& Time (sec) \(\downarrow\) \\
\midrule
Base (Llama-2 7B) & 0.024 & 0.010 & 0.519 & 0.628 & -- \\
Retrain & 0.601 & 0.852 & 0.681 & 0.613 & -- \\
\midrule
GA & 0.330 {\scriptsize [0.029]} & \textbf{0.829} {\scriptsize [0.022]} & 0.555 {\scriptsize [0.007]} & 0.459 {\scriptsize [0.014]} & 306.6 {\scriptsize [40.2]} \\
GradDiff & \underline{0.598} {\scriptsize [0.020]} & \underline{0.792} {\scriptsize [0.003]} & 0.514 {\scriptsize [0.003]} & 0.561 {\scriptsize [0.005]} & 907.3 {\scriptsize [168.5]} \\
NPO & 0.366 {\scriptsize [0.023]} & 0.666 {\scriptsize [0.005]} & \underline{0.580} {\scriptsize [0.012]} & 0.533 {\scriptsize [0.003]} & 856.3 {\scriptsize [36.4]} \\
NPO+KLR & 0.362 {\scriptsize [0.016]} & 0.713 {\scriptsize [0.006]} & 0.577 {\scriptsize [0.006]} & 0.516 {\scriptsize [0.006]} & 983.3 {\scriptsize [48.9]} \\
RMU & 0.080 {\scriptsize [0.004]} & 0.103 {\scriptsize [0.011]} & 0.523 {\scriptsize [0.000]} & \underline{0.618} {\scriptsize [0.001]} & \underline{305.4} {\scriptsize [41.8]} \\
SimNPO & 0.349 {\scriptsize [0.006]} & 0.497 {\scriptsize [0.004]} & 0.562 {\scriptsize [0.001]} & 0.596 {\scriptsize [0.001]} & 541.7 {\scriptsize [37.8]} \\
\midrule
MASC (Ours) & \textbf{0.629} {\scriptsize [0.142]} & 0.672 {\scriptsize [0.127]} & \textbf{0.633} {\scriptsize [0.020]} & \textbf{0.666} {\scriptsize [0.003]} & \textbf{87.9} {\scriptsize [8.1]} \\
\bottomrule
\end{tabular}
}

\vspace{0.7em}
\caption{TOFU (forget10/retain90) results. Averages and standard deviations are reported as avg {\scriptsize [std]}. We mark the best-performing unlearning method in bold and underline the runner-up for each metric.}
\label{tab:tofu-main-results}

\vspace{1.2em}

\resizebox{\linewidth}{!}{
\begin{tabular}{llcccc}
\toprule
&
& \multicolumn{2}{c}{Unlearning Efficacy}
& \multicolumn{1}{c}{Retain Utility}
& \multicolumn{1}{c}{Efficiency} \\
\cmidrule(lr){3-4}
\cmidrule(lr){5-5}
\cmidrule(lr){6-6}
Dataset
& Method
& VerbMem \(\Dforget\) \(\downarrow\)
& KnowMem \(\Dforget\) \(\downarrow\)
& KnowMem \(\Dretain\) \(\uparrow\)
& Time (sec) \(\downarrow\) \\
\midrule
\multirow{9}{*}{MUSE News}
& Base (Llama-2 7B) & 57.25 & 66.45 & 54.90 & -- \\
& Retrain & 20.26 & 32.55 & 55.31 & -- \\
\cmidrule(lr){2-6}
& GA & 0.00 {\scriptsize [0.00]} & 0.00 {\scriptsize [0.00]} & 0.00 {\scriptsize [0.00]} & \underline{183.04} {\scriptsize [0.01]} \\
& GradDiff & \textbf{0.26} {\scriptsize [0.19]} & \underline{25.30} {\scriptsize [3.22]} & 34.38 {\scriptsize [3.01]} & 2517.87 {\scriptsize [71.38]} \\
& NPO & 0.00 {\scriptsize [0.00]} & 0.00 {\scriptsize [0.00]} & 0.00 {\scriptsize [0.00]} & 227.88 {\scriptsize [0.01]} \\
& NPO+KLR & 6.32 {\scriptsize [1.38]} & 51.78 {\scriptsize [3.50]} & \textbf{44.36} {\scriptsize [2.97]} & 4062.76 {\scriptsize [628.86]} \\
& RMU & 27.15 {\scriptsize [1.33]} & 47.81 {\scriptsize [3.74]} & \underline{41.95} {\scriptsize [3.04]} & 1076.06 {\scriptsize [89.67]} \\
& SimNPO & 8.03 {\scriptsize [0.61]} & 45.81 {\scriptsize [3.64]} & 37.02 {\scriptsize [3.00]} & 1877.52 {\scriptsize [18.81]} \\
\cmidrule(lr){2-6}
& MASC (Ours) & \underline{1.10} {\scriptsize [0.25]} & \textbf{19.37} {\scriptsize [3.04]} & 23.14 {\scriptsize [2.79]} & \textbf{138.68} {\scriptsize [37.28]} \\
\midrule
\multirow{9}{*}{MUSE Books}
& Base (ICLM-7B) & 99.70 & 47.12 & 69.13 & -- \\
& Retrain & 14.45 & 30.29 & 68.74 & -- \\
\cmidrule(lr){2-6}
& GA & 0.00 {\scriptsize [0.00]} & 0.00 {\scriptsize [0.00]} & 0.00 {\scriptsize [0.00]} & \underline{290.16} {\scriptsize [0.01]} \\
& GradDiff & \textbf{0.00} {\scriptsize [0.00]} & \textbf{0.00} {\scriptsize [0.00]} & 41.23 {\scriptsize [4.12]} & 572.23 {\scriptsize [18.70]} \\
& NPO & 0.00 {\scriptsize [0.00]} & 0.00 {\scriptsize [0.00]} & 0.00 {\scriptsize [0.00]} & 355.13 {\scriptsize [2.90]} \\
& NPO+KLR & \textbf{0.00} {\scriptsize [0.00]} & 23.34 {\scriptsize [3.28]} & \textbf{67.74} {\scriptsize [3.89]} & 2570.19 {\scriptsize [407.62]} \\
& RMU & 11.05 {\scriptsize [0.38]} & 22.37 {\scriptsize [3.40]} & 60.76 {\scriptsize [3.96]} & 1477.86 {\scriptsize [108.12]} \\
& SimNPO & \textbf{0.00} {\scriptsize [0.00]} & \textbf{0.00} {\scriptsize [0.00]} & 47.79 {\scriptsize [4.17]} & 2647.71 {\scriptsize [45.48]} \\
\cmidrule(lr){2-6}
& MASC (Ours) & \underline{0.90} {\scriptsize [0.80]} & 30.90 {\scriptsize [1.30]} & \underline{65.30} {\scriptsize [1.20]} & \textbf{64.94} {\scriptsize [0.49]} \\
\bottomrule
\end{tabular}
}

\vspace{0.7em}
\caption{MUSE results. Averages and standard deviations are reported as $\mathrm{avg_{[std]}}$. We mark the best-performing model in bold, excluding Base, Retrain, and models with zero retain utility ($\mathrm{KnowMem} \ \Dretain$), and underline the runner-up for each metric. Ties are resolved by retain utility.}
\label{tab:muse-main-results}
\vspace{-1em}
\end{table}

\textbf{Baselines.}
We compare MASC against standard unlearning baselines: \textbf{(i)} Gradient Ascent (GA) \cite{yao2024large}; \textbf{(ii)} GradDiff (or GA+GDR) \cite{maini2024tofu}, which adds a retain-side correction by combining GA on the forget set with gradient descent on retained examples; \textbf{(iii)} NPO \cite{zhang2024negative}; \textbf{(iv)}
NPO+KLR, which combines the NPO objective together with a KL retain regularizer to reduce drift on
\(\Dretain\) \cite{shi2024muse}; \textbf{(v)} RMU, a representation-level method that redirects activations associated with the forget data \cite{li2024the}; \textbf{(vi)} SimNPO \cite{fan2026simplicity}. We also include pretrained and retrain (on $\Dretain$ only) baselines for comparison. In all MASC experiments, we keep the backbone model frozen and show that the unlearning update can be performed effectively by training only LoRA adapters \cite{hu2022lora}. This also makes MASC memory-efficient and confines the update to a small trainable module.

\textbf{Metrics.}
On \texttt{TOFU}, unlearning-efficacy metrics are: \(1-\)ROUGE-L, measuring lexical dissimilarity from the target answer; \(1-\)Prob., measuring the reduction in teacher-forced probability of the target continuation; Truth Ratio, measuring the preference for perturbed alternatives over the true forgotten answer. On the other hand, MU measures utility on the retain portion of the data. For \texttt{MUSE News} and \texttt{MUSE Books}, we report VerbMem \(\Dforget\), corresponding to \emph{verbatim} memorization on the forget set; KnowMem \(\Dforget\), measuring \emph{knowledge} memorization on the forget set; KnowMem \(\Dretain\), which accounts for knowledge preservation on the retain set. Additional metrics are reported in \Cref{app:exp-details}. We measure wall-clock unlearning runtime in seconds. To make timing comparable, all methods are run on the same hardware\footnote{All experiments are run on a single H100 GPU.}, data pipeline for each dataset, and evaluation-free training loop. Timing starts after model and data loading, and stops when the method returns its checkpoint: at the prescribed final step for fixed-schedule baselines, and at the first checkpoint satisfying \(\widehat V_\rho\leq\alpha\) for MASC. We exclude shared one-time costs such as model loading and dataset preprocessing, but we include optimizer steps, forward/backward passes, retain batches, KL computations, and MASC online probe checks. Baselines are run using code from their official repositories and with the reported best hyperparameters. For timing, we measure the number of epochs specified by each baseline's selected hyperparameter setting, even when reproducing the reported forget-retain trade-off required additional epochs in our runs. Moreover, when timing MASC's competitors, we do not include the cost of offline checkpoint selection based on downstream metric evaluation. Thus, whenever possible, our comparison favors the baselines in both runtime and forget--retain performance.

\textbf{Results.}
\Cref{tab:tofu-main-results,tab:muse-main-results} show that MASC consistently achieves a competitive forget--retain trade-off across the three datasets using a fraction of the wall-clock runtime required by the other methods (see \Cref{fig:pareto-fronts} in Appendix \ref{app:ablations} for a visual summary of such a trade-off). On \texttt{TOFU}, MASC obtains the best \(1-\)ROUGE-L, Truth Ratio, retain MU, as well as the shortest wall-clock time, while remaining competitive on \(1-\)Prob. On \texttt{MUSE Books}, MASC preserves retain utility close to the strongest non-collapsed (i.e., those with nonzero retain utility) baselines at a significantly lower computational cost. On \texttt{MUSE News}, MASC achieves lower KnowMem on \(\Dforget\) than most baselines, although its retain utility is sometimes slightly lower.

\begin{takeawaybox}[Takeaway 2]
MASC shifts unlearning from fixed-length training to targeted early stopping, reaching a competitive forget--retain trade-off substantially faster than fixed-schedule counterparts.
\end{takeawaybox}

\section{The effect of scale on unlearning}
\label{sec:scaling}
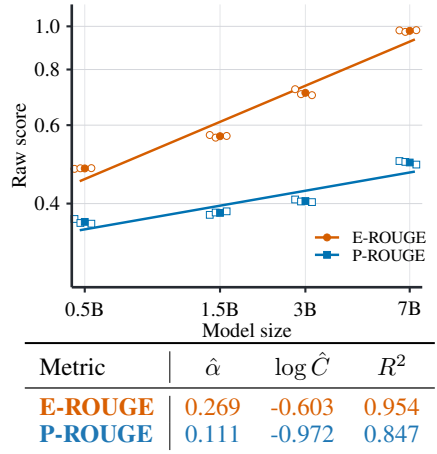
\begin{wrapfigure}[14]{r}{0.4\columnwidth}
    \vspace{-0.7in}
    \centering
    \resizebox{\linewidth}{!}{\input{plots/scaling_learning.pgf}}
    \vspace{0.8em}
    \resizebox{\linewidth}{!}{
    \begin{tabular}{l|ccc}
    \toprule
    Metric & \(\hat \alpha\) & \(\log \hat C\) & \(R^2\) \\
    \midrule
    \textcolor{rougeorange}{\textbf{E-ROUGE}}
    & \textcolor{rougeorange}{0.269}
    & \textcolor{rougeorange}{-0.603}
    & \textcolor{rougeorange}{0.954} \\
    \textcolor{parablue}{\textbf{P-ROUGE}}
    & \textcolor{parablue}{0.111}
    & \textcolor{parablue}{-0.972}
    & \textcolor{parablue}{0.847} \\
    \bottomrule
    \end{tabular}
    }
    \vspace{-1em}
    \caption{\textbf{Learning stage.} Fitted scaling laws (log-log plot).
    }
    \label{fig:scaling_learning}
\end{wrapfigure}
Improving the efficiency of unlearning methods is a necessary step toward scaling these procedures to increasingly large models. However, to the best of our knowledge, there is no systematic study on how model size affects unlearning efficacy. Some related recent studies suggest that memorization increases with scale during training \cite{carlini2022quantifying,morris2025much,lu2024scaling} so that larger models may therefore enter the unlearning stage with different levels and forms of memorization. At the same time, it is unclear how scale would then affect the final forget--retain trade-off after unlearning. In this section, we study the scaling behavior of the \emph{learning-unlearning pipeline} of MASC and SimNPO across model sizes on the Qwen2.5 model family using the TOFU dataset.

\subsection{Unlearning procedures and metrics}
We now describe the learning-unlearning pipeline, along with the memorization metrics used.

\textbf{End-to-end unlearning pipeline.} Most LLM unlearning benchmarks, including \texttt{TOFU} \cite{maini2024tofu}, \texttt{MUSE} \cite{shi2024muse}, and \texttt{WMDP} \cite{li2024the}, are built from a common two-stage pipeline.
\textbf{(1) Learning.} An initial model \(\pi_{\mathrm{init}}\) is finetuned on the full benchmark data \(\mathcal D=\Dforget\cup\Dretain\) to yield a task-adapted model \(\pi_{\theta_0}\) (usually called \emph{base} model) that contains both forget and retain information. 
\textbf{(2) Unlearning.} An unlearning algorithm \(\mathcal A\) (here MASC and SimNPO) is then applied to \(\pi_{\theta_0}\) using the split \((\Dforget,\Dretain)\), producing an unlearned model \(\pi_{\theta_{\mathrm{unl}}}\). Following this pipeline, we evaluate the Qwen2.5 model family at different scales on \texttt{TOFU} and track how forget-set memorization changes from learning to unlearning.

\vspace{-0.25em}

\textbf{Two memorization levels.} We use \texttt{TOFU} built-in metrics to distinguish between two levels of memorization. The first level is \emph{exact memorization}: the model reproduces the target answer under the original question. We measure this with Exact Q\&A ROUGE (E-ROUGE), that computes lexical overlap between the model output and the gold answer on the original \texttt{TOFU} questions. 
The second level is \emph{paraphrase-robust knowledge memorization}: the model recovers the same underlying answer even when the question is paraphrased, measured by Paraphrased Q\&A ROUGE (P-ROUGE). To measure the resulting forget--retain trade-off, we also report retain utility, measured by MU, \texttt{TOFU} aggregate score on \(\Dretain\).

\vspace{-0.45em}

\subsection{Effects of size on unlearning}
We first examine the learning stage, where the model is supervised-finetuned on the full benchmark data before unlearning. For each memorization metric, we average scores over seeds at each model size and fit an empirical power law \(s(N)=C N^\alpha\) in log--log space, where \(N\) denotes the number of parameters. As shown in \Cref{fig:scaling_learning}, scale affects the two memorization levels differently. Exact reproduction grows faster with model size than paraphrase-based knowledge recovery: E-ROUGE has exponent \(\hat\alpha=0.269\), while P-ROUGE has exponent \(\hat\alpha=0.111\). Since these metrics are bounded by one, the exact-memorization scores are already close to saturation at size 7B. Overall, this suggests that larger models become disproportionately better at reproducing target content in its original form compared to recovering the same content under paraphrased prompts. These results are consistent with prior evidence that memorization increases with model scale \cite{carlini2022quantifying,morris2025much,lu2024scaling}, while adding a distinction between forms of memorization that grow at different rates. As a consequence of this learning-stage behavior, models of different sizes enter the unlearning stage with different memorization profiles.
\begin{wrapfigure}[13]{r}{0.55\columnwidth}
    \vspace{-0.7em}
    \centering
    \scalebox{0.62}{\input{plots/qwen_tofu_unlearning_radar_by_scale.pgf}}
    \vspace{-0.5em}
    \caption{\textbf{Unlearning stage.} Cross-scale behavior after unlearning for MASC and SimNPO. All the metrics are plotted such that the higher the better.}
    \label{fig:unlearning-scaling-radar}
\end{wrapfigure}
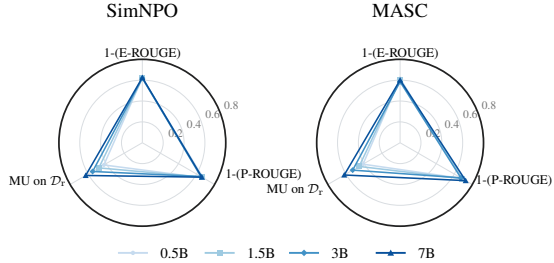

Interestingly, the picture changes after unlearning. Despite different starting memorization profiles (cf. \Cref{fig:scaling_learning}), unlearning brings the forget-side metrics back to a similar range across model sizes. In particular, we summarize the forget--retain trade-off after unlearning using the radar plots in \Cref{fig:unlearning-scaling-radar}. Each line in the plot corresponds to a model size, with \(1-\mathrm{E\text{-}ROUGE}\) and \(1-\mathrm{P\text{-}ROUGE}\), reported together with retain utility (MU). For both MASC and SimNPO, the forget-side metrics remain in a similar range across model sizes after unlearning, while retain utility improves more clearly with scale. This suggests that larger models mainly improve the utility side of the forget--retain trade-off, rather than yielding different degrees of forgetting across model sizes. For the unlearning-stage scaling-law plot, see \Cref{fig:unlearning-scaling} in Appendix.

\begin{takeawaybox}[Takeaway 3]
During learning, larger models amplify the two levels of memorization with different strengths. However, after unlearning, residual memorization is largely scale-invariant, while larger models preserve higher retain utility.
\end{takeawaybox}

\section{Discussion and future work}
\label{sec:discussion}

We introduce MASC, a margin-based unlearning method whose loss and stopping rule both target the same condition: forget tokens should no longer dominate plausible model-proposed alternatives. This makes the procedure self-stopping (or \emph{adaptive}) and thus substantially faster than fixed-budget baselines. While our experiments suggest that MASC also improves paraphrase-level forgetting metrics, our bound still does not control such behavior directly: doing so would require a notion of forgetting defined in a representation space invariant to surface \emph{rewordings}. Designing mathematically grounded unlearning objectives for this semantic regime is an interesting direction for future work.

\section*{Acknowledgment}
FDG was supported by Swiss National Science Foundation (SNSF) Grant 218343, and AS was supported by the Swiss National Science Foundation (SNSF) Grant 204439. The authors acknowledge the use of LLMs to improve exposition and generate code. The authors take full responsibility for the content of the paper.

\clearpage
\bibliographystyle{plainnat}
\bibliography{biblio}

\clearpage
\appendix

\section{Additional related work}
\label{app:additional-related-work}

Beyond the likelihood-reversal and preference-optimization baselines considered in the main text, one line of work uses \emph{relabeling-based finetuning}, replacing the original forget-set responses with generic, neutral, or refusal-like alternatives before further finetuning~\cite{eldan2023s,cao2024rwku}. Another line studies \emph{reinforcement-learning} formulations of unlearning, for example by using reward models or negative-similarity rewards to discourage undesirable generations while preserving fluency~\cite{lu2022quark,kassem2023preserving}. Localized-parameter approaches instead try to identify and edit parts of the model most responsible for the target information, including representation-engineering methods, adaptive variants of representation redirection, and locate-then-unlearn approaches based on neuron or parameter attribution~\cite{li2024the,dang2025effects,wang2025gru,hong2024dissecting}. A further family leverages auxiliary models, including task-vector methods, contrastive decoding, and knowledge-distillation-based unlearning~\cite{ilharco2022editing,lu2024eraser,ji2024reversing,wang2024rkld,dong2025undial}. Finally, some methods avoid weight updates altogether or combine them with input/output-side interventions, such as prompt classifiers, input corruption, guardrails, filtering, or in-context unlearning~\cite{thaker2024guardrail,pmlr-v235-pawelczyk24a}. These directions illustrate that LLM unlearning can be pursued through output losses, representation editing, auxiliary-model guidance, or inference-time control.

\paragraph{Positioning of MASC with similar methods.}
A few recent works might be considered close to MASC because they also modify the model's output distribution on forget examples. UNDIAL \cite{dong2025undial} proposes a self-distillation approach in which the target-token logit is adjusted downward and the model is trained to match the resulting softened distribution, with the goal of avoiding the over-unlearning and instability observed in GA and NPO. Unilogit \cite{vasilev-etal-2025-unilogit} further develops this direction by constructing self-distillation targets from the current model and dynamically adjusting the target logit so that the target token receives uniform probability. Another closely related line formulates unlearning through entropy or logit-flattening objectives: \citet{entesari2026constrained} cast forgetting and retention as a constrained optimization problem, uses a logit-margin flattening loss to drive the full predictive distribution toward uniformity on the forget set, and solves the resulting problem with a primal-dual procedure. These methods share with MASC the view that stable unlearning should act directly on the model's local predictive distribution rather than simply maximizing forget loss.
MASC differs from these approaches in both the target of suppression and the role of the training statistic. Rather than distilling toward a modified full-vocabulary target distribution, as in UNDIAL or Unilogit, or flattening the entire output distribution toward uniformity, as in entropy-based or logit-flattening methods, MASC imposes a relative local condition: the gold forget token should no longer dominate a small set of plausible non-gold alternatives proposed by the model itself. This makes the forget update selective. Tokens whose local dominance is already below threshold contribute no forget gradient, while only the still-dominant positions are corrected. Moreover, the same margin-violation event defines the loss, the constrained forget condition, and the stopping rule. Thus, MASC is not only a logit-level suppression objective; it is also a self-terminating unlearning procedure whose tolerance parameter directly selects a point along the empirical forget--retain frontier.

\section{Proofs}
\label{app:proofs}

\propboundexactrepro*
\begin{proof}
For \(\beta=1\), the restricted probability is obtained by restricting the normalization to
\(\{y_t\}\cup \mathcal{S}_{\theta,k}(c_t)\). Since this removes nonnegative terms from the full softmax denominator,
we have
\[
    \pi_\theta(y_t\mid c_t)
    \leq
    \pi_{\theta}^{(k,1)}(y_t\mid c_t)
    \leq \rho \quad \text{for every } t\in I.
\]
Moreover, by assumption, \(|I|\geq \lceil(1-\alpha)T\rceil\). For the remaining positions \(t\notin I\), we only use the trivial bound
\(\pi_\theta(y_t\mid c_t)\leq 1\). Hence
\[
    \pi_\theta(y\mid x)
    =
    \prod_{t=1}^{T}\pi_\theta(y_t\mid c_t)
    =
    \prod_{t\in I}\pi_\theta(y_t\mid c_t)
    \prod_{t\notin I}\pi_\theta(y_t\mid c_t)
    \leq
    \rho^{|I|}
    \leq
    \rho^{\lceil(1-\alpha)T\rceil},
\]
where the last inequality uses \(\rho\in(0,1)\) and \(|I|\geq \lceil(1-\alpha)T\rceil\).
\end{proof}

Note that the reproduction bound can be extended to \(\beta\geq 1\), at the cost of replacing \(\rho\) by a sigmoid-rescaled threshold. Let
\(\tau_\rho=\log(\rho/(1-\rho))\), and suppose that, for all \(t\in I\),
\[
    \pi_{\theta}^{(k,\beta)}(y_t\mid c_t)\leq \rho .
\]
Equivalently,
\[
    m_{\theta}^{(k,\beta)}(x,y,t)
    =
    \beta z_\theta(y_t\mid c_t)
    -
    \log\sum_{v\in \mathcal{S}_{\theta,k}(c_t)}
    \exp(\beta z_\theta(v\mid c_t))
    \leq \tau_\rho .
\]
Since \(\beta\geq 1\),
\[
    \frac{1}{\beta}
    \log\sum_{v\in \mathcal{S}_{\theta,k}(c_t)}
    \exp(\beta z_\theta(v\mid c_t))
    \leq
    \log\sum_{v\in \mathcal{S}_{\theta,k}(c_t)}
    \exp(z_\theta(v\mid c_t)).
\]
Therefore,
\[
    m_{\theta}^{(k,1)}(x,y,t)
    \leq
    \frac{1}{\beta}
    m_{\theta}^{(k,\beta)}(x,y,t)
    \leq
    \frac{\tau_\rho}{\beta}.
\]
Using the identity
\[
    \pi_{\theta}^{(k,1)}(y_t\mid c_t)
    =
    \sigma\!\left(m_{\theta}^{(k,1)}(x,y,t)\right),
\]
we obtain
\[
    \pi_{\theta}^{(k,1)}(y_t\mid c_t)
    \leq
    \sigma\!\left(\frac{\tau_\rho}{\beta}\right).
\]
Since the full softmax denominator contains all vocabulary tokens, $\pi_\theta(y_t\mid c_t) \leq \pi_{\theta}^{(k,1)}(y_t\mid c_t)$, and hence, if this condition holds on a set \(I\) with
\(|I|\geq \lceil(1-\alpha)T\rceil\), then
\[
    \pi_\theta(y\mid x)
    \leq
    \left[
    \sigma\!\left(\frac{\tau_\rho}{\beta}\right)
    \right]^{\lceil(1-\alpha)T\rceil}.
\]

\textbf{Remark.} For \(\rho>1/2\), we have \(\tau_\rho=\log(\rho/(1-\rho))>0\). Hence, as \(\beta\) increases,
\[
    \sigma\!\left(\frac{\tau_\rho}{\beta}\right)
    \to \frac{1}{2}.
\]
Thus, controlling the \(\beta\)-sharpened restricted probability implies a reproduction bound with an effective per-token threshold closer to \(1/2\). In the limit \(\beta\to\infty\), the restricted comparison approaches a hard maximum over the selected competitors, and the condition becomes the requirement that the target token should not beat its strongest plausible alternative. This matches the intended interpretation of MASC: forgotten tokens should no longer be clearly preferred over the model's own local alternatives.

\lemmaequivalence*
\begin{proof}
Since
\(\pi_{\theta}^{(k,\beta)}(y_t\mid c_t)
=\sigma(m_{\theta}^{(k,\beta)}(x,y,t))\) and
\(\sigma^{-1}(\rho)=\tau_\rho\), the claim follows by monotonicity of
\(\sigma\).
\end{proof}

\section{Gradient of the MASC forget term}
\label{app:masc-gradient}

In this section, we derive the token-level gradient effect of the MASC forget loss $\mathcal{L}_{\mathrm{fg}}^{\mathrm{MASC}}$ in \Cref{eq:masc-forget-loss}. The
calculation is performed at the level of logits for a single forget token. It
therefore describes the direct contribution of the forget term before averaging
over examples, positions, and minibatches. The retain term contributes an
additional gradient that is not included in this local calculation.

Fix a forget example \((x,y)\), a position \(t\), and the teacher-forced context
\(c_t=(x,y_{<t})\). Let \(\mathcal{S}=\mathcal{S}_{\theta,k}(c_t)\) denote the selected top-\(k\)
non-gold competitor set. As in the main text, we treat \(\mathcal{S}\) as fixed during
differentiation, since the top-\(k\) selection is not differentiated through.
The restricted margin is
\[
    m_\theta
    =
    m_{\theta}^{(k,\beta)}(x,y,t)
    =
    \beta z_\theta(y_t\mid c_t)
    -
    \log
    \sum_{u\in \mathcal{S}}
    \exp(\beta z_\theta(u\mid c_t)).
\]
The token-level MASC surrogate is
\[
    \psi_{\rho,\eta}(m_\theta)
    =
    \frac{[m_\theta-(\tau_\rho-\eta)]_+}{\eta}.
\]
Thus, if \(m_\theta\leq \tau_\rho-\eta\), the hinge is inactive and the
derivative of the forget surrogate with respect to all logits at this context is
zero. If \(m_\theta>\tau_\rho-\eta\), the hinge is active and
\[
    \psi_{\rho,\eta}(m_\theta)
    =
    \frac{m_\theta-(\tau_\rho-\eta)}{\eta}.
\]
Hence its derivative is \(1/\eta\) times the derivative of the margin. For \(v\in \mathcal{S}\), define the softmax weights over the competitor set
\[
    w_v
    =
    \frac{\exp(\beta z_\theta(v\mid c_t))}
    {\sum_{u\in \mathcal{S}}\exp(\beta z_\theta(u\mid c_t))}.
\]
Then, for any token \(a\in\mathcal V\),
\[
\frac{\partial \psi_{\rho,\eta}(m_\theta)}
{\partial z_\theta(a\mid c_t)}
=
\begin{cases}
0,
&
m_\theta\leq \tau_\rho-\eta,
\\[1.5mm]
\dfrac{\beta}{\eta},
&
m_\theta>\tau_\rho-\eta
\quad\text{and}\quad a=y_t,
\\[3mm]
-\dfrac{\beta}{\eta}w_a,
&
m_\theta>\tau_\rho-\eta
\quad\text{and}\quad a\in \mathcal{S},
\\[3mm]
0,
&
m_\theta>\tau_\rho-\eta
\quad\text{and}\quad a\notin \{y_t\}\cup \mathcal{S} .
\end{cases}
\]
Therefore, for an active forget token, gradient descent on the MASC forget term
decreases the gold logit \(z_\theta(y_t\mid c_t)\) and increases the logits of
the selected competitors \(z_\theta(v\mid c_t)\) for \(v\in \mathcal{S}\), with larger
updates for competitors that already have larger softmax weight within \(\mathcal{S}\).
Tokens outside \(\{y_t\}\cup \mathcal{S}\) receive no direct logit-level gradient from
this token-level term.

Since the full MASC forget loss averages this quantity over forget examples and
positions,
\[
    \Lfg^{\mathrm{MASC}}(\theta)
    =
    \mathbb{E}_{(x,y)\sim\Dforget}
    \left[
    \frac{1}{T}
    \sum_{t=1}^{T}
    \psi_{\rho,\eta}
    \!\left(
        m_{\theta}^{(k,\beta)}(x,y,t)
    \right)
    \right],
\]
its gradient is the corresponding average of the per-token contributions above.
In practice, this expectation is estimated by minibatches. The competitor set is
recomputed during training, so MASC behaves like a self-correcting active-set method: at each step, the gold token is challenged by the non-gold alternatives that the current model itself considers plausible.

\textbf{Remark.} The statement that logits outside \(\{y_t\}\cup \mathcal{S}\) have zero derivative refers to the direct derivative of the single token-level surrogate with respect to the logits at the current context. A parameter update can still affect other logits indirectly through the shared network parameters, and the retain regularizer adds its own gradient on retained
contexts.

\section{Additional Ablations}
\label{app:ablations}

\subsection{Robustness against quantization}
\label{app-subsec: robustness}
A growing line of work (for example \cite{lee2026distillation,fan2025towards,sheshadri2024latent,tamirisa2025tamperresistant}) shows that unlearned knowledge can often be recovered by simple post-processing or lightweight attacks on the unlearned model. In particular, quantization reveals a simple yet remarkable robustness failure of LLM unlearning, as shown by \citet{zhang2025catastrophic}. Indeed, applying low-bit quantization to an unlearned LLM can recover supposedly forgotten information, exposing a mismatch between full-precision unlearning metrics and robustness after deployment. In this sense, quantization is one of the easiest attacks on unlearning: it requires no access to the training pipeline, no carefully designed prompts, and no additional optimization over the forget set. Motivated by this observation, we evaluate whether MASC remains effective after 4-bit quantization. As shown in \Cref{tab:muse-4bit-forget}, MASC preserves low forget-set memorization after quantization on both MUSE News and MUSE Books. 

\begin{table}[h]
\centering
\scriptsize
\resizebox{\linewidth}{!}{
\begin{tabular}{lcccc}
\toprule
&
\multicolumn{2}{c}{MUSE News}
&
\multicolumn{2}{c}{MUSE Books} \\
\cmidrule(lr){2-3}
\cmidrule(lr){4-5}
Method
& VerbMem \(\Dforget\) \(\downarrow\)
& KnowMem \(\Dforget\) \(\downarrow\)
& VerbMem \(\Dforget\) \(\downarrow\)
& KnowMem \(\Dforget\) \(\downarrow\) \\
\midrule
Base (4-bit)
& 46.40 {\scriptsize (\(-10.85\))}
& 54.32 {\scriptsize (\(-12.13\))}
& 94.02 {\scriptsize (\(-5.68\))}
& 38.13 {\scriptsize (\(-8.99\))} \\

Retrain (4-bit)
& 20.06 {\scriptsize (\(-0.20\))}
& 35.35 {\scriptsize (\(+2.80\))}
& 14.00 {\scriptsize (\(-0.45\))}
& 24.41 {\scriptsize (\(-5.88\))} \\

\midrule
GA (4-bit)
& 0.00 {\scriptsize (\(+0.00\))}
& 0.00 {\scriptsize (\(+0.00\))}
& 0.00 {\scriptsize (\(+0.00\))}
& 0.00 {\scriptsize (\(+0.00\))} \\

GradDiff (4-bit)
& 7.36 {\scriptsize (\(+7.10\))}
& 47.28 {\scriptsize (\(+21.98\))}
& 0.00 {\scriptsize (\(+0.00\))}
& 29.04 {\scriptsize (\(+29.04\))} \\

NPO (4-bit)
& 14.61 {\scriptsize (\(+14.61\))}
& 29.41 {\scriptsize (\(+29.41\))}
& 4.70 {\scriptsize (\(+4.70\))}
& 3.53 {\scriptsize (\(+3.53\))} \\

NPO+KLR (4-bit)
& 39.35 {\scriptsize (\(+33.03\))}
& 52.19 {\scriptsize (\(+0.41\))}
& 49.17 {\scriptsize (\(+49.17\))}
& 35.96 {\scriptsize (\(+12.62\))} \\

RMU (4-bit)
& 21.77 {\scriptsize (\(-5.38\))}
& 36.80 {\scriptsize (\(-11.01\))}
& 8.37 {\scriptsize (\(-2.68\))}
& 11.16 {\scriptsize (\(-11.21\))} \\

SimNPO (4-bit)
& 37.58 {\scriptsize (\(+29.55\))}
& 48.82 {\scriptsize (\(+3.01\))}
& 71.27 {\scriptsize (\(+71.27\))}
& 33.99 {\scriptsize (\(+33.99\))} \\

\midrule
MASC (Ours, 4-bit)
& 5.79 {\scriptsize (\(+4.69\))}
& 28.80 {\scriptsize (\(+9.43\))}
& 0.87 {\scriptsize (\(-0.03\))}
& 25.48 {\scriptsize (\(-5.42\))} \\
\bottomrule
\end{tabular}
}
\vspace{1pt}
\caption{
MUSE forget metrics under 4-bit quantization. Parentheses report the change relative to the corresponding full-precision method in \Cref{tab:muse-main-results}, computed as
\(\Delta = \text{Model}_{\mathrm{4\text{-}bit}} - \text{Model}_{\mathrm{full}}\).
Both metrics are lower-is-better: negative values indicate that the 4-bit version reduces the forget metric, while positive values indicate worse forget-side performance relative to the full-precision model.
}
\label{tab:muse-4bit-forget}
\end{table}

\subsection{Timing without LoRA update}
\label{app-subsec:LoRA-timing}
MASC is implemented with LoRA adapters in our main experiments to reduce memory usage and wall-clock cost. To check whether the observed behavior is specific to this parameter-efficient implementation, we also run MASC with full finetuning on \texttt{TOFU}. As shown in \Cref{tab:lora-ablation}, full finetuning yields very similar forget--retain behavior to the LoRA implementation. The main difference is computational: full finetuning is slightly slower, while the evaluation metrics remain close. This suggests that LoRA mainly improves efficiency, rather than driving the empirical behavior of MASC. Importantly, even in the full-finetuning setting, MASC remains substantially faster than the strongest baselines with comparable forget--retain behavior in our experiments.

\begin{table}[h]
\centering
\begin{tabular}{lcccccc}
\toprule
\textbf{Config} 
& \(\mathbf{1-}\)\textbf{ROUGE-L} \(\uparrow\) 
& \(\mathbf{1-}\)\textbf{Prob} \(\uparrow\) 
& \textbf{Truth Ratio} \(\uparrow\) 
& \textbf{MU} \(\uparrow\) 
& \textbf{Time (sec)} \(\downarrow\) \\
\midrule
MASC-LoRA 
& 0.629
& 0.672
& 0.633
& 0.666
& 87.9 \\
MASC Full-FT 
& 0.609
& 0.608
& 0.660
& 0.647
& 140.6 \\
\bottomrule
\end{tabular}
\vspace{1em}
\caption{Impact of LoRA on MASC across metrics and time.}
\label{tab:lora-ablation}
\end{table}

\subsection{Stability of the stopping time $\tau_\alpha$} 
\label{app-subsec: stopping robustness}
We also check the stability of the MASC stopping rule across different random seeds. On \texttt{TOFU}, where five seeds are available, the stopping step is highly consistent across runs, with an average of \(80.4\) steps and standard deviation \(6.5\). The same pattern holds on \texttt{MUSE News} and \texttt{MUSE Books}: \(69.3\pm16.7\) for \texttt{MUSE News} and \(40.0\pm0.0\) for \texttt{MUSE Books}. These results suggest that the monitored violation rate yields a stable stopping criterion.

\subsection{Discussion on top-$k$ set of alternative tokens} 
\label{app-subsec:size-k}
The choice of \(k\) controls how many non-gold tokens are included in the MASC comparison set. Since the competitor term is a log-sum-exp,
\[
m_k = z_y - \log \sum_{r=1}^{k} e^{z_r},
\]
increasing \(k\) makes the competitor aggregate larger even if the model has not strongly changed the gold logit \(z_y\). Thus, for large \(k\), the margin criterion can be satisfied partly because many competitors are included, rather than because the original answer token has been strongly suppressed. To test this explanation, we measured
\[
\Delta z_y = z_y^{\mathrm{after}} - z_y^{\mathrm{before}},
\]
the average change in the logit assigned to the original answer token on TOFU forget continuations, comparing the unlearned model to the base model. More negative values indicate stronger suppression of the original answer token. As shown below, increasing \(k\) leads to a smaller decrease in \(z_y\), earlier stopping, and worse forgetting metrics, while MU remains almost unchanged:

\begin{table}[h]
\centering
\scriptsize
\resizebox{\linewidth}{!}{
\begin{tabular}{c|cccccc}
\toprule
\(k\)
& \(\Delta z_y\)
& Stop step
& \(1-\)ROUGE-L \(\uparrow\)
& \(1-\)Prob \(\uparrow\)
& Truth Ratio \(\uparrow\)
& MU \(\uparrow\) \\
\midrule
10   & -14.29 & 82 & 0.700 & 0.719 & 0.636 & 0.666 \\
100  & -13.24 & 76 & 0.493 & 0.604 & 0.617 & 0.665 \\
1000 & -12.93 & 74 & 0.365 & 0.539 & 0.610 & 0.662 \\
\bottomrule
\end{tabular}
}
\vspace{2pt}
\caption{
TOFU ablation of the top-\(k\) comparison set in MASC.
}
\label{tab:masc-topk-ablation}
\end{table}

These results suggest that larger comparison sets make the stopping criterion easier to satisfy without requiring as much direct suppression of the original answer tokens. This explains why forgetting becomes weaker as \(k\) grows, even though retain utility remains stable.

\section{Experimental Details \& Additional Metrics}
\label{app:exp-details}

\subsection{Full metrics \& Pareto Frontier}
We now report other available metrics (cf. \Cref{tab:tofu-additional-metrics,tab:muse-privacy-leak}) for each of the studied datasets, together with a visualization of the forget--retain trade-off based on \Cref{tab:tofu-main-results,tab:muse-main-results} metrics.

\begin{table}[h]
\centering
\small
\resizebox{\linewidth}{!}{
\begin{tabular}{lccccc}
\toprule
&
\multicolumn{1}{c}{Unlearning Privacy}
& \multicolumn{4}{c}{Retain Utility} \\
\cmidrule(lr){2-2}
\cmidrule(lr){3-6}
Method
& FQ \(\uparrow\)
& MU \(\uparrow\)
& Retain ROUGE \(\uparrow\)
& Retain Prob \(\uparrow\)
& Retain TR \(\uparrow\) \\
\midrule
Base/full
& 0
& 0.628
& 0.981
& 0.989
& 0.460 \\

Retrain
& 1
& 0.613
& 0.976
& 0.989
& 0.457 \\
\midrule

GA
& \(1.72{\times}10^{-17}\) {\scriptsize [\(1.59{\times}10^{-17}\)]}
& 0.459 {\scriptsize [0.014]}
& 0.732 {\scriptsize [0.024]}
& 0.186 {\scriptsize [0.024]}
& 0.455 {\scriptsize [0.005]} \\

GradDiff
& \(3.30{\times}10^{-18}\) {\scriptsize [\(4.89{\times}10^{-18}\)]}
& 0.561 {\scriptsize [0.005]}
& 0.556 {\scriptsize [0.019]}
& 0.739 {\scriptsize [0.011]}
& 0.464 {\scriptsize [0.003]} \\

NPO
& \(2.14{\times}10^{-18}\) {\scriptsize [\(2.56{\times}10^{-18}\)]}
& 0.533 {\scriptsize [0.003]}
& 0.713 {\scriptsize [0.030]}
& 0.403 {\scriptsize [0.011]}
& 0.430 {\scriptsize [0.009]} \\

NPO+KLR
& \(2.99{\times}10^{-18}\) {\scriptsize [\(4.95{\times}10^{-18}\)]}
& 0.516 {\scriptsize [0.006]}
& 0.722 {\scriptsize [0.019]}
& 0.342 {\scriptsize [0.006]}
& 0.434 {\scriptsize [0.005]} \\

RMU
& \(1.80{\times}10^{-23}\) {\scriptsize [\(4.77{\times}10^{-24}\)]}
& 0.618 {\scriptsize [0.001]}
& 0.901 {\scriptsize [0.006]}
& 0.876 {\scriptsize [0.011]}
& 0.455 {\scriptsize [0.001]} \\

SimNPO
& \(5.27{\times}10^{-24}\) {\scriptsize [\(4.03{\times}10^{-24}\)]}
& 0.614 {\scriptsize [0.001]}
& 0.976 {\scriptsize [0.005]}
& 0.992 {\scriptsize [0.001]}
& 0.464 {\scriptsize [0.001]} \\

\midrule
MASC (Ours)
& \(3.01{\times}10^{-7}\) {\scriptsize [\(6.50{\times}10^{-7}\)]}
& 0.666 {\scriptsize [0.003]}
& 0.899 {\scriptsize [0.019]}
& 0.832 {\scriptsize [0.033]}
& 0.446 {\scriptsize [0.005]} \\

\bottomrule
\end{tabular}
}
\vspace{2pt}
\caption{
Additional TOFU metrics. FQ measures forget-side privacy leakage, while MU, Retain ROUGE, Retain Prob, and Retain TR measure retain-side utility. Results are averaged over seeds, with standard deviations in brackets. FQ is bounded in $[0,1]$ and the higher the better.
}
\label{tab:tofu-additional-metrics}
\end{table}

\begin{table}[t]
\centering
\begin{tabular}{lcc}
\toprule
Method & MUSE News \(\downarrow\) & MUSE Books \(\downarrow\) \\
\midrule
Base       & -99.81 & -57.34 \\
Retrain    & -4.72  & 8.16 \\
\midrule
GA         & 5.22   & -28.64 \\
GradDiff   & 105.16 & -29.06 \\
NPO        & 14.99  & -22.31 \\
NPO+KLR    & 87.03  & -42.74 \\
RMU        & -99.73 & -23.75 \\
SimNPO     & 35.26  & -17.86 \\
\midrule
MASC (Ours)& 41.99  & -48.89 {\scriptsize} \\
\bottomrule
\end{tabular}
\vspace{2pt}
\caption{
Privacy-leak metrics on MUSE.
}
\label{tab:muse-privacy-leak}
\end{table}

\begin{figure}[h] 
    \centering 
    \resizebox{\linewidth}{!}{\input{plots/forget_retain_pareto.pgf}} 
    \caption{Pareto Frontier computed from the metrics of \Cref{tab:tofu-main-results,tab:muse-main-results} where top-right is better. For the forget metrics, we average the reported metrics to get an aggregate forget score. Retrain and base models are not timed since they are considered as given/oracle.} 
\label{fig:pareto-fronts} 
\end{figure}
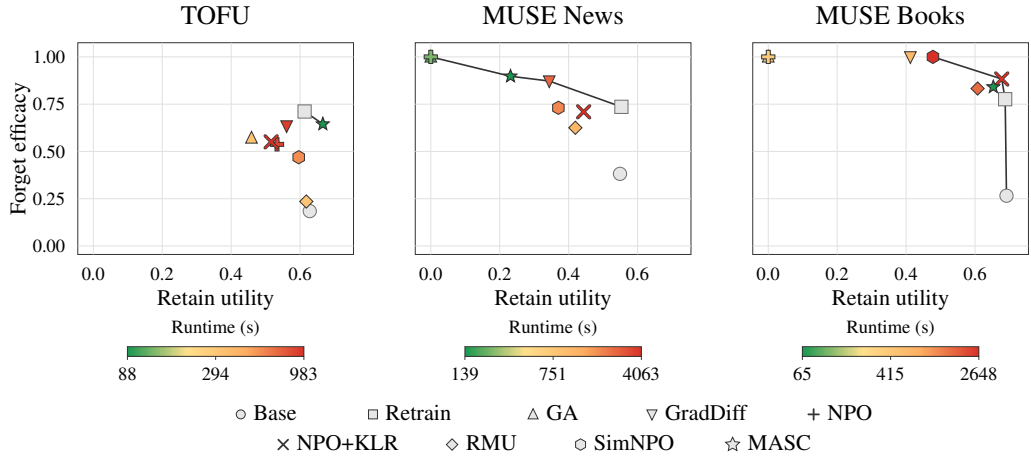

\paragraph{Baselines.} For each baseline, we use the authors' official implementation whenever available. When multiple public implementations are available, including benchmark-suite versions, we select the implementation that achieves the strongest time--metric trade-off in our setup. We set hyperparameters according to the corresponding papers and released code, using author-recommended configurations whenever possible. This protocol is intended to give each baseline a competitive configuration rather than comparing against under-tuned variants.

\paragraph{A comment on privacy metrics.}
We report FQ on TOFU and privacy leakage on MUSE as privacy-oriented diagnostics rather than as primary forget--retain metrics, following the classification of \citet{dorna2026openunlearning}. Both quantities rely on information about the behavior of a retrained or non-member reference distribution: FQ is a hypothesis-test p-value comparing the truth-ratio distribution of the unlearned model to that of the retain-only retrained model, while privacy leakage measures residual membership-style distinguishability of forgotten examples. Thus, unlike ROUGE, likelihood, or model utility, these metrics ask whether the unlearned model is statistically indistinguishable from an ideal deletion baseline, not only whether it stops reproducing the forgotten content. As also noted in Remark B.1 of \citet{entesari2026constrained}, such metrics are informative but imperfect: they require access to retrained/reference behavior and can be hard to interpret when models collapse or move away from the retrain distribution for reasons unrelated to memorization. In our experiments, privacy-oriented metrics also do not reflect good privacy guarantees, suggesting that current approximate unlearning methods should not be interpreted as providing seed-stable privacy guarantees. We therefore view private-unlearning as an important direction for future work.

\subsection{Hyperparameters}
\paragraph{MASC hyperparameters.}
\Cref{tab:masc-hyperparameters} reports the MASC hyperparameters used in the main experiments. Across all datasets, we keep the backbone frozen and train LoRA adapters only, using a retain-side KL penalty to the base model.

\begin{table}[h]
\centering
\scriptsize
\resizebox{\linewidth}{!}{
\begin{tabular}{lcccccccc}
\toprule
Dataset
& \(\lambda_{\mathrm{fg}}\)
& \(\rho\)
& \(\eta\)
& top-\(k\)
& \(\beta\)
& Stop \(\alpha\)
& LR
& LoRA (rank) \\
\midrule
TOFU
& 0.05
& 0.70
& 0.25
& \(k=10\)
& 1.0
& 0.475
& \(10^{-4}\)
& 16 \\

MUSE News
& 0.50
& 0.70
& 0.50
& \(k=2\)
& 5.0
& 0.55
& \(10^{-4}\)
& 16 \\

MUSE Books
& 0.05
& 0.50
& 0.50
& \(k=10\)
& 1.0
& 0.10
& \(10^{-4}\)
& 16 \\
\bottomrule
\end{tabular}
}
\vspace{2pt}
\caption{
MASC hyperparameters used in the main experiments. Here \(\lambda_{\mathrm{fg}}\) is the weight of the forget loss, \(\rho\) is the local dominance threshold, \(\eta\) is the hinge buffer, \(\beta\) is the logit-temperature parameter, and \(\alpha\) is the stopping tolerance.
}
\label{tab:masc-hyperparameters}
\end{table}

\paragraph{Effect of the learning rate.}
\Cref{fig:stopping_lr_ablation} shows the evolution of the stopping statistic \(\widehat V_\rho\) for different learning rates on TOFU. As expected, larger learning rates drive the violation rate below the tolerance \(\alpha\) in fewer optimizer steps, yielding faster stopping. Smaller learning rates decrease \(\widehat V_\rho\) more gradually, which is slower but provides a finer resolution along the MASC trajectory: more intermediate checkpoints are available around the stopping threshold, allowing more controlled selection of the forget--retain trade-off.

\begin{figure}[t] 
    \centering 
    \resizebox{0.7\linewidth}{!}{\input{plots/tofu_masc_lr.pgf}} \caption{Effect of the learning rate on the MASC stopping statistic \(\widehat V_\rho\) on TOFU.} 
\label{fig:stopping_lr_ablation} 
\end{figure}
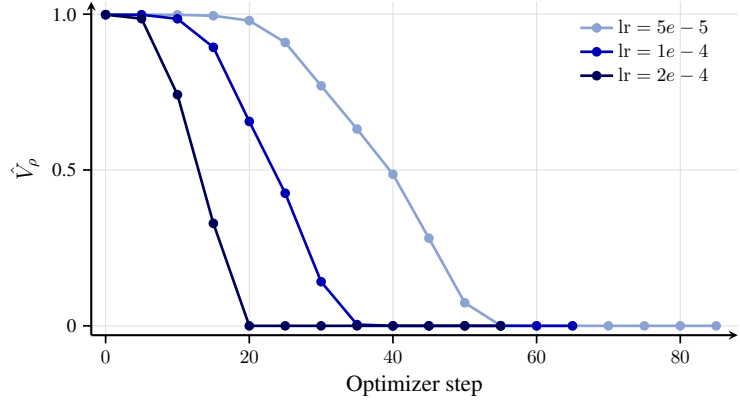

\paragraph{Top-$k$ alternatives.} The parameter \(k\) controls how broad this local comparison is. Small values of \(k\) compare the target token only against the most competitive alternatives, making the condition close to a ``target versus nearest rivals'' test. Larger values include more alternatives and therefore require the target token to share probability mass with a broader candidate set. In our experiments, we use a small \(k\) value ($=10$) so that the forget update remains focused on plausible replacements rather than on the full vocabulary, which contains many irrelevant tokens.

\subsection{Empirical Scaling Laws}
\label{app-subsec:scaling-fit}

\paragraph{Power-law fitting procedure.}
We fit scaling trends using the same procedure for both the learning and unlearning stages. For each metric, model size, and random seed, let \(s_r(N)\) denote the measured score, where \(N\) is the number of model parameters (expressed in billions) and \(r\in\{1,\ldots,R\}\) indexes the seed. We first average over seeds at fixed model size,
\[
    \bar s(N)
    =
    \frac{1}{R}
    \sum_{r=1}^{R}
    s_r(N).
\]
We then fit a power-law model \(\bar s(N)=C N^\alpha\), with \(C>0\), by ordinary least squares in log--log space:
\[
    \log \bar s(N_i)
    =
    \log C
    +
    \alpha \log N_i
    +
    \varepsilon_i,
\]
over the evaluated model sizes \(N_i\). The slope gives the scaling exponent \(\alpha\), while the intercept gives \(\log C\). The reported \(R^2\) is computed in log-space as
\[
    R^2
    =
    1
    -
    \frac{\sum_i \left(\log \bar s(N_i)-\log \widehat s(N_i)\right)^2}
    {\sum_i \left(\log \bar s(N_i)-\frac{1}{m}\sum_j \log \bar s(N_j)\right)^2},
\]
where \(\widehat s(N_i)=\widehat C N_i^{\widehat\alpha}\) is the fitted value and \(m\) is the number of evaluated model sizes.

\paragraph{Learning-stage reference scores.}
Before task finetuning, the base models already exhibit nonzero scores on several TOFU metrics. We report these base-model scores in \Cref{tab:base-model-scaling-scores} to make clear that the scaling trends in the main text refer to the additional memorization induced by supervised finetuning on the benchmark data.

\begin{table}[h]
\centering
\begin{tabular}{lcccc}
\toprule
Metric & 0.5B & 1.5B & 3B & 7B \\
\midrule
ES      & 0.061 & 0.078 & 0.042 & 0.039 \\
E-ROUGE & 0.198 & 0.193 & 0.384 & 0.425 \\
P-ROUGE & 0.198 & 0.167 & 0.362 & 0.405 \\
\bottomrule
\end{tabular}
\vspace{2pt}
\caption{
Initial-model TOFU scores before supervised finetuning on the benchmark data. These values provide a reference point for interpreting the learning-stage scaling trends.
}
\label{tab:base-model-scaling-scores}
\end{table}

\paragraph{Unlearning-stage fits.}
\Cref{tab:unlearning-scaling-fits} reports the fitted power-law parameters after unlearning for MASC and SimNPO. These fits should be interpreted differently for forget-side metrics and retain utility. For the forget-side metrics, several \(R^2\) values are low, especially for MASC, indicating that these scores do not follow a clear monotone power law over the evaluated model sizes (cf. \Cref{fig:unlearning-scaling}). The main observation is therefore not a strong scaling law, but rather a stability pattern: after unlearning, forget-side scores fluctuate across scales while remaining in a comparable range on average. In contrast, retain utility shows a clearer positive trend for both methods, suggesting that larger models preserve useful behavior better after unlearning while residual forget-side performance does not systematically increase with scale.

\begin{table}[h]
\centering
\begin{tabular}{llccc}
\toprule
Method & Metric & \(\hat C\) & \(\hat \alpha\) & \(R^2\) \\
\midrule
\multirow{3}{*}{SimNPO}
& E-ROUGE & 0.3797 & -0.0036 & 0.342 \\
& P-ROUGE & 0.3360 & 0.0190 & 0.454 \\
& MU on \(\Dretain\) & 0.4656 & 0.1489 & 0.988 \\
\midrule
\multirow{3}{*}{MASC}
& E-ROUGE & 0.3931 & 0.0183 & 0.414 \\
& P-ROUGE & 0.3145 & -0.0411 & 0.357 \\
& MU on \(\Dretain\) & 0.4451 & 0.1572 & 0.964 \\
\bottomrule
\end{tabular}
\vspace{2pt}
\caption{
Unlearning-stage scaling fits \(s(N)=C N^\alpha\) across model sizes, fitted in log-log space.
}
\label{tab:unlearning-scaling-fits}
\end{table}

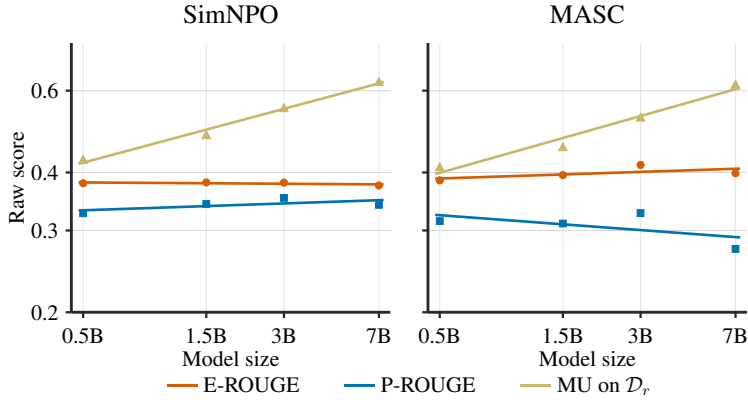
\begin{figure}[h]
    \vspace{-1em}
    \centering
    \resizebox{0.7\linewidth}{!}{\input{plots/scaling_unlearning.pgf}}
    \vspace{-3em}
    \caption{Fitted scaling trends for MASC and SimNPO after unlearning.}
    \label{fig:unlearning-scaling}
\end{figure}

\clearpage

\section{MASC pseudo-code}

\begin{algorithm}[hb]
\caption{\textsc{MASC}}
\label{alg:masc}
\begin{algorithmic}[1]
\Statex \textbf{Input:} base model \(\pi_{\theta_0}\), forget set \(\Dforget\), retain set \(\Dretain\)
\Statex \textbf{Hyperparameters:} \(k\), \(\beta\), \(\rho\), \(\eta\), \(\alpha\), probe size \(n_p\), \(\lambda_{\mathrm{fg}}\).

\State Initialize LoRA parameters \(\phi\) and set \(\theta=(\theta_0,\phi)\) \Comment{backbone frozen}
\State Set \(\tau_\rho \gets \log(\rho/(1-\rho))\) \Comment{probability threshold in margin form}

\For{unlearning step \(s=1,2,\ldots\)}
    \State Sample forget batch \(B_\mathrm{fg}\subset \Dforget\) and retain batch \(B_\mathrm{ret}\subset \Dretain\)
    \ForAll{\((x,y)\in B_\mathrm{fg}\) and active answer positions \(t\in A(x,y)\)}
        \State \(c_t \gets (x,y_{<t})\) \Comment{teacher-forced context}
        \State \(\mathcal{S}_{\theta,k}(c_t)\gets\) top-\(k\) non-gold tokens under \(\pi_\theta(\cdot\mid c_t)\)
        \State
        \[
        m_t \gets
        \beta z_\theta(y_t\mid c_t)
        -
        \log\sum_{v\in \mathcal{S}_{\theta,k}(c_t)}
        \exp\!\big(\beta z_\theta(v\mid c_t)\big)
        \]
        \State \(\ell_t^{\mathrm{MASC}}
        \gets
        \big[m_t-(\tau_\rho-\eta)\big]_+/\eta\)
        \Comment{MASC loss}
    \EndFor

    \State
    \[
    \mathcal L_\mathrm{fg}
    \gets
    \frac{1}{|B_\mathrm{fg}|}
    \sum_{(x,y)\in B_\mathrm{fg}}
    \frac{1}{|A(x,y)|}
    \sum_{t\in A(x,y)}
    \ell_t^{\mathrm{MASC}}
    \]
    \State
    \[
    \mathcal L_\mathrm{ret}
    \gets
    \mathbb E_{(x,y)\in B_\mathrm{ret}}
    \frac{1}{T}
    \sum_{t=1}^{T}
    \mathrm{KL}\!\left(
    \pi_{\theta_0}(\cdot\mid x,y_{<t})
    \,\middle\|\,
    \pi_{\theta}(\cdot\mid x,y_{<t})
    \right)
    \]

    \State Update LoRA parameters with
    \[
        \mathcal L
        =
        \lambda_\mathrm{fg} L_\mathrm{fg}
        +
        \lambda_\mathrm{ret} L_\mathrm{ret}
    \]
    \If{\(s\) is a probe step}
        \State Sample \(\mathcal P_\mathrm{fg}\subset \Dforget\setminus B_\mathrm{fg}\) uniformly at random with \(|\mathcal P_\mathrm{fg}|=n_p\)
        \State Estimate
        \[
        \widehat V_\rho(\theta)
        =
        \frac{1}{|\mathcal P_\mathrm{fg}|}
        \sum_{(x,y)\in\mathcal P_\mathrm{fg}}
        \frac{1}{|A(x,y)|}
        \sum_{t\in A(x,y)}
        \mathbf 1\{m_t>\tau_\rho\}
        \]
        \If{\(\widehat V_\rho(\theta)\le \alpha\)}
            \State \Return \(\theta\)
        \EndIf
    \EndIf
\EndFor
\end{algorithmic}
\end{algorithm}

\clearpage

\section{Example Q\&A Responses}
\label{app:examples}

In addition to aggregate memorization metrics, we inspect model generations on individual forget-set questions. 
\Cref{tab:qualitative-examples} reports representative examples from TOFU, MUSE News, and MUSE Books. Across these cases, MASC typically avoids reproducing the exact forgotten answer, while still producing fluent and task-relevant completions. 

\begin{table*}[hb]
\centering
\small
\setlength{\tabcolsep}{3pt}
\renewcommand{\arraystretch}{1.18}

\definecolor{questionblue}{HTML}{EAF2FF}
\definecolor{goldgreen}{HTML}{EAF7EA}
\definecolor{mascred}{HTML}{FFF0F0}

\newcommand{\qcell}[1]{\cellcolor{questionblue}#1}
\newcommand{\gcell}[1]{\cellcolor{goldgreen}\textbf{#1}}

\resizebox{0.98\textwidth}{!}{
\begin{tabularx}{\textwidth}{
>{\centering\arraybackslash}p{0.075\textwidth}|
p{0.365\textwidth}|
p{0.195\textwidth}|
>{\columncolor{mascred}\raggedright\arraybackslash}X@{}
}
\toprule
\textbf{Dataset}
& \textbf{Question}
& \textbf{Gold answer}
& \multicolumn{1}{|>{\raggedright\arraybackslash}X@{}}{\textbf{MASC answer}} \\
\datasetsep

\multirowcell{4}{\textsc{TOFU}}
& \qcell{What is the full name of the geology author born in Karachi, Pakistan on 06/30/1975?}
& \gcell{The author's name is Hina Ameen.}
& The full name is Aahan Ameerah. \\
\cmidrule(lr){2-4}

& \qcell{Were any books by Aysha Al-Hashim made into films or TV series?}
& \gcell{Her novel ``The Matrimony Plan'' is under negotiation for a film adaptation.}
& None of the books have been adapted into films or TV series. \\
\cmidrule(lr){2-4}

& \qcell{Are Kalkidan Abera's works used for academic or educational purposes?}
& \gcell{Her health books are used in health-science courses.}
& Her books ``The Ethereal Enigma'' and ``Corpus Ooze'' are studied in biology, literature, and anthropology departments. \\
\cmidrule(lr){2-4}

& \qcell{Who is the notable author born in Astana, Kazakhstan on Feb 7, 1952?}
& \gcell{Nikolai Abilov.}
& Eulalie Skat. \\

\datasetsep

\multirowcell{5}{\textsc{News}}
& \qcell{Who wrote ``The Fields Of Athenry''?}
& \gcell{Pete St John}
& Irish writer JohnB O'Hara \\
\cmidrule(lr){2-4}

& \qcell{Who was the Rome tour guide who described conditions as ``nightmarish''?}
& \gcell{Felicity Hinton}
& Alessandro, who has been working in Rome since 2017 \\
\cmidrule(lr){2-4}

& \qcell{What was Nicola Bulley's cause of death?}
& \gcell{drowned after falling into cold water}
& taken her own life \\
\cmidrule(lr){2-4}

& \qcell{How many Turkish parliament seats were secured by the AK Party and MHP?}
& \gcell{322 seats}
& 292 and 229 respectively \\
\cmidrule(lr){2-4}

& \qcell{How many candidates stood in Selby and Ainsty?}
& \gcell{13 candidates}
& 12 candidates \\

\datasetsep

\multirowcell{3}{\textsc{Books}}
& \qcell{Which bathroom did Cedric tell Harry to use to help figure out the clue for the golden egg?}
& \gcell{the prefects' bathroom}
& The one with the tapestry of Barnabas the Barmy \\
\cmidrule(lr){2-4}

& \qcell{What is the animal form that Rita Skeeter transforms into as an unregistered Animagus?}
& \gcell{a beetle}
& It is unknown, likely a hybrid of some sort. \\
\cmidrule(lr){2-4}

& \qcell{What headline did Harry read about himself in the Daily Prophet?}
& \gcell{\textsc{Harry Potter Disturbed and Dangerous}}
& \textsc{Harry Potter's Giant Egg-Faced Menace} \\

\bottomrule
\end{tabularx}
}
\vspace{2pt}
\caption{
Qualitative examples of MASC generations on forget-set questions from TOFU, MUSE News, and MUSE Books.
}
\label{tab:qualitative-examples}
\end{table*}


\end{document}

%% file: plots/runtimes.pgf
\begin{tikzpicture}[x=1cm,y=1cm]
\definecolor{NPOBlue}{HTML}{2F5D8A}
\definecolor{GAGreen}{HTML}{3A9363}
\definecolor{SimOrange}{HTML}{D17C2F}
\definecolor{MASCRed}{HTML}{C84545}
\definecolor{GridGrey}{HTML}{D3D3D3}
\definecolor{AxisDark}{HTML}{151515}
\path[use as bounding box] (-1.800,-0.900) rectangle (22.100,12.150);
\node[font=\Large] at (2.850,11.780) {TOFU};
\draw[AxisDark,line width=1.15pt] (0.000,6.850) -- (5.700,6.850);
\draw[AxisDark,line width=1.15pt] (0.000,6.850) -- (0.000,11.100);
\draw[GridGrey,line width=0.45pt] (0.000,6.850) -- (5.700,6.850);
\node[anchor=east,font=\normalsize] at (-0.100,6.850) {0};
\draw[GridGrey,line width=0.45pt] (0.000,7.659) -- (5.700,7.659);
\node[anchor=east,font=\normalsize] at (-0.100,7.659) {1000};
\draw[GridGrey,line width=0.45pt] (0.000,8.468) -- (5.700,8.468);
\node[anchor=east,font=\normalsize] at (-0.100,8.468) {2000};
\draw[GridGrey,line width=0.45pt] (0.000,9.276) -- (5.700,9.276);
\node[anchor=east,font=\normalsize] at (-0.100,9.276) {3000};
\draw[GridGrey,line width=0.45pt] (0.000,10.085) -- (5.700,10.085);
\node[anchor=east,font=\normalsize] at (-0.100,10.085) {4000};
\draw[GridGrey,line width=0.45pt] (0.000,10.894) -- (5.700,10.894);
\node[anchor=east,font=\normalsize] at (-0.100,10.894) {5000};
\fill[NPOBlue] (0.340,6.850) rectangle (1.220,7.645);
\draw[AxisDark,line width=0.8pt] (0.780,7.606) -- (0.780,7.685);
\draw[AxisDark,line width=0.8pt] (0.630,7.685) -- (0.930,7.685);
\draw[AxisDark,line width=0.8pt] (0.630,7.606) -- (0.930,7.606);
\fill[GAGreen] (1.720,6.850) rectangle (2.600,7.584);
\draw[AxisDark,line width=0.8pt] (2.160,7.448) -- (2.160,7.720);
\draw[AxisDark,line width=0.8pt] (2.010,7.720) -- (2.310,7.720);
\draw[AxisDark,line width=0.8pt] (2.010,7.448) -- (2.310,7.448);
\fill[SimOrange] (3.100,6.850) rectangle (3.980,7.288);
\draw[AxisDark,line width=0.8pt] (3.540,7.258) -- (3.540,7.319);
\draw[AxisDark,line width=0.8pt] (3.390,7.319) -- (3.690,7.319);
\draw[AxisDark,line width=0.8pt] (3.390,7.258) -- (3.690,7.258);
\fill[MASCRed] (4.480,6.850) rectangle (5.360,6.921);
\draw[AxisDark,line width=0.8pt] (4.920,6.915) -- (4.920,6.928);
\draw[AxisDark,line width=0.8pt] (4.770,6.928) -- (5.070,6.928);
\draw[AxisDark,line width=0.8pt] (4.770,6.915) -- (5.070,6.915);
\node[font=\Large] at (10.150,11.780) {MUSE News};
\draw[AxisDark,line width=1.15pt] (7.300,6.850) -- (13.000,6.850);
\draw[AxisDark,line width=1.15pt] (7.300,6.850) -- (7.300,11.100);
\draw[GridGrey,line width=0.45pt] (7.300,6.850) -- (13.000,6.850);
\node[anchor=east,font=\normalsize] at (7.200,6.850) {0};
\draw[GridGrey,line width=0.45pt] (7.300,7.659) -- (13.000,7.659);
\node[anchor=east,font=\normalsize] at (7.200,7.659) {1000};
\draw[GridGrey,line width=0.45pt] (7.300,8.468) -- (13.000,8.468);
\node[anchor=east,font=\normalsize] at (7.200,8.468) {2000};
\draw[GridGrey,line width=0.45pt] (7.300,9.276) -- (13.000,9.276);
\node[anchor=east,font=\normalsize] at (7.200,9.276) {3000};
\draw[GridGrey,line width=0.45pt] (7.300,10.085) -- (13.000,10.085);
\node[anchor=east,font=\normalsize] at (7.200,10.085) {4000};
\draw[GridGrey,line width=0.45pt] (7.300,10.894) -- (13.000,10.894);
\node[anchor=east,font=\normalsize] at (7.200,10.894) {5000};
\fill[NPOBlue] (7.640,6.850) rectangle (8.520,10.136);
\draw[AxisDark,line width=0.8pt] (8.080,9.627) -- (8.080,10.645);
\draw[AxisDark,line width=0.8pt] (7.930,10.645) -- (8.230,10.645);
\draw[AxisDark,line width=0.8pt] (7.930,9.627) -- (8.230,9.627);
\fill[GAGreen] (9.020,6.850) rectangle (9.900,8.886);
\draw[AxisDark,line width=0.8pt] (9.460,8.829) -- (9.460,8.944);
\draw[AxisDark,line width=0.8pt] (9.310,8.944) -- (9.610,8.944);
\draw[AxisDark,line width=0.8pt] (9.310,8.829) -- (9.610,8.829);
\fill[SimOrange] (10.400,6.850) rectangle (11.280,8.369);
\draw[AxisDark,line width=0.8pt] (10.840,8.354) -- (10.840,8.384);
\draw[AxisDark,line width=0.8pt] (10.690,8.384) -- (10.990,8.384);
\draw[AxisDark,line width=0.8pt] (10.690,8.354) -- (10.990,8.354);
\fill[MASCRed] (11.780,6.850) rectangle (12.660,6.962);
\draw[AxisDark,line width=0.8pt] (12.220,6.932) -- (12.220,6.992);
\draw[AxisDark,line width=0.8pt] (12.070,6.992) -- (12.370,6.992);
\draw[AxisDark,line width=0.8pt] (12.070,6.932) -- (12.370,6.932);
\node[font=\Large] at (17.450,11.780) {MUSE Books};
\draw[AxisDark,line width=1.15pt] (14.600,6.850) -- (20.300,6.850);
\draw[AxisDark,line width=1.15pt] (14.600,6.850) -- (14.600,11.100);
\draw[GridGrey,line width=0.45pt] (14.600,6.850) -- (20.300,6.850);
\node[anchor=east,font=\normalsize] at (14.500,6.850) {0};
\draw[GridGrey,line width=0.45pt] (14.600,7.659) -- (20.300,7.659);
\node[anchor=east,font=\normalsize] at (14.500,7.659) {1000};
\draw[GridGrey,line width=0.45pt] (14.600,8.468) -- (20.300,8.468);
\node[anchor=east,font=\normalsize] at (14.500,8.468) {2000};
\draw[GridGrey,line width=0.45pt] (14.600,9.276) -- (20.300,9.276);
\node[anchor=east,font=\normalsize] at (14.500,9.276) {3000};
\draw[GridGrey,line width=0.45pt] (14.600,10.085) -- (20.300,10.085);
\node[anchor=east,font=\normalsize] at (14.500,10.085) {4000};
\draw[GridGrey,line width=0.45pt] (14.600,10.894) -- (20.300,10.894);
\node[anchor=east,font=\normalsize] at (14.500,10.894) {5000};
\fill[NPOBlue] (14.940,6.850) rectangle (15.820,8.929);
\draw[AxisDark,line width=0.8pt] (15.380,8.599) -- (15.380,9.258);
\draw[AxisDark,line width=0.8pt] (15.230,9.258) -- (15.530,9.258);
\draw[AxisDark,line width=0.8pt] (15.230,8.599) -- (15.530,8.599);
\fill[GAGreen] (16.320,6.850) rectangle (17.200,7.313);
\draw[AxisDark,line width=0.8pt] (16.760,7.298) -- (16.760,7.328);
\draw[AxisDark,line width=0.8pt] (16.610,7.328) -- (16.910,7.328);
\draw[AxisDark,line width=0.8pt] (16.610,7.298) -- (16.910,7.298);
\fill[SimOrange] (17.700,6.850) rectangle (18.580,8.992);
\draw[AxisDark,line width=0.8pt] (18.140,8.955) -- (18.140,9.029);
\draw[AxisDark,line width=0.8pt] (17.990,9.029) -- (18.290,9.029);
\draw[AxisDark,line width=0.8pt] (17.990,8.955) -- (18.290,8.955);
\fill[MASCRed] (19.080,6.850) rectangle (19.960,6.903);
\draw[AxisDark,line width=0.8pt] (19.520,6.902) -- (19.520,6.903);
\draw[AxisDark,line width=0.8pt] (19.370,6.903) -- (19.670,6.903);
\draw[AxisDark,line width=0.8pt] (19.370,6.902) -- (19.670,6.902);
\node[rotate=90,font=\Large] at (-1.420,8.975) {Wall-clock time (seconds)};
\draw[GridGrey,line width=0.55pt] (2.850,2.550) circle[radius=0.430];
\draw[GridGrey,line width=0.55pt] (2.850,2.550) circle[radius=0.860];
\draw[GridGrey,line width=0.55pt] (2.850,2.550) circle[radius=1.290];
\draw[GridGrey,line width=0.55pt] (2.850,2.550) circle[radius=1.720];
\node[font=\scriptsize,text=black!55] at (2.983,2.959) {0.25};
\node[font=\scriptsize,text=black!55] at (3.116,3.368) {0.50};
\node[font=\scriptsize,text=black!55] at (3.249,3.777) {0.75};
\node[font=\scriptsize,text=black!55] at (3.382,4.186) {1.00};
\draw[GridGrey,line width=0.55pt] (2.850,2.550) -- (2.850,4.270);
\node[anchor=south,font=\normalsize,align=center] at (2.850,4.430) {1-ROUGE-L\\($D_f$)};
\draw[GridGrey,line width=0.55pt] (2.850,2.550) -- (4.570,2.550);
\node[anchor=west,font=\normalsize,align=center] at (4.730,2.550) {1-Prob\\($D_f$)};
\draw[GridGrey,line width=0.55pt] (2.850,2.550) -- (2.850,0.830);
\node[anchor=north,font=\normalsize,align=center] at (2.850,0.670) {Truth Ratio\\($D_f$)};
\draw[GridGrey,line width=0.55pt] (2.850,2.550) -- (1.130,2.550);
\node[anchor=east,font=\normalsize,align=center] at (0.970,2.550) {MU\\($D_r$)};
\fill[NPOBlue,opacity=0.07] (2.850,3.173) -- (4.076,2.550) -- (2.850,1.558) -- (1.962,2.550) -- cycle;
\draw[NPOBlue,line width=1.05pt] (2.850,3.173) -- (4.076,2.550) -- (2.850,1.558) -- (1.962,2.550) -- cycle;
\fill[NPOBlue] (2.850,3.173) circle[radius=0.035];
\fill[NPOBlue] (4.076,2.550) circle[radius=0.035];
\fill[NPOBlue] (2.850,1.558) circle[radius=0.035];
\fill[NPOBlue] (1.962,2.550) circle[radius=0.035];
\fill[GAGreen,opacity=0.07] (2.850,3.579) -- (4.212,2.550) -- (2.850,1.666) -- (1.885,2.550) -- cycle;
\draw[GAGreen,line width=1.05pt] (2.850,3.579) -- (4.212,2.550) -- (2.850,1.666) -- (1.885,2.550) -- cycle;
\fill[GAGreen] (2.850,3.579) circle[radius=0.035];
\fill[GAGreen] (4.212,2.550) circle[radius=0.035];
\fill[GAGreen] (2.850,1.666) circle[radius=0.035];
\fill[GAGreen] (1.885,2.550) circle[radius=0.035];
\fill[SimOrange,opacity=0.07] (2.850,3.150) -- (3.705,2.550) -- (2.850,1.583) -- (1.825,2.550) -- cycle;
\draw[SimOrange,line width=1.05pt] (2.850,3.150) -- (3.705,2.550) -- (2.850,1.583) -- (1.825,2.550) -- cycle;
\fill[SimOrange] (2.850,3.150) circle[radius=0.035];
\fill[SimOrange] (3.705,2.550) circle[radius=0.035];
\fill[SimOrange] (2.850,1.583) circle[radius=0.035];
\fill[SimOrange] (1.825,2.550) circle[radius=0.035];
\fill[MASCRed,opacity=0.07] (2.850,3.632) -- (4.006,2.550) -- (2.850,1.461) -- (1.704,2.550) -- cycle;
\draw[MASCRed,line width=1.05pt] (2.850,3.632) -- (4.006,2.550) -- (2.850,1.461) -- (1.704,2.550) -- cycle;
\fill[MASCRed] (2.850,3.632) circle[radius=0.035];
\fill[MASCRed] (4.006,2.550) circle[radius=0.035];
\fill[MASCRed] (2.850,1.461) circle[radius=0.035];
\fill[MASCRed] (1.704,2.550) circle[radius=0.035];
\draw[GridGrey,line width=0.55pt] (10.150,2.550) circle[radius=0.430];
\draw[GridGrey,line width=0.55pt] (10.150,2.550) circle[radius=0.860];
\draw[GridGrey,line width=0.55pt] (10.150,2.550) circle[radius=1.290];
\draw[GridGrey,line width=0.55pt] (10.150,2.550) circle[radius=1.720];
\node[font=\scriptsize,text=black!55] at (10.283,2.959) {0.25};
\node[font=\scriptsize,text=black!55] at (10.416,3.368) {0.50};
\node[font=\scriptsize,text=black!55] at (10.549,3.777) {0.75};
\node[font=\scriptsize,text=black!55] at (10.682,4.186) {1.00};
\draw[GridGrey,line width=0.55pt] (10.150,2.550) -- (10.150,4.270);
\node[anchor=south,font=\normalsize,align=center] at (10.150,4.430) {1-VerbMem\\($D_f$)};
\draw[GridGrey,line width=0.55pt] (10.150,2.550) -- (11.640,1.690);
\node[anchor=north west,font=\normalsize,align=center] at (11.778,1.610) {1-KnowMem\\($D_f$)};
\draw[GridGrey,line width=0.55pt] (10.150,2.550) -- (8.660,1.690);
\node[anchor=north east,font=\normalsize,align=center] at (8.522,1.610) {KnowMem\\($D_r$)};
\fill[NPOBlue,opacity=0.07] (10.150,4.161) -- (10.868,2.135) -- (9.489,2.169) -- cycle;
\draw[NPOBlue,line width=1.05pt] (10.150,4.161) -- (10.868,2.135) -- (9.489,2.169) -- cycle;
\fill[NPOBlue] (10.150,4.161) circle[radius=0.035];
\fill[NPOBlue] (10.868,2.135) circle[radius=0.035];
\fill[NPOBlue] (9.489,2.169) circle[radius=0.035];
\fill[GAGreen,opacity=0.07] (10.150,4.266) -- (11.263,1.908) -- (9.638,2.254) -- cycle;
\draw[GAGreen,line width=1.05pt] (10.150,4.266) -- (11.263,1.908) -- (9.638,2.254) -- cycle;
\fill[GAGreen] (10.150,4.266) circle[radius=0.035];
\fill[GAGreen] (11.263,1.908) circle[radius=0.035];
\fill[GAGreen] (9.638,2.254) circle[radius=0.035];
\fill[SimOrange,opacity=0.07] (10.150,4.132) -- (10.957,2.084) -- (9.599,2.232) -- cycle;
\draw[SimOrange,line width=1.05pt] (10.150,4.132) -- (10.957,2.084) -- (9.599,2.232) -- cycle;
\fill[SimOrange] (10.150,4.132) circle[radius=0.035];
\fill[SimOrange] (10.957,2.084) circle[radius=0.035];
\fill[SimOrange] (9.599,2.232) circle[radius=0.035];
\fill[MASCRed,opacity=0.07] (10.150,4.251) -- (11.351,1.857) -- (9.805,2.351) -- cycle;
\draw[MASCRed,line width=1.05pt] (10.150,4.251) -- (11.351,1.857) -- (9.805,2.351) -- cycle;
\fill[MASCRed] (10.150,4.251) circle[radius=0.035];
\fill[MASCRed] (11.351,1.857) circle[radius=0.035];
\fill[MASCRed] (9.805,2.351) circle[radius=0.035];
\draw[GridGrey,line width=0.55pt] (17.450,2.550) circle[radius=0.430];
\draw[GridGrey,line width=0.55pt] (17.450,2.550) circle[radius=0.860];
\draw[GridGrey,line width=0.55pt] (17.450,2.550) circle[radius=1.290];
\draw[GridGrey,line width=0.55pt] (17.450,2.550) circle[radius=1.720];
\node[font=\scriptsize,text=black!55] at (17.583,2.959) {0.25};
\node[font=\scriptsize,text=black!55] at (17.716,3.368) {0.50};
\node[font=\scriptsize,text=black!55] at (17.849,3.777) {0.75};
\node[font=\scriptsize,text=black!55] at (17.982,4.186) {1.00};
\draw[GridGrey,line width=0.55pt] (17.450,2.550) -- (17.450,4.270);
\node[anchor=south,font=\normalsize,align=center] at (17.450,4.430) {1-VerbMem\\($D_f$)};
\draw[GridGrey,line width=0.55pt] (17.450,2.550) -- (18.940,1.690);
\node[anchor=north west,font=\normalsize,align=center] at (19.078,1.610) {1-KnowMem\\($D_f$)};
\draw[GridGrey,line width=0.55pt] (17.450,2.550) -- (15.960,1.690);
\node[anchor=north east,font=\normalsize,align=center] at (15.822,1.610) {KnowMem\\($D_r$)};
\fill[NPOBlue,opacity=0.07] (17.450,4.270) -- (18.592,1.891) -- (16.441,1.967) -- cycle;
\draw[NPOBlue,line width=1.05pt] (17.450,4.270) -- (18.592,1.891) -- (16.441,1.967) -- cycle;
\fill[NPOBlue] (17.450,4.270) circle[radius=0.035];
\fill[NPOBlue] (18.592,1.891) circle[radius=0.035];
\fill[NPOBlue] (16.441,1.967) circle[radius=0.035];
\fill[GAGreen,opacity=0.07] (17.450,4.270) -- (18.940,1.690) -- (16.836,2.195) -- cycle;
\draw[GAGreen,line width=1.05pt] (17.450,4.270) -- (18.940,1.690) -- (16.836,2.195) -- cycle;
\fill[GAGreen] (17.450,4.270) circle[radius=0.035];
\fill[GAGreen] (18.940,1.690) circle[radius=0.035];
\fill[GAGreen] (16.836,2.195) circle[radius=0.035];
\fill[SimOrange,opacity=0.07] (17.450,4.270) -- (18.940,1.690) -- (16.738,2.139) -- cycle;
\draw[SimOrange,line width=1.05pt] (17.450,4.270) -- (18.940,1.690) -- (16.738,2.139) -- cycle;
\fill[SimOrange] (17.450,4.270) circle[radius=0.035];
\fill[SimOrange] (18.940,1.690) circle[radius=0.035];
\fill[SimOrange] (16.738,2.139) circle[radius=0.035];
\fill[MASCRed,opacity=0.07] (17.450,4.255) -- (18.479,1.956) -- (16.477,1.988) -- cycle;
\draw[MASCRed,line width=1.05pt] (17.450,4.255) -- (18.479,1.956) -- (16.477,1.988) -- cycle;
\fill[MASCRed] (17.450,4.255) circle[radius=0.035];
\fill[MASCRed] (18.479,1.956) circle[radius=0.035];
\fill[MASCRed] (16.477,1.988) circle[radius=0.035];
\node[anchor=center,font=\Large] at (10.150,5.920) {\begin{tabular}{@{}c@{\hspace{0.98cm}}c@{\hspace{0.98cm}}c@{\hspace{0.98cm}}c@{}}\makebox[2.75cm][c]{\textcolor{NPOBlue}{\rule[0.30ex]{1.18cm}{0.18cm}}\;NPO+KLR} & \makebox[2.75cm][c]{\textcolor{GAGreen}{\rule[0.30ex]{1.18cm}{0.18cm}}\;GA+GDR} & \makebox[2.75cm][c]{\textcolor{SimOrange}{\rule[0.30ex]{1.18cm}{0.18cm}}\;SimNPO} & \makebox[3.25cm][c]{\textcolor{MASCRed}{\rule[0.30ex]{1.18cm}{0.18cm}}\;MASC (Ours)}\end{tabular}};
\node[anchor=center,font=\Large] at (10.150,-0.480) {\begin{tabular}{@{}c@{\hspace{0.82cm}}c@{\hspace{0.82cm}}c@{\hspace{0.82cm}}c@{}}\makebox[2.45cm][c]{\textcolor{NPOBlue}{\rule[0.45ex]{0.72cm}{2.0pt}}\;NPO+KLR} & \makebox[2.45cm][c]{\textcolor{GAGreen}{\rule[0.45ex]{0.72cm}{2.0pt}}\;GA+GDR} & \makebox[2.45cm][c]{\textcolor{SimOrange}{\rule[0.45ex]{0.72cm}{2.0pt}}\;SimNPO} & \makebox[2.95cm][c]{\textcolor{MASCRed}{\rule[0.45ex]{0.72cm}{2.0pt}}\;MASC (Ours)}\end{tabular}};
\end{tikzpicture}

%% file: plots/scaling_learning.pgf
\begin{tikzpicture}[x=1cm,y=1cm]
\definecolor{ERougeOrange}{HTML}{D55E00}
\definecolor{PRougeBlue}{HTML}{0072B2}
\definecolor{GridGrey}{HTML}{D5D7DC}
\definecolor{AxisDark}{HTML}{252A31}
\path[use as bounding box] (0,0) rectangle (8.65,7.85);
\draw[GridGrey,line width=0.45pt] (1.1,2.8103) -- (8.35,2.8103);
\node[anchor=east,font=\large] at (0.98,2.8103) {0.4};
\draw[GridGrey,line width=0.45pt] (1.1,4.4326) -- (8.35,4.4326);
\node[anchor=east,font=\large] at (0.98,4.4326) {0.6};
\draw[GridGrey,line width=0.45pt] (1.1,5.5837) -- (8.35,5.5837);
\node[anchor=east,font=\large] at (0.98,5.5837) {0.8};
\draw[GridGrey,line width=0.45pt] (1.1,6.4765) -- (8.35,6.4765);
\node[anchor=east,font=\large] at (0.98,6.4765) {1.0};
\draw[GridGrey,line width=0.35pt,opacity=0.65] (1.3684,1.08) -- (1.3684,6.93);
\draw[GridGrey,line width=0.35pt,opacity=0.65] (4.1668,1.08) -- (4.1668,6.93);
\draw[GridGrey,line width=0.35pt,opacity=0.65] (5.9324,1.08) -- (5.9324,6.93);
\draw[GridGrey,line width=0.35pt,opacity=0.65] (8.0907,1.08) -- (8.0907,6.93);
\draw[ERougeOrange,line width=1.45pt] (1.2644,3.273) -- (1.2934,3.2853) -- (1.3224,3.2975) -- (1.3514,3.3098) -- (1.3804,3.3221) -- (1.4094,3.3344) -- (1.4385,3.3466) -- (1.4675,3.3589) -- (1.4965,3.3712) -- (1.5255,3.3834) -- (1.5545,3.3957) -- (1.5835,3.408) -- (1.6125,3.4203) -- (1.6415,3.4325) -- (1.6705,3.4448) -- (1.6995,3.4571) -- (1.7285,3.4694) -- (1.7576,3.4816) -- (1.7866,3.4939) -- (1.8156,3.5062) -- (1.8446,3.5185) -- (1.8736,3.5307) -- (1.9026,3.543) -- (1.9316,3.5553) -- (1.9606,3.5676) -- (1.9896,3.5798) -- (2.0186,3.5921) -- (2.0476,3.6044) -- (2.0767,3.6166) -- (2.1057,3.6289) -- (2.1347,3.6412) -- (2.1637,3.6535) -- (2.1927,3.6657) -- (2.2217,3.678) -- (2.2507,3.6903) -- (2.2797,3.7026) -- (2.3087,3.7148) -- (2.3377,3.7271) -- (2.3668,3.7394) -- (2.3958,3.7517) -- (2.4248,3.7639) -- (2.4538,3.7762) -- (2.4828,3.7885) -- (2.5118,3.8007) -- (2.5408,3.813) -- (2.5698,3.8253) -- (2.5988,3.8376) -- (2.6278,3.8498) -- (2.6568,3.8621) -- (2.6859,3.8744) -- (2.7149,3.8867) -- (2.7439,3.8989) -- (2.7729,3.9112) -- (2.8019,3.9235) -- (2.8309,3.9358) -- (2.8599,3.948) -- (2.8889,3.9603) -- (2.9179,3.9726) -- (2.9469,3.9849) -- (2.9759,3.9971) -- (3.005,4.0094) -- (3.034,4.0217) -- (3.063,4.0339) -- (3.092,4.0462) -- (3.121,4.0585) -- (3.15,4.0708) -- (3.179,4.083) -- (3.208,4.0953) -- (3.237,4.1076) -- (3.266,4.1199) -- (3.295,4.1321) -- (3.3241,4.1444) -- (3.3531,4.1567) -- (3.3821,4.169) -- (3.4111,4.1812) -- (3.4401,4.1935) -- (3.4691,4.2058) -- (3.4981,4.2181) -- (3.5271,4.2303) -- (3.5561,4.2426) -- (3.5851,4.2549) -- (3.6142,4.2671) -- (3.6432,4.2794) -- (3.6722,4.2917) -- (3.7012,4.304) -- (3.7302,4.3162) -- (3.7592,4.3285) -- (3.7882,4.3408) -- (3.8172,4.3531) -- (3.8462,4.3653) -- (3.8752,4.3776) -- (3.9042,4.3899) -- (3.9333,4.4022) -- (3.9623,4.4144) -- (3.9913,4.4267) -- (4.0203,4.439) -- (4.0493,4.4513) -- (4.0783,4.4635) -- (4.1073,4.4758) -- (4.1363,4.4881) -- (4.1653,4.5003) -- (4.1943,4.5126) -- (4.2233,4.5249) -- (4.2524,4.5372) -- (4.2814,4.5494) -- (4.3104,4.5617) -- (4.3394,4.574) -- (4.3684,4.5863) -- (4.3974,4.5985) -- (4.4264,4.6108) -- (4.4554,4.6231) -- (4.4844,4.6354) -- (4.5134,4.6476) -- (4.5425,4.6599) -- (4.5715,4.6722) -- (4.6005,4.6844) -- (4.6295,4.6967) -- (4.6585,4.709) -- (4.6875,4.7213) -- (4.7165,4.7335) -- (4.7455,4.7458) -- (4.7745,4.7581) -- (4.8035,4.7704) -- (4.8325,4.7826) -- (4.8616,4.7949) -- (4.8906,4.8072) -- (4.9196,4.8195) -- (4.9486,4.8317) -- (4.9776,4.844) -- (5.0066,4.8563) -- (5.0356,4.8686) -- (5.0646,4.8808) -- (5.0936,4.8931) -- (5.1226,4.9054) -- (5.1516,4.9176) -- (5.1807,4.9299) -- (5.2097,4.9422) -- (5.2387,4.9545) -- (5.2677,4.9667) -- (5.2967,4.979) -- (5.3257,4.9913) -- (5.3547,5.0036) -- (5.3837,5.0158) -- (5.4127,5.0281) -- (5.4417,5.0404) -- (5.4707,5.0527) -- (5.4998,5.0649) -- (5.5288,5.0772) -- (5.5578,5.0895) -- (5.5868,5.1018) -- (5.6158,5.114) -- (5.6448,5.1263) -- (5.6738,5.1386) -- (5.7028,5.1508) -- (5.7318,5.1631) -- (5.7608,5.1754) -- (5.7899,5.1877) -- (5.8189,5.1999) -- (5.8479,5.2122) -- (5.8769,5.2245) -- (5.9059,5.2368) -- (5.9349,5.249) -- (5.9639,5.2613) -- (5.9929,5.2736) -- (6.0219,5.2859) -- (6.0509,5.2981) -- (6.0799,5.3104) -- (6.109,5.3227) -- (6.138,5.335) -- (6.167,5.3472) -- (6.196,5.3595) -- (6.225,5.3718) -- (6.254,5.384) -- (6.283,5.3963) -- (6.312,5.4086) -- (6.341,5.4209) -- (6.37,5.4331) -- (6.399,5.4454) -- (6.4281,5.4577) -- (6.4571,5.47) -- (6.4861,5.4822) -- (6.5151,5.4945) -- (6.5441,5.5068) -- (6.5731,5.5191) -- (6.6021,5.5313) -- (6.6311,5.5436) -- (6.6601,5.5559) -- (6.6891,5.5681) -- (6.7182,5.5804) -- (6.7472,5.5927) -- (6.7762,5.605) -- (6.8052,5.6172) -- (6.8342,5.6295) -- (6.8632,5.6418) -- (6.8922,5.6541) -- (6.9212,5.6663) -- (6.9502,5.6786) -- (6.9792,5.6909) -- (7.0082,5.7032) -- (7.0373,5.7154) -- (7.0663,5.7277) -- (7.0953,5.74) -- (7.1243,5.7523) -- (7.1533,5.7645) -- (7.1823,5.7768) -- (7.2113,5.7891) -- (7.2403,5.8013) -- (7.2693,5.8136) -- (7.2983,5.8259) -- (7.3273,5.8382) -- (7.3564,5.8504) -- (7.3854,5.8627) -- (7.4144,5.875) -- (7.4434,5.8873) -- (7.4724,5.8995) -- (7.5014,5.9118) -- (7.5304,5.9241) -- (7.5594,5.9364) -- (7.5884,5.9486) -- (7.6174,5.9609) -- (7.6464,5.9732) -- (7.6755,5.9855) -- (7.7045,5.9977) -- (7.7335,6.01) -- (7.7625,6.0223) -- (7.7915,6.0345) -- (7.8205,6.0468) -- (7.8495,6.0591) -- (7.8785,6.0714) -- (7.9075,6.0836) -- (7.9365,6.0959) -- (7.9656,6.1082) -- (7.9946,6.1205) -- (8.0236,6.1327) -- (8.0526,6.145) -- (8.0816,6.1573) -- (8.1106,6.1696) -- (8.1396,6.1818) -- (8.1686,6.1941) -- (8.1976,6.2064);
\filldraw[fill=white,draw=ERougeOrange,line width=0.45pt] (1.156,3.5285) circle[radius=0.072];
\filldraw[fill=white,draw=ERougeOrange,line width=0.45pt] (1.2776,3.5402) circle[radius=0.072];
\filldraw[fill=white,draw=ERougeOrange,line width=0.45pt] (1.5048,3.5432) circle[radius=0.072];
\filldraw[fill=white,draw=ERougeOrange,line width=0.45pt] (3.9544,4.2287) circle[radius=0.072];
\filldraw[fill=white,draw=ERougeOrange,line width=0.45pt] (4.0761,4.1744) circle[radius=0.072];
\filldraw[fill=white,draw=ERougeOrange,line width=0.45pt] (4.3032,4.2093) circle[radius=0.072];
\filldraw[fill=white,draw=ERougeOrange,line width=0.45pt] (5.7201,5.1763) circle[radius=0.072];
\filldraw[fill=white,draw=ERougeOrange,line width=0.45pt] (5.8417,5.0737) circle[radius=0.072];
\filldraw[fill=white,draw=ERougeOrange,line width=0.45pt] (6.0688,5.0528) circle[radius=0.072];
\filldraw[fill=white,draw=ERougeOrange,line width=0.45pt] (7.8783,6.393) circle[radius=0.072];
\filldraw[fill=white,draw=ERougeOrange,line width=0.45pt] (8,6.3622) circle[radius=0.072];
\filldraw[fill=white,draw=ERougeOrange,line width=0.45pt] (8.2271,6.3976) circle[radius=0.072];
\filldraw[fill=ERougeOrange,draw=ERougeOrange,line width=0.45pt] (1.3684,3.5373) circle[radius=0.072];
\filldraw[fill=ERougeOrange,draw=ERougeOrange,line width=0.45pt] (4.1668,4.2042) circle[radius=0.072];
\filldraw[fill=ERougeOrange,draw=ERougeOrange,line width=0.45pt] (5.9324,5.1013) circle[radius=0.072];
\filldraw[fill=ERougeOrange,draw=ERougeOrange,line width=0.45pt] (8.0907,6.3843) circle[radius=0.072];
\draw[PRougeBlue,line width=1.45pt] (1.2644,2.2586) -- (1.2934,2.2637) -- (1.3224,2.2687) -- (1.3514,2.2738) -- (1.3804,2.2789) -- (1.4094,2.284) -- (1.4385,2.2891) -- (1.4675,2.2941) -- (1.4965,2.2992) -- (1.5255,2.3043) -- (1.5545,2.3094) -- (1.5835,2.3145) -- (1.6125,2.3195) -- (1.6415,2.3246) -- (1.6705,2.3297) -- (1.6995,2.3348) -- (1.7285,2.3398) -- (1.7576,2.3449) -- (1.7866,2.35) -- (1.8156,2.3551) -- (1.8446,2.3602) -- (1.8736,2.3652) -- (1.9026,2.3703) -- (1.9316,2.3754) -- (1.9606,2.3805) -- (1.9896,2.3856) -- (2.0186,2.3906) -- (2.0476,2.3957) -- (2.0767,2.4008) -- (2.1057,2.4059) -- (2.1347,2.411) -- (2.1637,2.416) -- (2.1927,2.4211) -- (2.2217,2.4262) -- (2.2507,2.4313) -- (2.2797,2.4363) -- (2.3087,2.4414) -- (2.3377,2.4465) -- (2.3668,2.4516) -- (2.3958,2.4567) -- (2.4248,2.4617) -- (2.4538,2.4668) -- (2.4828,2.4719) -- (2.5118,2.477) -- (2.5408,2.4821) -- (2.5698,2.4871) -- (2.5988,2.4922) -- (2.6278,2.4973) -- (2.6568,2.5024) -- (2.6859,2.5075) -- (2.7149,2.5125) -- (2.7439,2.5176) -- (2.7729,2.5227) -- (2.8019,2.5278) -- (2.8309,2.5328) -- (2.8599,2.5379) -- (2.8889,2.543) -- (2.9179,2.5481) -- (2.9469,2.5532) -- (2.9759,2.5582) -- (3.005,2.5633) -- (3.034,2.5684) -- (3.063,2.5735) -- (3.092,2.5786) -- (3.121,2.5836) -- (3.15,2.5887) -- (3.179,2.5938) -- (3.208,2.5989) -- (3.237,2.604) -- (3.266,2.609) -- (3.295,2.6141) -- (3.3241,2.6192) -- (3.3531,2.6243) -- (3.3821,2.6293) -- (3.4111,2.6344) -- (3.4401,2.6395) -- (3.4691,2.6446) -- (3.4981,2.6497) -- (3.5271,2.6547) -- (3.5561,2.6598) -- (3.5851,2.6649) -- (3.6142,2.67) -- (3.6432,2.6751) -- (3.6722,2.6801) -- (3.7012,2.6852) -- (3.7302,2.6903) -- (3.7592,2.6954) -- (3.7882,2.7005) -- (3.8172,2.7055) -- (3.8462,2.7106) -- (3.8752,2.7157) -- (3.9042,2.7208) -- (3.9333,2.7258) -- (3.9623,2.7309) -- (3.9913,2.736) -- (4.0203,2.7411) -- (4.0493,2.7462) -- (4.0783,2.7512) -- (4.1073,2.7563) -- (4.1363,2.7614) -- (4.1653,2.7665) -- (4.1943,2.7716) -- (4.2233,2.7766) -- (4.2524,2.7817) -- (4.2814,2.7868) -- (4.3104,2.7919) -- (4.3394,2.797) -- (4.3684,2.802) -- (4.3974,2.8071) -- (4.4264,2.8122) -- (4.4554,2.8173) -- (4.4844,2.8224) -- (4.5134,2.8274) -- (4.5425,2.8325) -- (4.5715,2.8376) -- (4.6005,2.8427) -- (4.6295,2.8477) -- (4.6585,2.8528) -- (4.6875,2.8579) -- (4.7165,2.863) -- (4.7455,2.8681) -- (4.7745,2.8731) -- (4.8035,2.8782) -- (4.8325,2.8833) -- (4.8616,2.8884) -- (4.8906,2.8935) -- (4.9196,2.8985) -- (4.9486,2.9036) -- (4.9776,2.9087) -- (5.0066,2.9138) -- (5.0356,2.9189) -- (5.0646,2.9239) -- (5.0936,2.929) -- (5.1226,2.9341) -- (5.1516,2.9392) -- (5.1807,2.9442) -- (5.2097,2.9493) -- (5.2387,2.9544) -- (5.2677,2.9595) -- (5.2967,2.9646) -- (5.3257,2.9696) -- (5.3547,2.9747) -- (5.3837,2.9798) -- (5.4127,2.9849) -- (5.4417,2.99) -- (5.4707,2.995) -- (5.4998,3.0001) -- (5.5288,3.0052) -- (5.5578,3.0103) -- (5.5868,3.0154) -- (5.6158,3.0204) -- (5.6448,3.0255) -- (5.6738,3.0306) -- (5.7028,3.0357) -- (5.7318,3.0407) -- (5.7608,3.0458) -- (5.7899,3.0509) -- (5.8189,3.056) -- (5.8479,3.0611) -- (5.8769,3.0661) -- (5.9059,3.0712) -- (5.9349,3.0763) -- (5.9639,3.0814) -- (5.9929,3.0865) -- (6.0219,3.0915) -- (6.0509,3.0966) -- (6.0799,3.1017) -- (6.109,3.1068) -- (6.138,3.1119) -- (6.167,3.1169) -- (6.196,3.122) -- (6.225,3.1271) -- (6.254,3.1322) -- (6.283,3.1372) -- (6.312,3.1423) -- (6.341,3.1474) -- (6.37,3.1525) -- (6.399,3.1576) -- (6.4281,3.1626) -- (6.4571,3.1677) -- (6.4861,3.1728) -- (6.5151,3.1779) -- (6.5441,3.183) -- (6.5731,3.188) -- (6.6021,3.1931) -- (6.6311,3.1982) -- (6.6601,3.2033) -- (6.6891,3.2084) -- (6.7182,3.2134) -- (6.7472,3.2185) -- (6.7762,3.2236) -- (6.8052,3.2287) -- (6.8342,3.2337) -- (6.8632,3.2388) -- (6.8922,3.2439) -- (6.9212,3.249) -- (6.9502,3.2541) -- (6.9792,3.2591) -- (7.0082,3.2642) -- (7.0373,3.2693) -- (7.0663,3.2744) -- (7.0953,3.2795) -- (7.1243,3.2845) -- (7.1533,3.2896) -- (7.1823,3.2947) -- (7.2113,3.2998) -- (7.2403,3.3049) -- (7.2693,3.3099) -- (7.2983,3.315) -- (7.3273,3.3201) -- (7.3564,3.3252) -- (7.3854,3.3302) -- (7.4144,3.3353) -- (7.4434,3.3404) -- (7.4724,3.3455) -- (7.5014,3.3506) -- (7.5304,3.3556) -- (7.5594,3.3607) -- (7.5884,3.3658) -- (7.6174,3.3709) -- (7.6464,3.376) -- (7.6755,3.381) -- (7.7045,3.3861) -- (7.7335,3.3912) -- (7.7625,3.3963) -- (7.7915,3.4014) -- (7.8205,3.4064) -- (7.8495,3.4115) -- (7.8785,3.4166) -- (7.9075,3.4217) -- (7.9365,3.4267) -- (7.9656,3.4318) -- (7.9946,3.4369) -- (8.0236,3.442) -- (8.0526,3.4471) -- (8.0816,3.4521) -- (8.1106,3.4572) -- (8.1396,3.4623) -- (8.1686,3.4674) -- (8.1976,3.4725);
\filldraw[fill=white,draw=PRougeBlue,line width=0.45pt] (1.086,2.4218) rectangle (1.226,2.5618);
\filldraw[fill=white,draw=PRougeBlue,line width=0.45pt] (1.2076,2.3392) rectangle (1.3476,2.4792);
\filldraw[fill=white,draw=PRougeBlue,line width=0.45pt] (1.4348,2.3244) rectangle (1.5748,2.4644);
\filldraw[fill=white,draw=PRougeBlue,line width=0.45pt] (3.8844,2.5071) rectangle (4.0244,2.6471);
\filldraw[fill=white,draw=PRougeBlue,line width=0.45pt] (4.0061,2.5558) rectangle (4.1461,2.6958);
\filldraw[fill=white,draw=PRougeBlue,line width=0.45pt] (4.2332,2.5801) rectangle (4.3732,2.7201);
\filldraw[fill=white,draw=PRougeBlue,line width=0.45pt] (5.6501,2.8247) rectangle (5.7901,2.9647);
\filldraw[fill=white,draw=PRougeBlue,line width=0.45pt] (5.7717,2.7837) rectangle (5.9117,2.9237);
\filldraw[fill=white,draw=PRougeBlue,line width=0.45pt] (5.9988,2.7709) rectangle (6.1388,2.9109);
\filldraw[fill=white,draw=PRougeBlue,line width=0.45pt] (7.8083,3.6203) rectangle (7.9483,3.7603);
\filldraw[fill=white,draw=PRougeBlue,line width=0.45pt] (7.93,3.6025) rectangle (8.07,3.7425);
\filldraw[fill=white,draw=PRougeBlue,line width=0.45pt] (8.1571,3.5474) rectangle (8.2971,3.6874);
\filldraw[fill=PRougeBlue,draw=PRougeBlue,line width=0.45pt] (1.2984,2.362) rectangle (1.4384,2.502);
\filldraw[fill=PRougeBlue,draw=PRougeBlue,line width=0.45pt] (4.0968,2.5478) rectangle (4.2368,2.6878);
\filldraw[fill=PRougeBlue,draw=PRougeBlue,line width=0.45pt] (5.8624,2.7932) rectangle (6.0024,2.9332);
\filldraw[fill=PRougeBlue,draw=PRougeBlue,line width=0.45pt] (8.0207,3.5902) rectangle (8.1607,3.7302);
\draw[AxisDark,line width=1.55pt] (1.1,1.08) -- (8.35,1.08);
\draw[AxisDark,line width=1.55pt] (1.1,1.08) -- (1.1,6.93);
\draw[AxisDark,line width=1.25pt] (1.3684,1.08) -- (1.3684,0.97);
\node[anchor=north,font=\large] at (1.3684,0.9) {0.5B};
\draw[AxisDark,line width=1.25pt] (4.1668,1.08) -- (4.1668,0.97);
\node[anchor=north,font=\large] at (4.1668,0.9) {1.5B};
\draw[AxisDark,line width=1.25pt] (5.9324,1.08) -- (5.9324,0.97);
\node[anchor=north,font=\large] at (5.9324,0.9) {3B};
\draw[AxisDark,line width=1.25pt] (8.0907,1.08) -- (8.0907,0.97);
\node[anchor=north,font=\large] at (8.0907,0.9) {7B};
\draw[AxisDark,line width=1.25pt] (1.1,2.8103) -- (0.99,2.8103);
\draw[AxisDark,line width=1.25pt] (1.1,4.4326) -- (0.99,4.4326);
\draw[AxisDark,line width=1.25pt] (1.1,5.5837) -- (0.99,5.5837);
\draw[AxisDark,line width=1.25pt] (1.1,6.4765) -- (0.99,6.4765);
\node[anchor=north,font=\large] at (4.725,0.46) {Model size};
\node[rotate=90,anchor=south,font=\large] at (0.28,4.005) {Raw score};
\draw[ERougeOrange,line width=1.45pt] (6.13,2.11) -- (6.63,2.11);
\filldraw[fill=ERougeOrange,draw=ERougeOrange,line width=0.45pt] (6.38,2.11) circle[radius=0.072];
\node[anchor=west,font=\normalsize] at (6.78,2.11) {E-ROUGE};
\draw[PRougeBlue,line width=1.45pt] (6.13,1.75) -- (6.63,1.75);
\filldraw[fill=PRougeBlue,draw=PRougeBlue,line width=0.45pt] (6.31,1.68) rectangle (6.45,1.82);
\node[anchor=west,font=\normalsize] at (6.78,1.75) {P-ROUGE};
\end{tikzpicture}

%% file: plots/qwen_tofu_unlearning_radar_by_scale.pgf
\begin{tikzpicture}[x=1cm,y=1cm]
\path[use as bounding box] (-5.15,-2.65) rectangle (5.15,3.10);
\definecolor{ScaleHalfB}{HTML}{C6DBEF}
\definecolor{ScaleOneHalfB}{HTML}{9ECAE1}
\definecolor{ScaleThreeB}{HTML}{4292C6}
\definecolor{ScaleSevenB}{HTML}{08519C}
\definecolor{RadarGrid}{HTML}{D9DDE1}
\definecolor{RadarAxis}{HTML}{222222}
\definecolor{RadarTick}{HTML}{777777}
\node[font=\large] at (-2.750,2.820) {SimNPO};
\draw[RadarGrid,line width=0.45pt] (-2.750,0.000) circle[radius=0.445];
\node[font=\scriptsize,text=RadarTick,anchor=west] at (-2.319,0.202) {0.2};
\draw[RadarGrid,line width=0.45pt] (-2.750,0.000) circle[radius=0.890];
\node[font=\scriptsize,text=RadarTick,anchor=west] at (-1.922,0.404) {0.4};
\draw[RadarGrid,line width=0.45pt] (-2.750,0.000) circle[radius=1.335];
\node[font=\scriptsize,text=RadarTick,anchor=west] at (-1.526,0.606) {0.6};
\draw[RadarGrid,line width=0.45pt] (-2.750,0.000) circle[radius=1.780];
\node[font=\scriptsize,text=RadarTick,anchor=west] at (-1.129,0.808) {0.8};
\draw[RadarGrid,line width=0.45pt] (-2.750,0.000) -- (-2.750,1.780);
\draw[RadarGrid,line width=0.45pt] (-2.750,0.000) -- (-1.208,-0.890);
\draw[RadarGrid,line width=0.45pt] (-2.750,0.000) -- (-4.292,-0.890);
\draw[RadarAxis,line width=1.05pt] (-2.750,0.000) circle[radius=1.780];
\node[font=\footnotesize,anchor=center] at (-2.750,1.900) {1-(E-ROUGE)};
\node[font=\footnotesize,anchor=west] at (-1.103,-0.676) {1-(P-ROUGE)};
\node[font=\footnotesize,anchor=east] at (-4.245,-0.819) {MU on $\mathcal{D}_\mathrm{r}$};
\draw[ScaleHalfB,line width=0.95pt] (-2.750,1.382) -- (-1.452,-0.749) -- (-3.571,-0.474) -- (-2.750,1.382);
\fill[ScaleHalfB] (-2.750,1.382) circle[radius=0.043];
\fill[ScaleHalfB] (-1.452,-0.749) circle[radius=0.043];
\fill[ScaleHalfB] (-3.571,-0.474) circle[radius=0.043];
\draw[ScaleOneHalfB,line width=0.95pt] (-2.750,1.378) -- (-1.482,-0.732) -- (-3.678,-0.536) -- (-2.750,1.378);
\fill[ScaleOneHalfB] (-2.805,1.323) rectangle (-2.695,1.433);
\fill[ScaleOneHalfB] (-1.537,-0.787) rectangle (-1.427,-0.677);
\fill[ScaleOneHalfB] (-3.733,-0.591) rectangle (-3.623,-0.481);
\draw[ScaleThreeB,line width=0.95pt] (-2.750,1.379) -- (-1.502,-0.720) -- (-3.811,-0.613) -- (-2.750,1.379);
\fill[ScaleThreeB] (-2.750,1.444) -- (-2.685,1.379) -- (-2.750,1.314) -- (-2.815,1.379) -- cycle;
\fill[ScaleThreeB] (-1.502,-0.655) -- (-1.438,-0.720) -- (-1.502,-0.785) -- (-1.567,-0.720) -- cycle;
\fill[ScaleThreeB] (-3.811,-0.548) -- (-3.746,-0.613) -- (-3.811,-0.678) -- (-3.876,-0.613) -- cycle;
\draw[ScaleSevenB,line width=0.95pt] (-2.750,1.391) -- (-1.480,-0.734) -- (-3.959,-0.698) -- (-2.750,1.391);
\fill[ScaleSevenB] (-2.750,1.456) -- (-2.815,1.336) -- (-2.685,1.336) -- cycle;
\fill[ScaleSevenB] (-1.480,-0.669) -- (-1.544,-0.789) -- (-1.415,-0.789) -- cycle;
\fill[ScaleSevenB] (-3.959,-0.633) -- (-4.023,-0.753) -- (-3.894,-0.753) -- cycle;
\node[font=\large] at (2.750,2.820) {MASC};
\draw[RadarGrid,line width=0.45pt] (2.750,0.000) circle[radius=0.445];
\node[font=\scriptsize,text=RadarTick,anchor=west] at (3.181,0.202) {0.2};
\draw[RadarGrid,line width=0.45pt] (2.750,0.000) circle[radius=0.890];
\node[font=\scriptsize,text=RadarTick,anchor=west] at (3.578,0.404) {0.4};
\draw[RadarGrid,line width=0.45pt] (2.750,0.000) circle[radius=1.335];
\node[font=\scriptsize,text=RadarTick,anchor=west] at (3.974,0.606) {0.6};
\draw[RadarGrid,line width=0.45pt] (2.750,0.000) circle[radius=1.780];
\node[font=\scriptsize,text=RadarTick,anchor=west] at (4.371,0.808) {0.8};
\draw[RadarGrid,line width=0.45pt] (2.750,0.000) -- (2.750,1.780);
\draw[RadarGrid,line width=0.45pt] (2.750,0.000) -- (4.292,-0.890);
\draw[RadarGrid,line width=0.45pt] (2.750,0.000) -- (1.208,-0.890);
\draw[RadarAxis,line width=1.05pt] (2.750,0.000) circle[radius=1.780];
\node[font=\footnotesize,anchor=center] at (2.750,1.900) {1-(E-ROUGE)};
\node[font=\footnotesize,anchor=west] at (4.245,-0.819) {1-(P-ROUGE)};
\node[font=\footnotesize,anchor=east] at (1.255,-0.997) {MU on $\mathcal{D}_\mathrm{r}$};
\draw[ScaleHalfB,line width=0.95pt] (2.750,1.369) -- (4.071,-0.763) -- (1.959,-0.457) -- (2.750,1.369);
\fill[ScaleHalfB] (2.750,1.369) circle[radius=0.043];
\fill[ScaleHalfB] (4.071,-0.763) circle[radius=0.043];
\fill[ScaleHalfB] (1.959,-0.457) circle[radius=0.043];
\draw[ScaleOneHalfB,line width=0.95pt] (2.750,1.347) -- (4.079,-0.767) -- (1.876,-0.505) -- (2.750,1.347);
\fill[ScaleOneHalfB] (2.695,1.292) rectangle (2.805,1.402);
\fill[ScaleOneHalfB] (4.024,-0.822) rectangle (4.134,-0.712);
\fill[ScaleOneHalfB] (1.821,-0.560) rectangle (1.931,-0.450);
\draw[ScaleThreeB,line width=0.95pt] (2.750,1.301) -- (4.047,-0.749) -- (1.738,-0.584) -- (2.750,1.301);
\fill[ScaleThreeB] (2.750,1.366) -- (2.815,1.301) -- (2.750,1.236) -- (2.685,1.301) -- cycle;
\fill[ScaleThreeB] (4.047,-0.684) -- (4.112,-0.749) -- (4.047,-0.814) -- (3.982,-0.749) -- cycle;
\fill[ScaleThreeB] (1.738,-0.520) -- (1.803,-0.584) -- (1.738,-0.649) -- (1.673,-0.584) -- cycle;
\draw[ScaleSevenB,line width=0.95pt] (2.750,1.338) -- (4.149,-0.808) -- (1.558,-0.688) -- (2.750,1.338);
\fill[ScaleSevenB] (2.750,1.403) -- (2.685,1.283) -- (2.815,1.283) -- cycle;
\fill[ScaleSevenB] (4.149,-0.743) -- (4.085,-0.863) -- (4.214,-0.863) -- cycle;
\fill[ScaleSevenB] (1.558,-0.623) -- (1.493,-0.743) -- (1.623,-0.743) -- cycle;
\draw[ScaleHalfB,line width=1.15pt] (-3.270,-2.350) -- (-2.630,-2.350);
\fill[ScaleHalfB] (-2.950,-2.350) circle[radius=0.047];
\node[anchor=west,font=\small] at (-2.490,-2.360) {0.5B};
\draw[ScaleOneHalfB,line width=1.15pt] (-1.420,-2.350) -- (-0.780,-2.350);
\fill[ScaleOneHalfB] (-1.160,-2.410) rectangle (-1.040,-2.290);
\node[anchor=west,font=\small] at (-0.640,-2.360) {1.5B};
\draw[ScaleThreeB,line width=1.15pt] (0.380,-2.350) -- (1.020,-2.350);
\fill[ScaleThreeB] (0.700,-2.279) -- (0.771,-2.350) -- (0.700,-2.421) -- (0.629,-2.350) -- cycle;
\node[anchor=west,font=\small] at (1.160,-2.360) {3B};
\draw[ScaleSevenB,line width=1.15pt] (2.230,-2.350) -- (2.870,-2.350);
\fill[ScaleSevenB] (2.550,-2.279) -- (2.479,-2.410) -- (2.621,-2.410) -- cycle;
\node[anchor=west,font=\small] at (3.010,-2.360) {7B};
\end{tikzpicture}

%% file: plots/forget_retain_pareto.pgf
\begin{tikzpicture}[x=1cm,y=1cm]
\definecolor{GridGrey}{HTML}{E0E0E0}
\definecolor{AxisDark}{HTML}{333333}
\definecolor{NeutralFill}{HTML}{E6E6E6}
\definecolor{NeutralEdge}{HTML}{6F6F6F}
\path[use as bounding box] (-1.65,0.05) rectangle (22.200,10.450);
\definecolor{CTOFUGA}{HTML}{FDC474}
\definecolor{CTOFUGradDiff}{HTML}{DB3C2C}
\definecolor{CTOFUNPO}{HTML}{DD4531}
\definecolor{CTOFUNPOKLR}{HTML}{D73027}
\definecolor{CTOFURMU}{HTML}{FDC474}
\definecolor{CTOFUSimNPO}{HTML}{F38D52}
\definecolor{CTOFUMASC}{HTML}{1A9850}
\definecolor{CMUSENewsGA}{HTML}{52AA5F}
\definecolor{CMUSENewsGradDiff}{HTML}{E76540}
\definecolor{CMUSENewsNPO}{HTML}{7DB76A}
\definecolor{CMUSENewsNPOKLR}{HTML}{D73027}
\definecolor{CMUSENewsRMU}{HTML}{FDB768}
\definecolor{CMUSENewsSimNPO}{HTML}{F1864F}
\definecolor{CMUSENewsMASC}{HTML}{1A9850}
\definecolor{CMUSEBooksGA}{HTML}{FED582}
\definecolor{CMUSEBooksGradDiff}{HTML}{FDBA6B}
\definecolor{CMUSEBooksNPO}{HTML}{FECD7B}
\definecolor{CMUSEBooksNPOKLR}{HTML}{D83328}
\definecolor{CMUSEBooksRMU}{HTML}{E96B42}
\definecolor{CMUSEBooksSimNPO}{HTML}{D73027}
\definecolor{CMUSEBooksMASC}{HTML}{1A9850}
\node[font=\LARGE] at (3.125,9.970) {TOFU};
\draw[draw=AxisDark,line width=0.65pt] (0.000,4.550) rectangle (6.250,9.350);
\draw[draw=GridGrey,line width=0.45pt] (0.352,4.550) -- (0.352,9.350);
\node[anchor=north,font=\large] at (0.352,4.390) {0.0};
\draw[draw=GridGrey,line width=0.45pt] (1.914,4.550) -- (1.914,9.350);
\node[anchor=north,font=\large] at (1.914,4.390) {0.2};
\draw[draw=GridGrey,line width=0.45pt] (3.477,4.550) -- (3.477,9.350);
\node[anchor=north,font=\large] at (3.477,4.390) {0.4};
\draw[draw=GridGrey,line width=0.45pt] (5.039,4.550) -- (5.039,9.350);
\node[anchor=north,font=\large] at (5.039,4.390) {0.6};
\draw[draw=GridGrey,line width=0.45pt] (0.000,4.743) -- (6.250,4.743);
\node[anchor=east,font=\large] at (-0.160,4.743) {0.00};
\draw[draw=GridGrey,line width=0.45pt] (0.000,5.814) -- (6.250,5.814);
\node[anchor=east,font=\large] at (-0.160,5.814) {0.25};
\draw[draw=GridGrey,line width=0.45pt] (0.000,6.886) -- (6.250,6.886);
\node[anchor=east,font=\large] at (-0.160,6.886) {0.50};
\draw[draw=GridGrey,line width=0.45pt] (0.000,7.957) -- (6.250,7.957);
\node[anchor=east,font=\large] at (-0.160,7.957) {0.75};
\draw[draw=GridGrey,line width=0.45pt] (0.000,9.029) -- (6.250,9.029);
\node[anchor=east,font=\large] at (-0.160,9.029) {1.00};
\node[font=\Large] at (3.125,3.530) {Retain utility};
\node[rotate=90,font=\Large] at (-1.300,6.950) {Forget efficacy};
\draw[draw=AxisDark,line width=1.05pt] (5.141,7.791) -- (5.555,7.506);
\draw[draw=NeutralEdge,fill=NeutralFill,line width=0.65pt] (5.258,5.533) circle[radius=0.150];
\draw[draw=NeutralEdge,fill=NeutralFill,line width=0.65pt] (4.991,7.641) rectangle (5.291,7.941);
\draw[draw=AxisDark,fill=CTOFUGA,line width=0.65pt] (3.938,7.341) -- (3.805,7.074) -- (4.069,7.074) -- cycle;
\draw[draw=AxisDark,fill=CTOFUGradDiff,line width=0.65pt] (4.734,7.313) -- (4.602,7.580) -- (4.866,7.580) -- cycle;
\draw[draw=AxisDark,fill=CTOFUNPO,line width=0.65pt] (4.471,7.196) -- (4.561,7.196) -- (4.561,7.091) -- (4.666,7.091) -- (4.666,7.001) -- (4.561,7.001) -- (4.561,6.896) -- (4.471,6.896) -- (4.471,7.001) -- (4.366,7.001) -- (4.366,7.091) -- (4.471,7.091) -- cycle;
\draw[draw=AxisDark,line width=2.20pt,line cap=round] (4.248,6.968) -- (4.518,7.238);
\draw[draw=AxisDark,line width=2.20pt,line cap=round] (4.248,7.238) -- (4.518,6.968);
\draw[draw=CTOFUNPOKLR,line width=1.25pt,line cap=round] (4.248,6.968) -- (4.518,7.238);
\draw[draw=CTOFUNPOKLR,line width=1.25pt,line cap=round] (4.248,7.238) -- (4.518,6.968);
\draw[draw=AxisDark,fill=CTOFURMU,line width=0.65pt] (5.180,5.901) -- (5.330,5.751) -- (5.180,5.601) -- (5.030,5.751) -- cycle;
\draw[draw=AxisDark,fill=CTOFUSimNPO,line width=0.65pt] (5.138,6.829) -- (5.008,6.904) -- (4.878,6.829) -- (4.878,6.679) -- (5.008,6.604) -- (5.138,6.679) -- cycle;
\draw[draw=AxisDark,fill=CTOFUMASC,line width=0.80pt] (5.555,7.656) -- (5.517,7.558) -- (5.412,7.552) -- (5.493,7.486) -- (5.467,7.384) -- (5.555,7.441) -- (5.643,7.384) -- (5.616,7.486) -- (5.697,7.552) -- (5.593,7.558) -- cycle;
\node[font=\large] at (3.125,2.880) {Runtime (s)};
\definecolor{CBTOFU0}{HTML}{1A9850}
\fill[CBTOFU0] (1.125,2.240) rectangle (1.175,2.460);
\definecolor{CBTOFU1}{HTML}{229B52}
\fill[CBTOFU1] (1.175,2.240) rectangle (1.225,2.460);
\definecolor{CBTOFU2}{HTML}{2A9D54}
\fill[CBTOFU2] (1.225,2.240) rectangle (1.275,2.460);
\definecolor{CBTOFU3}{HTML}{32A056}
\fill[CBTOFU3] (1.275,2.240) rectangle (1.325,2.460);
\definecolor{CBTOFU4}{HTML}{3AA258}
\fill[CBTOFU4] (1.325,2.240) rectangle (1.375,2.460);
\definecolor{CBTOFU5}{HTML}{45A65B}
\fill[CBTOFU5] (1.375,2.240) rectangle (1.425,2.460);
\definecolor{CBTOFU6}{HTML}{4DA85D}
\fill[CBTOFU6] (1.425,2.240) rectangle (1.475,2.460);
\definecolor{CBTOFU7}{HTML}{55AB5F}
\fill[CBTOFU7] (1.475,2.240) rectangle (1.525,2.460);
\definecolor{CBTOFU8}{HTML}{5DAD61}
\fill[CBTOFU8] (1.525,2.240) rectangle (1.575,2.460);
\definecolor{CBTOFU9}{HTML}{65B063}
\fill[CBTOFU9] (1.575,2.240) rectangle (1.625,2.460);
\definecolor{CBTOFU10}{HTML}{70B366}
\fill[CBTOFU10] (1.625,2.240) rectangle (1.675,2.460);
\definecolor{CBTOFU11}{HTML}{78B668}
\fill[CBTOFU11] (1.675,2.240) rectangle (1.725,2.460);
\definecolor{CBTOFU12}{HTML}{80B86A}
\fill[CBTOFU12] (1.725,2.240) rectangle (1.775,2.460);
\definecolor{CBTOFU13}{HTML}{88BB6C}
\fill[CBTOFU13] (1.775,2.240) rectangle (1.825,2.460);
\definecolor{CBTOFU14}{HTML}{90BD6F}
\fill[CBTOFU14] (1.825,2.240) rectangle (1.875,2.460);
\definecolor{CBTOFU15}{HTML}{9BC171}
\fill[CBTOFU15] (1.875,2.240) rectangle (1.925,2.460);
\definecolor{CBTOFU16}{HTML}{A3C373}
\fill[CBTOFU16] (1.925,2.240) rectangle (1.975,2.460);
\definecolor{CBTOFU17}{HTML}{ABC675}
\fill[CBTOFU17] (1.975,2.240) rectangle (2.025,2.460);
\definecolor{CBTOFU18}{HTML}{B3C878}
\fill[CBTOFU18] (2.025,2.240) rectangle (2.075,2.460);
\definecolor{CBTOFU19}{HTML}{BBCB7A}
\fill[CBTOFU19] (2.075,2.240) rectangle (2.125,2.460);
\definecolor{CBTOFU20}{HTML}{C6CE7C}
\fill[CBTOFU20] (2.125,2.240) rectangle (2.175,2.460);
\definecolor{CBTOFU21}{HTML}{CED17F}
\fill[CBTOFU21] (2.175,2.240) rectangle (2.225,2.460);
\definecolor{CBTOFU22}{HTML}{D6D381}
\fill[CBTOFU22] (2.225,2.240) rectangle (2.275,2.460);
\definecolor{CBTOFU23}{HTML}{DED683}
\fill[CBTOFU23] (2.275,2.240) rectangle (2.325,2.460);
\definecolor{CBTOFU24}{HTML}{E6D885}
\fill[CBTOFU24] (2.325,2.240) rectangle (2.375,2.460);
\definecolor{CBTOFU25}{HTML}{F1DC88}
\fill[CBTOFU25] (2.375,2.240) rectangle (2.425,2.460);
\definecolor{CBTOFU26}{HTML}{F9DE8A}
\fill[CBTOFU26] (2.425,2.240) rectangle (2.475,2.460);
\definecolor{CBTOFU27}{HTML}{FEDF8B}
\fill[CBTOFU27] (2.475,2.240) rectangle (2.525,2.460);
\definecolor{CBTOFU28}{HTML}{FEDE89}
\fill[CBTOFU28] (2.525,2.240) rectangle (2.575,2.460);
\definecolor{CBTOFU29}{HTML}{FEDC88}
\fill[CBTOFU29] (2.575,2.240) rectangle (2.625,2.460);
\definecolor{CBTOFU30}{HTML}{FEDA86}
\fill[CBTOFU30] (2.625,2.240) rectangle (2.675,2.460);
\definecolor{CBTOFU31}{HTML}{FED884}
\fill[CBTOFU31] (2.675,2.240) rectangle (2.725,2.460);
\definecolor{CBTOFU32}{HTML}{FED683}
\fill[CBTOFU32] (2.725,2.240) rectangle (2.775,2.460);
\definecolor{CBTOFU33}{HTML}{FED481}
\fill[CBTOFU33] (2.775,2.240) rectangle (2.825,2.460);
\definecolor{CBTOFU34}{HTML}{FED280}
\fill[CBTOFU34] (2.825,2.240) rectangle (2.875,2.460);
\definecolor{CBTOFU35}{HTML}{FED07E}
\fill[CBTOFU35] (2.875,2.240) rectangle (2.925,2.460);
\definecolor{CBTOFU36}{HTML}{FECE7C}
\fill[CBTOFU36] (2.925,2.240) rectangle (2.975,2.460);
\definecolor{CBTOFU37}{HTML}{FECD7B}
\fill[CBTOFU37] (2.975,2.240) rectangle (3.025,2.460);
\definecolor{CBTOFU38}{HTML}{FECB79}
\fill[CBTOFU38] (3.025,2.240) rectangle (3.075,2.460);
\definecolor{CBTOFU39}{HTML}{FEC978}
\fill[CBTOFU39] (3.075,2.240) rectangle (3.125,2.460);
\definecolor{CBTOFU40}{HTML}{FDC776}
\fill[CBTOFU40] (3.125,2.240) rectangle (3.175,2.460);
\definecolor{CBTOFU41}{HTML}{FDC574}
\fill[CBTOFU41] (3.175,2.240) rectangle (3.225,2.460);
\definecolor{CBTOFU42}{HTML}{FDC373}
\fill[CBTOFU42] (3.225,2.240) rectangle (3.275,2.460);
\definecolor{CBTOFU43}{HTML}{FDC171}
\fill[CBTOFU43] (3.275,2.240) rectangle (3.325,2.460);
\definecolor{CBTOFU44}{HTML}{FDC070}
\fill[CBTOFU44] (3.325,2.240) rectangle (3.375,2.460);
\definecolor{CBTOFU45}{HTML}{FDBD6E}
\fill[CBTOFU45] (3.375,2.240) rectangle (3.425,2.460);
\definecolor{CBTOFU46}{HTML}{FDBC6C}
\fill[CBTOFU46] (3.425,2.240) rectangle (3.475,2.460);
\definecolor{CBTOFU47}{HTML}{FDBA6B}
\fill[CBTOFU47] (3.475,2.240) rectangle (3.525,2.460);
\definecolor{CBTOFU48}{HTML}{FDB869}
\fill[CBTOFU48] (3.525,2.240) rectangle (3.575,2.460);
\definecolor{CBTOFU49}{HTML}{FDB668}
\fill[CBTOFU49] (3.575,2.240) rectangle (3.625,2.460);
\definecolor{CBTOFU50}{HTML}{FDB466}
\fill[CBTOFU50] (3.625,2.240) rectangle (3.675,2.460);
\definecolor{CBTOFU51}{HTML}{FDB264}
\fill[CBTOFU51] (3.675,2.240) rectangle (3.725,2.460);
\definecolor{CBTOFU52}{HTML}{FDB063}
\fill[CBTOFU52] (3.725,2.240) rectangle (3.775,2.460);
\definecolor{CBTOFU53}{HTML}{FDAF61}
\fill[CBTOFU53] (3.775,2.240) rectangle (3.825,2.460);
\definecolor{CBTOFU54}{HTML}{FCAB60}
\fill[CBTOFU54] (3.825,2.240) rectangle (3.875,2.460);
\definecolor{CBTOFU55}{HTML}{FAA55D}
\fill[CBTOFU55] (3.875,2.240) rectangle (3.925,2.460);
\definecolor{CBTOFU56}{HTML}{F9A15B}
\fill[CBTOFU56] (3.925,2.240) rectangle (3.975,2.460);
\definecolor{CBTOFU57}{HTML}{F89C59}
\fill[CBTOFU57] (3.975,2.240) rectangle (4.025,2.460);
\definecolor{CBTOFU58}{HTML}{F69857}
\fill[CBTOFU58] (4.025,2.240) rectangle (4.075,2.460);
\definecolor{CBTOFU59}{HTML}{F59355}
\fill[CBTOFU59] (4.075,2.240) rectangle (4.125,2.460);
\definecolor{CBTOFU60}{HTML}{F38D52}
\fill[CBTOFU60] (4.125,2.240) rectangle (4.175,2.460);
\definecolor{CBTOFU61}{HTML}{F28950}
\fill[CBTOFU61] (4.175,2.240) rectangle (4.225,2.460);
\definecolor{CBTOFU62}{HTML}{F0844E}
\fill[CBTOFU62] (4.225,2.240) rectangle (4.275,2.460);
\definecolor{CBTOFU63}{HTML}{EF804C}
\fill[CBTOFU63] (4.275,2.240) rectangle (4.325,2.460);
\definecolor{CBTOFU64}{HTML}{EE7C4A}
\fill[CBTOFU64] (4.325,2.240) rectangle (4.375,2.460);
\definecolor{CBTOFU65}{HTML}{EC7647}
\fill[CBTOFU65] (4.375,2.240) rectangle (4.425,2.460);
\definecolor{CBTOFU66}{HTML}{EB7145}
\fill[CBTOFU66] (4.425,2.240) rectangle (4.475,2.460);
\definecolor{CBTOFU67}{HTML}{E96D43}
\fill[CBTOFU67] (4.475,2.240) rectangle (4.525,2.460);
\definecolor{CBTOFU68}{HTML}{E86841}
\fill[CBTOFU68] (4.525,2.240) rectangle (4.575,2.460);
\definecolor{CBTOFU69}{HTML}{E7643F}
\fill[CBTOFU69] (4.575,2.240) rectangle (4.625,2.460);
\definecolor{CBTOFU70}{HTML}{E55E3C}
\fill[CBTOFU70] (4.625,2.240) rectangle (4.675,2.460);
\definecolor{CBTOFU71}{HTML}{E45A3A}
\fill[CBTOFU71] (4.675,2.240) rectangle (4.725,2.460);
\definecolor{CBTOFU72}{HTML}{E25538}
\fill[CBTOFU72] (4.725,2.240) rectangle (4.775,2.460);
\definecolor{CBTOFU73}{HTML}{E15136}
\fill[CBTOFU73] (4.775,2.240) rectangle (4.825,2.460);
\definecolor{CBTOFU74}{HTML}{DF4C34}
\fill[CBTOFU74] (4.825,2.240) rectangle (4.875,2.460);
\definecolor{CBTOFU75}{HTML}{DE4631}
\fill[CBTOFU75] (4.875,2.240) rectangle (4.925,2.460);
\definecolor{CBTOFU76}{HTML}{DC422F}
\fill[CBTOFU76] (4.925,2.240) rectangle (4.975,2.460);
\definecolor{CBTOFU77}{HTML}{DB3D2D}
\fill[CBTOFU77] (4.975,2.240) rectangle (5.025,2.460);
\definecolor{CBTOFU78}{HTML}{DA392B}
\fill[CBTOFU78] (5.025,2.240) rectangle (5.075,2.460);
\definecolor{CBTOFU79}{HTML}{D83429}
\fill[CBTOFU79] (5.075,2.240) rectangle (5.125,2.460);
\draw[draw=AxisDark,line width=0.45pt] (1.125,2.240) rectangle (5.125,2.460);
\draw[draw=AxisDark,line width=0.45pt] (1.125,2.240) -- (1.125,2.160);
\node[anchor=north,font=\large] at (1.125,2.120) {88};
\draw[draw=AxisDark,line width=0.45pt] (3.125,2.240) -- (3.125,2.160);
\node[anchor=north,font=\large] at (3.125,2.120) {294};
\draw[draw=AxisDark,line width=0.45pt] (5.125,2.240) -- (5.125,2.160);
\node[anchor=north,font=\large] at (5.125,2.120) {983};
\node[font=\LARGE] at (10.775,9.970) {MUSE News};
\draw[draw=AxisDark,line width=0.65pt] (7.650,4.550) rectangle (13.900,9.350);
\draw[draw=GridGrey,line width=0.45pt] (8.002,4.550) -- (8.002,9.350);
\node[anchor=north,font=\large] at (8.002,4.390) {0.0};
\draw[draw=GridGrey,line width=0.45pt] (9.564,4.550) -- (9.564,9.350);
\node[anchor=north,font=\large] at (9.564,4.390) {0.2};
\draw[draw=GridGrey,line width=0.45pt] (11.127,4.550) -- (11.127,9.350);
\node[anchor=north,font=\large] at (11.127,4.390) {0.4};
\draw[draw=GridGrey,line width=0.45pt] (12.689,4.550) -- (12.689,9.350);
\node[anchor=north,font=\large] at (12.689,4.390) {0.6};
\draw[draw=GridGrey,line width=0.45pt] (7.650,4.743) -- (13.900,4.743);
\draw[draw=GridGrey,line width=0.45pt] (7.650,5.814) -- (13.900,5.814);
\draw[draw=GridGrey,line width=0.45pt] (7.650,6.886) -- (13.900,6.886);
\draw[draw=GridGrey,line width=0.45pt] (7.650,7.957) -- (13.900,7.957);
\draw[draw=GridGrey,line width=0.45pt] (7.650,9.029) -- (13.900,9.029);
\node[font=\Large] at (10.775,3.530) {Retain utility};
\draw[draw=AxisDark,line width=1.05pt] (8.002,9.029) -- (8.002,9.029) -- (9.809,8.590) -- (10.688,8.481) -- (12.323,7.897);
\draw[draw=NeutralEdge,fill=NeutralFill,line width=0.65pt] (12.291,6.378) circle[radius=0.150];
\draw[draw=NeutralEdge,fill=NeutralFill,line width=0.65pt] (12.173,7.747) rectangle (12.473,8.047);
\draw[draw=AxisDark,fill=CMUSENewsGA,line width=0.65pt] (8.002,9.179) -- (7.870,8.912) -- (8.134,8.912) -- cycle;
\draw[draw=AxisDark,fill=CMUSENewsGradDiff,line width=0.65pt] (10.688,8.331) -- (10.556,8.598) -- (10.819,8.598) -- cycle;
\draw[draw=AxisDark,fill=CMUSENewsNPO,line width=0.65pt] (7.957,9.179) -- (8.047,9.179) -- (8.047,9.074) -- (8.152,9.074) -- (8.152,8.984) -- (8.047,8.984) -- (8.047,8.879) -- (7.957,8.879) -- (7.957,8.984) -- (7.852,8.984) -- (7.852,9.074) -- (7.957,9.074) -- cycle;
\draw[draw=AxisDark,line width=2.20pt,line cap=round] (11.332,7.649) -- (11.602,7.919);
\draw[draw=AxisDark,line width=2.20pt,line cap=round] (11.332,7.919) -- (11.602,7.649);
\draw[draw=CMUSENewsNPOKLR,line width=1.25pt,line cap=round] (11.332,7.649) -- (11.602,7.919);
\draw[draw=CMUSENewsNPOKLR,line width=1.25pt,line cap=round] (11.332,7.919) -- (11.602,7.649);
\draw[draw=AxisDark,fill=CMUSENewsRMU,line width=0.65pt] (11.279,7.572) -- (11.429,7.422) -- (11.279,7.272) -- (11.129,7.422) -- cycle;
\draw[draw=AxisDark,fill=CMUSENewsSimNPO,line width=0.65pt] (11.024,7.950) -- (10.894,8.025) -- (10.764,7.950) -- (10.764,7.800) -- (10.894,7.725) -- (11.024,7.800) -- cycle;
\draw[draw=AxisDark,fill=CMUSENewsMASC,line width=0.80pt] (9.809,8.740) -- (9.771,8.642) -- (9.667,8.636) -- (9.748,8.570) -- (9.721,8.469) -- (9.809,8.525) -- (9.898,8.469) -- (9.871,8.570) -- (9.952,8.636) -- (9.847,8.642) -- cycle;
\node[font=\large] at (10.775,2.880) {Runtime (s)};
\definecolor{CBMUSENews0}{HTML}{1A9850}
\fill[CBMUSENews0] (8.775,2.240) rectangle (8.825,2.460);
\definecolor{CBMUSENews1}{HTML}{229B52}
\fill[CBMUSENews1] (8.825,2.240) rectangle (8.875,2.460);
\definecolor{CBMUSENews2}{HTML}{2A9D54}
\fill[CBMUSENews2] (8.875,2.240) rectangle (8.925,2.460);
\definecolor{CBMUSENews3}{HTML}{32A056}
\fill[CBMUSENews3] (8.925,2.240) rectangle (8.975,2.460);
\definecolor{CBMUSENews4}{HTML}{3AA258}
\fill[CBMUSENews4] (8.975,2.240) rectangle (9.025,2.460);
\definecolor{CBMUSENews5}{HTML}{45A65B}
\fill[CBMUSENews5] (9.025,2.240) rectangle (9.075,2.460);
\definecolor{CBMUSENews6}{HTML}{4DA85D}
\fill[CBMUSENews6] (9.075,2.240) rectangle (9.125,2.460);
\definecolor{CBMUSENews7}{HTML}{55AB5F}
\fill[CBMUSENews7] (9.125,2.240) rectangle (9.175,2.460);
\definecolor{CBMUSENews8}{HTML}{5DAD61}
\fill[CBMUSENews8] (9.175,2.240) rectangle (9.225,2.460);
\definecolor{CBMUSENews9}{HTML}{65B063}
\fill[CBMUSENews9] (9.225,2.240) rectangle (9.275,2.460);
\definecolor{CBMUSENews10}{HTML}{70B366}
\fill[CBMUSENews10] (9.275,2.240) rectangle (9.325,2.460);
\definecolor{CBMUSENews11}{HTML}{78B668}
\fill[CBMUSENews11] (9.325,2.240) rectangle (9.375,2.460);
\definecolor{CBMUSENews12}{HTML}{80B86A}
\fill[CBMUSENews12] (9.375,2.240) rectangle (9.425,2.460);
\definecolor{CBMUSENews13}{HTML}{88BB6C}
\fill[CBMUSENews13] (9.425,2.240) rectangle (9.475,2.460);
\definecolor{CBMUSENews14}{HTML}{90BD6F}
\fill[CBMUSENews14] (9.475,2.240) rectangle (9.525,2.460);
\definecolor{CBMUSENews15}{HTML}{9BC171}
\fill[CBMUSENews15] (9.525,2.240) rectangle (9.575,2.460);
\definecolor{CBMUSENews16}{HTML}{A3C373}
\fill[CBMUSENews16] (9.575,2.240) rectangle (9.625,2.460);
\definecolor{CBMUSENews17}{HTML}{ABC675}
\fill[CBMUSENews17] (9.625,2.240) rectangle (9.675,2.460);
\definecolor{CBMUSENews18}{HTML}{B3C878}
\fill[CBMUSENews18] (9.675,2.240) rectangle (9.725,2.460);
\definecolor{CBMUSENews19}{HTML}{BBCB7A}
\fill[CBMUSENews19] (9.725,2.240) rectangle (9.775,2.460);
\definecolor{CBMUSENews20}{HTML}{C6CE7C}
\fill[CBMUSENews20] (9.775,2.240) rectangle (9.825,2.460);
\definecolor{CBMUSENews21}{HTML}{CED17F}
\fill[CBMUSENews21] (9.825,2.240) rectangle (9.875,2.460);
\definecolor{CBMUSENews22}{HTML}{D6D381}
\fill[CBMUSENews22] (9.875,2.240) rectangle (9.925,2.460);
\definecolor{CBMUSENews23}{HTML}{DED683}
\fill[CBMUSENews23] (9.925,2.240) rectangle (9.975,2.460);
\definecolor{CBMUSENews24}{HTML}{E6D885}
\fill[CBMUSENews24] (9.975,2.240) rectangle (10.025,2.460);
\definecolor{CBMUSENews25}{HTML}{F1DC88}
\fill[CBMUSENews25] (10.025,2.240) rectangle (10.075,2.460);
\definecolor{CBMUSENews26}{HTML}{F9DE8A}
\fill[CBMUSENews26] (10.075,2.240) rectangle (10.125,2.460);
\definecolor{CBMUSENews27}{HTML}{FEDF8B}
\fill[CBMUSENews27] (10.125,2.240) rectangle (10.175,2.460);
\definecolor{CBMUSENews28}{HTML}{FEDE89}
\fill[CBMUSENews28] (10.175,2.240) rectangle (10.225,2.460);
\definecolor{CBMUSENews29}{HTML}{FEDC88}
\fill[CBMUSENews29] (10.225,2.240) rectangle (10.275,2.460);
\definecolor{CBMUSENews30}{HTML}{FEDA86}
\fill[CBMUSENews30] (10.275,2.240) rectangle (10.325,2.460);
\definecolor{CBMUSENews31}{HTML}{FED884}
\fill[CBMUSENews31] (10.325,2.240) rectangle (10.375,2.460);
\definecolor{CBMUSENews32}{HTML}{FED683}
\fill[CBMUSENews32] (10.375,2.240) rectangle (10.425,2.460);
\definecolor{CBMUSENews33}{HTML}{FED481}
\fill[CBMUSENews33] (10.425,2.240) rectangle (10.475,2.460);
\definecolor{CBMUSENews34}{HTML}{FED280}
\fill[CBMUSENews34] (10.475,2.240) rectangle (10.525,2.460);
\definecolor{CBMUSENews35}{HTML}{FED07E}
\fill[CBMUSENews35] (10.525,2.240) rectangle (10.575,2.460);
\definecolor{CBMUSENews36}{HTML}{FECE7C}
\fill[CBMUSENews36] (10.575,2.240) rectangle (10.625,2.460);
\definecolor{CBMUSENews37}{HTML}{FECD7B}
\fill[CBMUSENews37] (10.625,2.240) rectangle (10.675,2.460);
\definecolor{CBMUSENews38}{HTML}{FECB79}
\fill[CBMUSENews38] (10.675,2.240) rectangle (10.725,2.460);
\definecolor{CBMUSENews39}{HTML}{FEC978}
\fill[CBMUSENews39] (10.725,2.240) rectangle (10.775,2.460);
\definecolor{CBMUSENews40}{HTML}{FDC776}
\fill[CBMUSENews40] (10.775,2.240) rectangle (10.825,2.460);
\definecolor{CBMUSENews41}{HTML}{FDC574}
\fill[CBMUSENews41] (10.825,2.240) rectangle (10.875,2.460);
\definecolor{CBMUSENews42}{HTML}{FDC373}
\fill[CBMUSENews42] (10.875,2.240) rectangle (10.925,2.460);
\definecolor{CBMUSENews43}{HTML}{FDC171}
\fill[CBMUSENews43] (10.925,2.240) rectangle (10.975,2.460);
\definecolor{CBMUSENews44}{HTML}{FDC070}
\fill[CBMUSENews44] (10.975,2.240) rectangle (11.025,2.460);
\definecolor{CBMUSENews45}{HTML}{FDBD6E}
\fill[CBMUSENews45] (11.025,2.240) rectangle (11.075,2.460);
\definecolor{CBMUSENews46}{HTML}{FDBC6C}
\fill[CBMUSENews46] (11.075,2.240) rectangle (11.125,2.460);
\definecolor{CBMUSENews47}{HTML}{FDBA6B}
\fill[CBMUSENews47] (11.125,2.240) rectangle (11.175,2.460);
\definecolor{CBMUSENews48}{HTML}{FDB869}
\fill[CBMUSENews48] (11.175,2.240) rectangle (11.225,2.460);
\definecolor{CBMUSENews49}{HTML}{FDB668}
\fill[CBMUSENews49] (11.225,2.240) rectangle (11.275,2.460);
\definecolor{CBMUSENews50}{HTML}{FDB466}
\fill[CBMUSENews50] (11.275,2.240) rectangle (11.325,2.460);
\definecolor{CBMUSENews51}{HTML}{FDB264}
\fill[CBMUSENews51] (11.325,2.240) rectangle (11.375,2.460);
\definecolor{CBMUSENews52}{HTML}{FDB063}
\fill[CBMUSENews52] (11.375,2.240) rectangle (11.425,2.460);
\definecolor{CBMUSENews53}{HTML}{FDAF61}
\fill[CBMUSENews53] (11.425,2.240) rectangle (11.475,2.460);
\definecolor{CBMUSENews54}{HTML}{FCAB60}
\fill[CBMUSENews54] (11.475,2.240) rectangle (11.525,2.460);
\definecolor{CBMUSENews55}{HTML}{FAA55D}
\fill[CBMUSENews55] (11.525,2.240) rectangle (11.575,2.460);
\definecolor{CBMUSENews56}{HTML}{F9A15B}
\fill[CBMUSENews56] (11.575,2.240) rectangle (11.625,2.460);
\definecolor{CBMUSENews57}{HTML}{F89C59}
\fill[CBMUSENews57] (11.625,2.240) rectangle (11.675,2.460);
\definecolor{CBMUSENews58}{HTML}{F69857}
\fill[CBMUSENews58] (11.675,2.240) rectangle (11.725,2.460);
\definecolor{CBMUSENews59}{HTML}{F59355}
\fill[CBMUSENews59] (11.725,2.240) rectangle (11.775,2.460);
\definecolor{CBMUSENews60}{HTML}{F38D52}
\fill[CBMUSENews60] (11.775,2.240) rectangle (11.825,2.460);
\definecolor{CBMUSENews61}{HTML}{F28950}
\fill[CBMUSENews61] (11.825,2.240) rectangle (11.875,2.460);
\definecolor{CBMUSENews62}{HTML}{F0844E}
\fill[CBMUSENews62] (11.875,2.240) rectangle (11.925,2.460);
\definecolor{CBMUSENews63}{HTML}{EF804C}
\fill[CBMUSENews63] (11.925,2.240) rectangle (11.975,2.460);
\definecolor{CBMUSENews64}{HTML}{EE7C4A}
\fill[CBMUSENews64] (11.975,2.240) rectangle (12.025,2.460);
\definecolor{CBMUSENews65}{HTML}{EC7647}
\fill[CBMUSENews65] (12.025,2.240) rectangle (12.075,2.460);
\definecolor{CBMUSENews66}{HTML}{EB7145}
\fill[CBMUSENews66] (12.075,2.240) rectangle (12.125,2.460);
\definecolor{CBMUSENews67}{HTML}{E96D43}
\fill[CBMUSENews67] (12.125,2.240) rectangle (12.175,2.460);
\definecolor{CBMUSENews68}{HTML}{E86841}
\fill[CBMUSENews68] (12.175,2.240) rectangle (12.225,2.460);
\definecolor{CBMUSENews69}{HTML}{E7643F}
\fill[CBMUSENews69] (12.225,2.240) rectangle (12.275,2.460);
\definecolor{CBMUSENews70}{HTML}{E55E3C}
\fill[CBMUSENews70] (12.275,2.240) rectangle (12.325,2.460);
\definecolor{CBMUSENews71}{HTML}{E45A3A}
\fill[CBMUSENews71] (12.325,2.240) rectangle (12.375,2.460);
\definecolor{CBMUSENews72}{HTML}{E25538}
\fill[CBMUSENews72] (12.375,2.240) rectangle (12.425,2.460);
\definecolor{CBMUSENews73}{HTML}{E15136}
\fill[CBMUSENews73] (12.425,2.240) rectangle (12.475,2.460);
\definecolor{CBMUSENews74}{HTML}{DF4C34}
\fill[CBMUSENews74] (12.475,2.240) rectangle (12.525,2.460);
\definecolor{CBMUSENews75}{HTML}{DE4631}
\fill[CBMUSENews75] (12.525,2.240) rectangle (12.575,2.460);
\definecolor{CBMUSENews76}{HTML}{DC422F}
\fill[CBMUSENews76] (12.575,2.240) rectangle (12.625,2.460);
\definecolor{CBMUSENews77}{HTML}{DB3D2D}
\fill[CBMUSENews77] (12.625,2.240) rectangle (12.675,2.460);
\definecolor{CBMUSENews78}{HTML}{DA392B}
\fill[CBMUSENews78] (12.675,2.240) rectangle (12.725,2.460);
\definecolor{CBMUSENews79}{HTML}{D83429}
\fill[CBMUSENews79] (12.725,2.240) rectangle (12.775,2.460);
\draw[draw=AxisDark,line width=0.45pt] (8.775,2.240) rectangle (12.775,2.460);
\draw[draw=AxisDark,line width=0.45pt] (8.775,2.240) -- (8.775,2.160);
\node[anchor=north,font=\large] at (8.775,2.120) {139};
\draw[draw=AxisDark,line width=0.45pt] (10.775,2.240) -- (10.775,2.160);
\node[anchor=north,font=\large] at (10.775,2.120) {751};
\draw[draw=AxisDark,line width=0.45pt] (12.775,2.240) -- (12.775,2.160);
\node[anchor=north,font=\large] at (12.775,2.120) {4063};
\node[font=\LARGE] at (18.425,9.970) {MUSE Books};
\draw[draw=AxisDark,line width=0.65pt] (15.300,4.550) rectangle (21.550,9.350);
\draw[draw=GridGrey,line width=0.45pt] (15.652,4.550) -- (15.652,9.350);
\node[anchor=north,font=\large] at (15.652,4.390) {0.0};
\draw[draw=GridGrey,line width=0.45pt] (17.214,4.550) -- (17.214,9.350);
\node[anchor=north,font=\large] at (17.214,4.390) {0.2};
\draw[draw=GridGrey,line width=0.45pt] (18.777,4.550) -- (18.777,9.350);
\node[anchor=north,font=\large] at (18.777,4.390) {0.4};
\draw[draw=GridGrey,line width=0.45pt] (20.339,4.550) -- (20.339,9.350);
\node[anchor=north,font=\large] at (20.339,4.390) {0.6};
\draw[draw=GridGrey,line width=0.45pt] (15.300,4.743) -- (21.550,4.743);
\draw[draw=GridGrey,line width=0.45pt] (15.300,5.814) -- (21.550,5.814);
\draw[draw=GridGrey,line width=0.45pt] (15.300,6.886) -- (21.550,6.886);
\draw[draw=GridGrey,line width=0.45pt] (15.300,7.957) -- (21.550,7.957);
\draw[draw=GridGrey,line width=0.45pt] (15.300,9.029) -- (21.550,9.029);
\node[font=\Large] at (18.425,3.530) {Retain utility};
\draw[draw=AxisDark,line width=1.05pt] (19.385,9.029) -- (20.944,8.528) -- (21.022,8.070) -- (21.052,5.882);
\draw[draw=NeutralEdge,fill=NeutralFill,line width=0.65pt] (21.052,5.882) circle[radius=0.150];
\draw[draw=NeutralEdge,fill=NeutralFill,line width=0.65pt] (20.872,7.920) rectangle (21.172,8.220);
\draw[draw=AxisDark,fill=CMUSEBooksGA,line width=0.65pt] (15.652,9.179) -- (15.520,8.912) -- (15.784,8.912) -- cycle;
\draw[draw=AxisDark,fill=CMUSEBooksGradDiff,line width=0.65pt] (18.873,8.879) -- (18.741,9.146) -- (19.005,9.146) -- cycle;
\draw[draw=AxisDark,fill=CMUSEBooksNPO,line width=0.65pt] (15.607,9.179) -- (15.697,9.179) -- (15.697,9.074) -- (15.802,9.074) -- (15.802,8.984) -- (15.697,8.984) -- (15.697,8.879) -- (15.607,8.879) -- (15.607,8.984) -- (15.502,8.984) -- (15.502,9.074) -- (15.607,9.074) -- cycle;
\draw[draw=AxisDark,line width=2.20pt,line cap=round] (20.809,8.393) -- (21.079,8.663);
\draw[draw=AxisDark,line width=2.20pt,line cap=round] (20.809,8.663) -- (21.079,8.393);
\draw[draw=CMUSEBooksNPOKLR,line width=1.25pt,line cap=round] (20.809,8.393) -- (21.079,8.663);
\draw[draw=CMUSEBooksNPOKLR,line width=1.25pt,line cap=round] (20.809,8.663) -- (21.079,8.393);
\draw[draw=AxisDark,fill=CMUSEBooksRMU,line width=0.65pt] (20.398,8.462) -- (20.548,8.312) -- (20.398,8.162) -- (20.248,8.312) -- cycle;
\draw[draw=AxisDark,fill=CMUSEBooksSimNPO,line width=0.65pt] (19.515,9.104) -- (19.385,9.179) -- (19.255,9.104) -- (19.255,8.954) -- (19.385,8.879) -- (19.515,8.954) -- cycle;
\draw[draw=AxisDark,fill=CMUSEBooksMASC,line width=0.80pt] (20.753,8.497) -- (20.715,8.399) -- (20.610,8.393) -- (20.692,8.327) -- (20.665,8.226) -- (20.753,8.283) -- (20.841,8.226) -- (20.814,8.327) -- (20.896,8.393) -- (20.791,8.399) -- cycle;
\node[font=\large] at (18.425,2.880) {Runtime (s)};
\definecolor{CBMUSEBooks0}{HTML}{1A9850}
\fill[CBMUSEBooks0] (16.425,2.240) rectangle (16.475,2.460);
\definecolor{CBMUSEBooks1}{HTML}{229B52}
\fill[CBMUSEBooks1] (16.475,2.240) rectangle (16.525,2.460);
\definecolor{CBMUSEBooks2}{HTML}{2A9D54}
\fill[CBMUSEBooks2] (16.525,2.240) rectangle (16.575,2.460);
\definecolor{CBMUSEBooks3}{HTML}{32A056}
\fill[CBMUSEBooks3] (16.575,2.240) rectangle (16.625,2.460);
\definecolor{CBMUSEBooks4}{HTML}{3AA258}
\fill[CBMUSEBooks4] (16.625,2.240) rectangle (16.675,2.460);
\definecolor{CBMUSEBooks5}{HTML}{45A65B}
\fill[CBMUSEBooks5] (16.675,2.240) rectangle (16.725,2.460);
\definecolor{CBMUSEBooks6}{HTML}{4DA85D}
\fill[CBMUSEBooks6] (16.725,2.240) rectangle (16.775,2.460);
\definecolor{CBMUSEBooks7}{HTML}{55AB5F}
\fill[CBMUSEBooks7] (16.775,2.240) rectangle (16.825,2.460);
\definecolor{CBMUSEBooks8}{HTML}{5DAD61}
\fill[CBMUSEBooks8] (16.825,2.240) rectangle (16.875,2.460);
\definecolor{CBMUSEBooks9}{HTML}{65B063}
\fill[CBMUSEBooks9] (16.875,2.240) rectangle (16.925,2.460);
\definecolor{CBMUSEBooks10}{HTML}{70B366}
\fill[CBMUSEBooks10] (16.925,2.240) rectangle (16.975,2.460);
\definecolor{CBMUSEBooks11}{HTML}{78B668}
\fill[CBMUSEBooks11] (16.975,2.240) rectangle (17.025,2.460);
\definecolor{CBMUSEBooks12}{HTML}{80B86A}
\fill[CBMUSEBooks12] (17.025,2.240) rectangle (17.075,2.460);
\definecolor{CBMUSEBooks13}{HTML}{88BB6C}
\fill[CBMUSEBooks13] (17.075,2.240) rectangle (17.125,2.460);
\definecolor{CBMUSEBooks14}{HTML}{90BD6F}
\fill[CBMUSEBooks14] (17.125,2.240) rectangle (17.175,2.460);
\definecolor{CBMUSEBooks15}{HTML}{9BC171}
\fill[CBMUSEBooks15] (17.175,2.240) rectangle (17.225,2.460);
\definecolor{CBMUSEBooks16}{HTML}{A3C373}
\fill[CBMUSEBooks16] (17.225,2.240) rectangle (17.275,2.460);
\definecolor{CBMUSEBooks17}{HTML}{ABC675}
\fill[CBMUSEBooks17] (17.275,2.240) rectangle (17.325,2.460);
\definecolor{CBMUSEBooks18}{HTML}{B3C878}
\fill[CBMUSEBooks18] (17.325,2.240) rectangle (17.375,2.460);
\definecolor{CBMUSEBooks19}{HTML}{BBCB7A}
\fill[CBMUSEBooks19] (17.375,2.240) rectangle (17.425,2.460);
\definecolor{CBMUSEBooks20}{HTML}{C6CE7C}
\fill[CBMUSEBooks20] (17.425,2.240) rectangle (17.475,2.460);
\definecolor{CBMUSEBooks21}{HTML}{CED17F}
\fill[CBMUSEBooks21] (17.475,2.240) rectangle (17.525,2.460);
\definecolor{CBMUSEBooks22}{HTML}{D6D381}
\fill[CBMUSEBooks22] (17.525,2.240) rectangle (17.575,2.460);
\definecolor{CBMUSEBooks23}{HTML}{DED683}
\fill[CBMUSEBooks23] (17.575,2.240) rectangle (17.625,2.460);
\definecolor{CBMUSEBooks24}{HTML}{E6D885}
\fill[CBMUSEBooks24] (17.625,2.240) rectangle (17.675,2.460);
\definecolor{CBMUSEBooks25}{HTML}{F1DC88}
\fill[CBMUSEBooks25] (17.675,2.240) rectangle (17.725,2.460);
\definecolor{CBMUSEBooks26}{HTML}{F9DE8A}
\fill[CBMUSEBooks26] (17.725,2.240) rectangle (17.775,2.460);
\definecolor{CBMUSEBooks27}{HTML}{FEDF8B}
\fill[CBMUSEBooks27] (17.775,2.240) rectangle (17.825,2.460);
\definecolor{CBMUSEBooks28}{HTML}{FEDE89}
\fill[CBMUSEBooks28] (17.825,2.240) rectangle (17.875,2.460);
\definecolor{CBMUSEBooks29}{HTML}{FEDC88}
\fill[CBMUSEBooks29] (17.875,2.240) rectangle (17.925,2.460);
\definecolor{CBMUSEBooks30}{HTML}{FEDA86}
\fill[CBMUSEBooks30] (17.925,2.240) rectangle (17.975,2.460);
\definecolor{CBMUSEBooks31}{HTML}{FED884}
\fill[CBMUSEBooks31] (17.975,2.240) rectangle (18.025,2.460);
\definecolor{CBMUSEBooks32}{HTML}{FED683}
\fill[CBMUSEBooks32] (18.025,2.240) rectangle (18.075,2.460);
\definecolor{CBMUSEBooks33}{HTML}{FED481}
\fill[CBMUSEBooks33] (18.075,2.240) rectangle (18.125,2.460);
\definecolor{CBMUSEBooks34}{HTML}{FED280}
\fill[CBMUSEBooks34] (18.125,2.240) rectangle (18.175,2.460);
\definecolor{CBMUSEBooks35}{HTML}{FED07E}
\fill[CBMUSEBooks35] (18.175,2.240) rectangle (18.225,2.460);
\definecolor{CBMUSEBooks36}{HTML}{FECE7C}
\fill[CBMUSEBooks36] (18.225,2.240) rectangle (18.275,2.460);
\definecolor{CBMUSEBooks37}{HTML}{FECD7B}
\fill[CBMUSEBooks37] (18.275,2.240) rectangle (18.325,2.460);
\definecolor{CBMUSEBooks38}{HTML}{FECB79}
\fill[CBMUSEBooks38] (18.325,2.240) rectangle (18.375,2.460);
\definecolor{CBMUSEBooks39}{HTML}{FEC978}
\fill[CBMUSEBooks39] (18.375,2.240) rectangle (18.425,2.460);
\definecolor{CBMUSEBooks40}{HTML}{FDC776}
\fill[CBMUSEBooks40] (18.425,2.240) rectangle (18.475,2.460);
\definecolor{CBMUSEBooks41}{HTML}{FDC574}
\fill[CBMUSEBooks41] (18.475,2.240) rectangle (18.525,2.460);
\definecolor{CBMUSEBooks42}{HTML}{FDC373}
\fill[CBMUSEBooks42] (18.525,2.240) rectangle (18.575,2.460);
\definecolor{CBMUSEBooks43}{HTML}{FDC171}
\fill[CBMUSEBooks43] (18.575,2.240) rectangle (18.625,2.460);
\definecolor{CBMUSEBooks44}{HTML}{FDC070}
\fill[CBMUSEBooks44] (18.625,2.240) rectangle (18.675,2.460);
\definecolor{CBMUSEBooks45}{HTML}{FDBD6E}
\fill[CBMUSEBooks45] (18.675,2.240) rectangle (18.725,2.460);
\definecolor{CBMUSEBooks46}{HTML}{FDBC6C}
\fill[CBMUSEBooks46] (18.725,2.240) rectangle (18.775,2.460);
\definecolor{CBMUSEBooks47}{HTML}{FDBA6B}
\fill[CBMUSEBooks47] (18.775,2.240) rectangle (18.825,2.460);
\definecolor{CBMUSEBooks48}{HTML}{FDB869}
\fill[CBMUSEBooks48] (18.825,2.240) rectangle (18.875,2.460);
\definecolor{CBMUSEBooks49}{HTML}{FDB668}
\fill[CBMUSEBooks49] (18.875,2.240) rectangle (18.925,2.460);
\definecolor{CBMUSEBooks50}{HTML}{FDB466}
\fill[CBMUSEBooks50] (18.925,2.240) rectangle (18.975,2.460);
\definecolor{CBMUSEBooks51}{HTML}{FDB264}
\fill[CBMUSEBooks51] (18.975,2.240) rectangle (19.025,2.460);
\definecolor{CBMUSEBooks52}{HTML}{FDB063}
\fill[CBMUSEBooks52] (19.025,2.240) rectangle (19.075,2.460);
\definecolor{CBMUSEBooks53}{HTML}{FDAF61}
\fill[CBMUSEBooks53] (19.075,2.240) rectangle (19.125,2.460);
\definecolor{CBMUSEBooks54}{HTML}{FCAB60}
\fill[CBMUSEBooks54] (19.125,2.240) rectangle (19.175,2.460);
\definecolor{CBMUSEBooks55}{HTML}{FAA55D}
\fill[CBMUSEBooks55] (19.175,2.240) rectangle (19.225,2.460);
\definecolor{CBMUSEBooks56}{HTML}{F9A15B}
\fill[CBMUSEBooks56] (19.225,2.240) rectangle (19.275,2.460);
\definecolor{CBMUSEBooks57}{HTML}{F89C59}
\fill[CBMUSEBooks57] (19.275,2.240) rectangle (19.325,2.460);
\definecolor{CBMUSEBooks58}{HTML}{F69857}
\fill[CBMUSEBooks58] (19.325,2.240) rectangle (19.375,2.460);
\definecolor{CBMUSEBooks59}{HTML}{F59355}
\fill[CBMUSEBooks59] (19.375,2.240) rectangle (19.425,2.460);
\definecolor{CBMUSEBooks60}{HTML}{F38D52}
\fill[CBMUSEBooks60] (19.425,2.240) rectangle (19.475,2.460);
\definecolor{CBMUSEBooks61}{HTML}{F28950}
\fill[CBMUSEBooks61] (19.475,2.240) rectangle (19.525,2.460);
\definecolor{CBMUSEBooks62}{HTML}{F0844E}
\fill[CBMUSEBooks62] (19.525,2.240) rectangle (19.575,2.460);
\definecolor{CBMUSEBooks63}{HTML}{EF804C}
\fill[CBMUSEBooks63] (19.575,2.240) rectangle (19.625,2.460);
\definecolor{CBMUSEBooks64}{HTML}{EE7C4A}
\fill[CBMUSEBooks64] (19.625,2.240) rectangle (19.675,2.460);
\definecolor{CBMUSEBooks65}{HTML}{EC7647}
\fill[CBMUSEBooks65] (19.675,2.240) rectangle (19.725,2.460);
\definecolor{CBMUSEBooks66}{HTML}{EB7145}
\fill[CBMUSEBooks66] (19.725,2.240) rectangle (19.775,2.460);
\definecolor{CBMUSEBooks67}{HTML}{E96D43}
\fill[CBMUSEBooks67] (19.775,2.240) rectangle (19.825,2.460);
\definecolor{CBMUSEBooks68}{HTML}{E86841}
\fill[CBMUSEBooks68] (19.825,2.240) rectangle (19.875,2.460);
\definecolor{CBMUSEBooks69}{HTML}{E7643F}
\fill[CBMUSEBooks69] (19.875,2.240) rectangle (19.925,2.460);
\definecolor{CBMUSEBooks70}{HTML}{E55E3C}
\fill[CBMUSEBooks70] (19.925,2.240) rectangle (19.975,2.460);
\definecolor{CBMUSEBooks71}{HTML}{E45A3A}
\fill[CBMUSEBooks71] (19.975,2.240) rectangle (20.025,2.460);
\definecolor{CBMUSEBooks72}{HTML}{E25538}
\fill[CBMUSEBooks72] (20.025,2.240) rectangle (20.075,2.460);
\definecolor{CBMUSEBooks73}{HTML}{E15136}
\fill[CBMUSEBooks73] (20.075,2.240) rectangle (20.125,2.460);
\definecolor{CBMUSEBooks74}{HTML}{DF4C34}
\fill[CBMUSEBooks74] (20.125,2.240) rectangle (20.175,2.460);
\definecolor{CBMUSEBooks75}{HTML}{DE4631}
\fill[CBMUSEBooks75] (20.175,2.240) rectangle (20.225,2.460);
\definecolor{CBMUSEBooks76}{HTML}{DC422F}
\fill[CBMUSEBooks76] (20.225,2.240) rectangle (20.275,2.460);
\definecolor{CBMUSEBooks77}{HTML}{DB3D2D}
\fill[CBMUSEBooks77] (20.275,2.240) rectangle (20.325,2.460);
\definecolor{CBMUSEBooks78}{HTML}{DA392B}
\fill[CBMUSEBooks78] (20.325,2.240) rectangle (20.375,2.460);
\definecolor{CBMUSEBooks79}{HTML}{D83429}
\fill[CBMUSEBooks79] (20.375,2.240) rectangle (20.425,2.460);
\draw[draw=AxisDark,line width=0.45pt] (16.425,2.240) rectangle (20.425,2.460);
\draw[draw=AxisDark,line width=0.45pt] (16.425,2.240) -- (16.425,2.160);
\node[anchor=north,font=\large] at (16.425,2.120) {65};
\draw[draw=AxisDark,line width=0.45pt] (18.425,2.240) -- (18.425,2.160);
\node[anchor=north,font=\large] at (18.425,2.120) {415};
\draw[draw=AxisDark,line width=0.45pt] (20.425,2.240) -- (20.425,2.160);
\node[anchor=north,font=\large] at (20.425,2.120) {2648};
\node[anchor=center,align=center,font=\Large] at (10.775,0.620) {\makebox[17.40cm][c]{\makebox[3.25cm][c]{\tikz[baseline=-0.55ex,x=1.85cm,y=1.85cm]{\draw[draw=AxisDark,fill=NeutralFill,line width=0.45pt] (0,0) circle[radius=0.055];}\hspace{0.35em}Base}\makebox[3.25cm][c]{\tikz[baseline=-0.55ex,x=1.85cm,y=1.85cm]{\draw[draw=AxisDark,fill=NeutralFill,line width=0.45pt] (-0.055,-0.055) rectangle (0.055,0.055);}\hspace{0.35em}Retrain}\makebox[3.25cm][c]{\tikz[baseline=-0.55ex,x=1.85cm,y=1.85cm]{\draw[draw=AxisDark,fill=NeutralFill,line width=0.50pt] (0,0.070) -- (-0.062,-0.055) -- (0.062,-0.055) -- cycle;}\hspace{0.35em}GA}\makebox[3.25cm][c]{\tikz[baseline=-0.55ex,x=1.85cm,y=1.85cm]{\draw[draw=AxisDark,fill=NeutralFill,line width=0.50pt] (0,-0.070) -- (-0.062,0.055) -- (0.062,0.055) -- cycle;}\hspace{0.35em}GradDiff}\makebox[3.25cm][c]{\tikz[baseline=-0.55ex,x=1.85cm,y=1.85cm]{\draw[draw=AxisDark,line width=1.15pt,line cap=round] (-0.065,0) -- (0.065,0);\draw[draw=AxisDark,line width=1.15pt,line cap=round] (0,-0.065) -- (0,0.065);}\hspace{0.35em}NPO}}\\[0.24em]\makebox[17.40cm][c]{\makebox[3.25cm][c]{\tikz[baseline=-0.55ex,x=1.85cm,y=1.85cm]{\draw[draw=AxisDark,line width=1.15pt,line cap=round] (-0.060,-0.060) -- (0.060,0.060);\draw[draw=AxisDark,line width=1.15pt,line cap=round] (-0.060,0.060) -- (0.060,-0.060);}\hspace{0.35em}NPO+KLR}\makebox[3.25cm][c]{\tikz[baseline=-0.55ex,x=1.85cm,y=1.85cm]{\draw[draw=AxisDark,fill=NeutralFill,line width=0.50pt] (0,0.070) -- (0.070,0) -- (0,-0.070) -- (-0.070,0) -- cycle;}\hspace{0.35em}RMU}\makebox[3.25cm][c]{\tikz[baseline=-0.55ex,x=1.85cm,y=1.85cm]{\draw[draw=AxisDark,fill=NeutralFill,line width=0.50pt] (0.061,0.035) -- (0,0.070) -- (-0.061,0.035) -- (-0.061,-0.035) -- (0,-0.070) -- (0.061,-0.035) -- cycle;}\hspace{0.35em}SimNPO}\makebox[3.25cm][c]{\tikz[baseline=-0.55ex,x=1.85cm,y=1.85cm]{\draw[draw=AxisDark,fill=NeutralFill,line width=0.50pt] (0,0.090) -- (-0.026,0.036) -- (-0.086,0.028) -- (-0.042,-0.014) -- (-0.053,-0.073) -- (0,-0.045) -- (0.053,-0.073) -- (0.042,-0.014) -- (0.086,0.028) -- (0.026,0.036) -- cycle;}\hspace{0.35em}MASC}}};
\end{tikzpicture}

%% file: plots/tofu_masc_lr.pgf
\begin{tikzpicture}
\begin{axis}[
    width=0.96\linewidth,
    height=0.55\linewidth,
    xmin=-2, xmax=88,
    ymin=-0.03, ymax=1.04,
    xlabel={Optimizer step},
    ylabel={$\hat V_\rho$},
    axis lines=left,
    grid=major,
    grid style={draw=gray!25},
    tick align=outside,
    tick style={line width=0.9pt, black},
    axis line style={line width=1.1pt},
    xtick={0,20,40,60,80},
    ytick={0,0.5,1.0},
    yticklabels={0,0.5,1.0},
    xlabel style={font=\large},
    ylabel style={font=\large},
    tick label style={font=\normalsize},
    legend style={draw=none, fill=none, font=\normalsize, at={(0.97,0.97)}, anchor=north east},
    legend cell align={left},
]
\addplot[color=cyan!35!blue!45, mark=*, mark size=1.8pt, line width=1.4pt] coordinates {(0,0.998393) (5,0.998418) (10,0.998075) (15,0.994883) (20,0.979436) (25,0.909190) (30,0.770185) (35,0.631471) (40,0.486174) (45,0.281361) (50,0.073905) (55,0.000815) (60,0.000000) (65,0.000000) (70,0.000000) (75,0.000000) (80,0.000000) (85,0.000000)};
\addlegendentry{$\mathrm{lr}=5e-5$}
\addplot[color=blue!70!black, mark=*, mark size=1.8pt, line width=1.4pt] coordinates {(0,0.998393) (5,0.998075) (10,0.985165) (15,0.893631) (20,0.655836) (25,0.425522) (30,0.141612) (35,0.004098) (40,0.000000) (45,0.000000) (50,0.000000) (55,0.000000) (60,0.000000) (65,0.000000)};
\addlegendentry{$\mathrm{lr}=1e-4$}
\addplot[color=blue!35!black, mark=*, mark size=1.8pt, line width=1.4pt] coordinates {(0,0.998393) (5,0.985700) (10,0.741784) (15,0.328811) (20,0.000000) (25,0.000000) (30,0.000000) (35,0.000000) (40,0.000000) (45,0.000000) (50,0.000000) (55,0.000000)};
\addlegendentry{$\mathrm{lr}=2e-4$}
\end{axis}
\end{tikzpicture}

%% file: plots/scaling_unlearning.pgf
\begin{tikzpicture}[x=1cm,y=1cm]
\definecolor{ERougeOrange}{HTML}{D55E00}
\definecolor{PRougeBlue}{HTML}{0072B2}
\definecolor{MUCream}{HTML}{C8B66A}
\definecolor{GridGrey}{HTML}{D5D7DC}
\definecolor{AxisDark}{HTML}{252A31}
\path[use as bounding box] (0,-1.95) rectangle (14.05,7.15);
\draw[GridGrey,line width=0.45pt] (1,1) -- (7.1,1);
\node[anchor=east,font=\large] at (0.88,1) {0.2};
\draw[GridGrey,line width=0.45pt] (1,2.5642) -- (7.1,2.5642);
\node[anchor=east,font=\large] at (0.88,2.5642) {0.3};
\draw[GridGrey,line width=0.45pt] (1,3.6739) -- (7.1,3.6739);
\node[anchor=east,font=\large] at (0.88,3.6739) {0.4};
\draw[GridGrey,line width=0.45pt] (1,5.2381) -- (7.1,5.2381);
\node[anchor=east,font=\large] at (0.88,5.2381) {0.6};
\draw[GridGrey,line width=0.35pt,opacity=0.65] (1.2258,1) -- (1.2258,6.15);
\draw[GridGrey,line width=0.35pt,opacity=0.65] (3.5804,1) -- (3.5804,6.15);
\draw[GridGrey,line width=0.35pt,opacity=0.65] (5.0659,1) -- (5.0659,6.15);
\draw[GridGrey,line width=0.35pt,opacity=0.65] (6.8819,1) -- (6.8819,6.15);
\draw[ERougeOrange,line width=1.45pt] (1.1383,3.4833) -- (1.1627,3.4832) -- (1.1871,3.483) -- (1.2115,3.4828) -- (1.236,3.4827) -- (1.2604,3.4825) -- (1.2848,3.4824) -- (1.3092,3.4822) -- (1.3336,3.482) -- (1.358,3.4819) -- (1.3824,3.4817) -- (1.4068,3.4816) -- (1.4312,3.4814) -- (1.4556,3.4813) -- (1.48,3.4811) -- (1.5044,3.4809) -- (1.5288,3.4808) -- (1.5533,3.4806) -- (1.5777,3.4805) -- (1.6021,3.4803) -- (1.6265,3.4801) -- (1.6509,3.48) -- (1.6753,3.4798) -- (1.6997,3.4797) -- (1.7241,3.4795) -- (1.7485,3.4793) -- (1.7729,3.4792) -- (1.7973,3.479) -- (1.8217,3.4789) -- (1.8461,3.4787) -- (1.8706,3.4786) -- (1.895,3.4784) -- (1.9194,3.4782) -- (1.9438,3.4781) -- (1.9682,3.4779) -- (1.9926,3.4778) -- (2.017,3.4776) -- (2.0414,3.4774) -- (2.0658,3.4773) -- (2.0902,3.4771) -- (2.1146,3.477) -- (2.139,3.4768) -- (2.1634,3.4766) -- (2.1879,3.4765) -- (2.2123,3.4763) -- (2.2367,3.4762) -- (2.2611,3.476) -- (2.2855,3.4759) -- (2.3099,3.4757) -- (2.3343,3.4755) -- (2.3587,3.4754) -- (2.3831,3.4752) -- (2.4075,3.4751) -- (2.4319,3.4749) -- (2.4563,3.4747) -- (2.4808,3.4746) -- (2.5052,3.4744) -- (2.5296,3.4743) -- (2.554,3.4741) -- (2.5784,3.4739) -- (2.6028,3.4738) -- (2.6272,3.4736) -- (2.6516,3.4735) -- (2.676,3.4733) -- (2.7004,3.4732) -- (2.7248,3.473) -- (2.7492,3.4728) -- (2.7736,3.4727) -- (2.7981,3.4725) -- (2.8225,3.4724) -- (2.8469,3.4722) -- (2.8713,3.472) -- (2.8957,3.4719) -- (2.9201,3.4717) -- (2.9445,3.4716) -- (2.9689,3.4714) -- (2.9933,3.4712) -- (3.0177,3.4711) -- (3.0421,3.4709) -- (3.0665,3.4708) -- (3.0909,3.4706) -- (3.1154,3.4705) -- (3.1398,3.4703) -- (3.1642,3.4701) -- (3.1886,3.47) -- (3.213,3.4698) -- (3.2374,3.4697) -- (3.2618,3.4695) -- (3.2862,3.4693) -- (3.3106,3.4692) -- (3.335,3.469) -- (3.3594,3.4689) -- (3.3838,3.4687) -- (3.4082,3.4686) -- (3.4327,3.4684) -- (3.4571,3.4682) -- (3.4815,3.4681) -- (3.5059,3.4679) -- (3.5303,3.4678) -- (3.5547,3.4676) -- (3.5791,3.4674) -- (3.6035,3.4673) -- (3.6279,3.4671) -- (3.6523,3.467) -- (3.6767,3.4668) -- (3.7011,3.4666) -- (3.7256,3.4665) -- (3.75,3.4663) -- (3.7744,3.4662) -- (3.7988,3.466) -- (3.8232,3.4659) -- (3.8476,3.4657) -- (3.872,3.4655) -- (3.8964,3.4654) -- (3.9208,3.4652) -- (3.9452,3.4651) -- (3.9696,3.4649) -- (3.994,3.4647) -- (4.0184,3.4646) -- (4.0429,3.4644) -- (4.0673,3.4643) -- (4.0917,3.4641) -- (4.1161,3.4639) -- (4.1405,3.4638) -- (4.1649,3.4636) -- (4.1893,3.4635) -- (4.2137,3.4633) -- (4.2381,3.4632) -- (4.2625,3.463) -- (4.2869,3.4628) -- (4.3113,3.4627) -- (4.3357,3.4625) -- (4.3602,3.4624) -- (4.3846,3.4622) -- (4.409,3.462) -- (4.4334,3.4619) -- (4.4578,3.4617) -- (4.4822,3.4616) -- (4.5066,3.4614) -- (4.531,3.4612) -- (4.5554,3.4611) -- (4.5798,3.4609) -- (4.6042,3.4608) -- (4.6286,3.4606) -- (4.6531,3.4605) -- (4.6775,3.4603) -- (4.7019,3.4601) -- (4.7263,3.46) -- (4.7507,3.4598) -- (4.7751,3.4597) -- (4.7995,3.4595) -- (4.8239,3.4593) -- (4.8483,3.4592) -- (4.8727,3.459) -- (4.8971,3.4589) -- (4.9215,3.4587) -- (4.9459,3.4585) -- (4.9704,3.4584) -- (4.9948,3.4582) -- (5.0192,3.4581) -- (5.0436,3.4579) -- (5.068,3.4578) -- (5.0924,3.4576) -- (5.1168,3.4574) -- (5.1412,3.4573) -- (5.1656,3.4571) -- (5.19,3.457) -- (5.2144,3.4568) -- (5.2388,3.4566) -- (5.2632,3.4565) -- (5.2877,3.4563) -- (5.3121,3.4562) -- (5.3365,3.456) -- (5.3609,3.4558) -- (5.3853,3.4557) -- (5.4097,3.4555) -- (5.4341,3.4554) -- (5.4585,3.4552) -- (5.4829,3.4551) -- (5.5073,3.4549) -- (5.5317,3.4547) -- (5.5561,3.4546) -- (5.5805,3.4544) -- (5.605,3.4543) -- (5.6294,3.4541) -- (5.6538,3.4539) -- (5.6782,3.4538) -- (5.7026,3.4536) -- (5.727,3.4535) -- (5.7514,3.4533) -- (5.7758,3.4531) -- (5.8002,3.453) -- (5.8246,3.4528) -- (5.849,3.4527) -- (5.8734,3.4525) -- (5.8979,3.4524) -- (5.9223,3.4522) -- (5.9467,3.452) -- (5.9711,3.4519) -- (5.9955,3.4517) -- (6.0199,3.4516) -- (6.0443,3.4514) -- (6.0687,3.4512) -- (6.0931,3.4511) -- (6.1175,3.4509) -- (6.1419,3.4508) -- (6.1663,3.4506) -- (6.1907,3.4504) -- (6.2152,3.4503) -- (6.2396,3.4501) -- (6.264,3.45) -- (6.2884,3.4498) -- (6.3128,3.4497) -- (6.3372,3.4495) -- (6.3616,3.4493) -- (6.386,3.4492) -- (6.4104,3.449) -- (6.4348,3.4489) -- (6.4592,3.4487) -- (6.4836,3.4485) -- (6.508,3.4484) -- (6.5325,3.4482) -- (6.5569,3.4481) -- (6.5813,3.4479) -- (6.6057,3.4477) -- (6.6301,3.4476) -- (6.6545,3.4474) -- (6.6789,3.4473) -- (6.7033,3.4471) -- (6.7277,3.447) -- (6.7521,3.4468) -- (6.7765,3.4466) -- (6.8009,3.4465) -- (6.8253,3.4463) -- (6.8498,3.4462) -- (6.8742,3.446) -- (6.8986,3.4458) -- (6.923,3.4457) -- (6.9474,3.4455) -- (6.9718,3.4454);
\filldraw[fill=ERougeOrange,draw=ERougeOrange,line width=0.42pt] (1.2258,3.4667) circle[radius=0.068];
\filldraw[fill=ERougeOrange,draw=ERougeOrange,line width=0.42pt] (3.5804,3.4832) circle[radius=0.068];
\filldraw[fill=ERougeOrange,draw=ERougeOrange,line width=0.42pt] (5.0659,3.479) circle[radius=0.068];
\filldraw[fill=ERougeOrange,draw=ERougeOrange,line width=0.42pt] (6.8819,3.4249) circle[radius=0.068];
\draw[PRougeBlue,line width=1.45pt] (1.1383,2.9471) -- (1.1627,2.9479) -- (1.1871,2.9488) -- (1.2115,2.9496) -- (1.236,2.9504) -- (1.2604,2.9513) -- (1.2848,2.9521) -- (1.3092,2.9529) -- (1.3336,2.9538) -- (1.358,2.9546) -- (1.3824,2.9554) -- (1.4068,2.9563) -- (1.4312,2.9571) -- (1.4556,2.9579) -- (1.48,2.9588) -- (1.5044,2.9596) -- (1.5288,2.9604) -- (1.5533,2.9613) -- (1.5777,2.9621) -- (1.6021,2.9629) -- (1.6265,2.9638) -- (1.6509,2.9646) -- (1.6753,2.9654) -- (1.6997,2.9663) -- (1.7241,2.9671) -- (1.7485,2.9679) -- (1.7729,2.9688) -- (1.7973,2.9696) -- (1.8217,2.9705) -- (1.8461,2.9713) -- (1.8706,2.9721) -- (1.895,2.973) -- (1.9194,2.9738) -- (1.9438,2.9746) -- (1.9682,2.9755) -- (1.9926,2.9763) -- (2.017,2.9771) -- (2.0414,2.978) -- (2.0658,2.9788) -- (2.0902,2.9796) -- (2.1146,2.9805) -- (2.139,2.9813) -- (2.1634,2.9821) -- (2.1879,2.983) -- (2.2123,2.9838) -- (2.2367,2.9846) -- (2.2611,2.9855) -- (2.2855,2.9863) -- (2.3099,2.9871) -- (2.3343,2.988) -- (2.3587,2.9888) -- (2.3831,2.9896) -- (2.4075,2.9905) -- (2.4319,2.9913) -- (2.4563,2.9921) -- (2.4808,2.993) -- (2.5052,2.9938) -- (2.5296,2.9946) -- (2.554,2.9955) -- (2.5784,2.9963) -- (2.6028,2.9971) -- (2.6272,2.998) -- (2.6516,2.9988) -- (2.676,2.9996) -- (2.7004,3.0005) -- (2.7248,3.0013) -- (2.7492,3.0021) -- (2.7736,3.003) -- (2.7981,3.0038) -- (2.8225,3.0046) -- (2.8469,3.0055) -- (2.8713,3.0063) -- (2.8957,3.0071) -- (2.9201,3.008) -- (2.9445,3.0088) -- (2.9689,3.0096) -- (2.9933,3.0105) -- (3.0177,3.0113) -- (3.0421,3.0121) -- (3.0665,3.013) -- (3.0909,3.0138) -- (3.1154,3.0146) -- (3.1398,3.0155) -- (3.1642,3.0163) -- (3.1886,3.0171) -- (3.213,3.018) -- (3.2374,3.0188) -- (3.2618,3.0196) -- (3.2862,3.0205) -- (3.3106,3.0213) -- (3.335,3.0221) -- (3.3594,3.023) -- (3.3838,3.0238) -- (3.4082,3.0246) -- (3.4327,3.0255) -- (3.4571,3.0263) -- (3.4815,3.0271) -- (3.5059,3.028) -- (3.5303,3.0288) -- (3.5547,3.0296) -- (3.5791,3.0305) -- (3.6035,3.0313) -- (3.6279,3.0321) -- (3.6523,3.033) -- (3.6767,3.0338) -- (3.7011,3.0346) -- (3.7256,3.0355) -- (3.75,3.0363) -- (3.7744,3.0371) -- (3.7988,3.038) -- (3.8232,3.0388) -- (3.8476,3.0396) -- (3.872,3.0405) -- (3.8964,3.0413) -- (3.9208,3.0421) -- (3.9452,3.043) -- (3.9696,3.0438) -- (3.994,3.0446) -- (4.0184,3.0455) -- (4.0429,3.0463) -- (4.0673,3.0471) -- (4.0917,3.048) -- (4.1161,3.0488) -- (4.1405,3.0496) -- (4.1649,3.0505) -- (4.1893,3.0513) -- (4.2137,3.0521) -- (4.2381,3.053) -- (4.2625,3.0538) -- (4.2869,3.0546) -- (4.3113,3.0555) -- (4.3357,3.0563) -- (4.3602,3.0571) -- (4.3846,3.058) -- (4.409,3.0588) -- (4.4334,3.0597) -- (4.4578,3.0605) -- (4.4822,3.0613) -- (4.5066,3.0622) -- (4.531,3.063) -- (4.5554,3.0638) -- (4.5798,3.0647) -- (4.6042,3.0655) -- (4.6286,3.0663) -- (4.6531,3.0672) -- (4.6775,3.068) -- (4.7019,3.0688) -- (4.7263,3.0697) -- (4.7507,3.0705) -- (4.7751,3.0713) -- (4.7995,3.0722) -- (4.8239,3.073) -- (4.8483,3.0738) -- (4.8727,3.0747) -- (4.8971,3.0755) -- (4.9215,3.0763) -- (4.9459,3.0772) -- (4.9704,3.078) -- (4.9948,3.0788) -- (5.0192,3.0797) -- (5.0436,3.0805) -- (5.068,3.0813) -- (5.0924,3.0822) -- (5.1168,3.083) -- (5.1412,3.0838) -- (5.1656,3.0847) -- (5.19,3.0855) -- (5.2144,3.0863) -- (5.2388,3.0872) -- (5.2632,3.088) -- (5.2877,3.0888) -- (5.3121,3.0897) -- (5.3365,3.0905) -- (5.3609,3.0913) -- (5.3853,3.0922) -- (5.4097,3.093) -- (5.4341,3.0938) -- (5.4585,3.0947) -- (5.4829,3.0955) -- (5.5073,3.0963) -- (5.5317,3.0972) -- (5.5561,3.098) -- (5.5805,3.0988) -- (5.605,3.0997) -- (5.6294,3.1005) -- (5.6538,3.1013) -- (5.6782,3.1022) -- (5.7026,3.103) -- (5.727,3.1038) -- (5.7514,3.1047) -- (5.7758,3.1055) -- (5.8002,3.1063) -- (5.8246,3.1072) -- (5.849,3.108) -- (5.8734,3.1088) -- (5.8979,3.1097) -- (5.9223,3.1105) -- (5.9467,3.1113) -- (5.9711,3.1122) -- (5.9955,3.113) -- (6.0199,3.1138) -- (6.0443,3.1147) -- (6.0687,3.1155) -- (6.0931,3.1163) -- (6.1175,3.1172) -- (6.1419,3.118) -- (6.1663,3.1188) -- (6.1907,3.1197) -- (6.2152,3.1205) -- (6.2396,3.1213) -- (6.264,3.1222) -- (6.2884,3.123) -- (6.3128,3.1238) -- (6.3372,3.1247) -- (6.3616,3.1255) -- (6.386,3.1263) -- (6.4104,3.1272) -- (6.4348,3.128) -- (6.4592,3.1288) -- (6.4836,3.1297) -- (6.508,3.1305) -- (6.5325,3.1313) -- (6.5569,3.1322) -- (6.5813,3.133) -- (6.6057,3.1338) -- (6.6301,3.1347) -- (6.6545,3.1355) -- (6.6789,3.1363) -- (6.7033,3.1372) -- (6.7277,3.138) -- (6.7521,3.1388) -- (6.7765,3.1397) -- (6.8009,3.1405) -- (6.8253,3.1413) -- (6.8498,3.1422) -- (6.8742,3.143) -- (6.8986,3.1438) -- (6.923,3.1447) -- (6.9474,3.1455) -- (6.9718,3.1463);
\filldraw[fill=PRougeBlue,draw=PRougeBlue,line width=0.42pt] (1.1608,2.8273) rectangle (1.2908,2.9573);
\filldraw[fill=PRougeBlue,draw=PRougeBlue,line width=0.42pt] (3.5154,3.0064) rectangle (3.6454,3.1364);
\filldraw[fill=PRougeBlue,draw=PRougeBlue,line width=0.42pt] (5.0009,3.1219) rectangle (5.1309,3.2519);
\filldraw[fill=PRougeBlue,draw=PRougeBlue,line width=0.42pt] (6.8169,2.9895) rectangle (6.9469,3.1195);
\draw[MUCream,line width=1.45pt] (1.1383,3.8377) -- (1.1627,3.8443) -- (1.1871,3.8508) -- (1.2115,3.8574) -- (1.236,3.8639) -- (1.2604,3.8705) -- (1.2848,3.877) -- (1.3092,3.8836) -- (1.3336,3.8901) -- (1.358,3.8966) -- (1.3824,3.9032) -- (1.4068,3.9097) -- (1.4312,3.9163) -- (1.4556,3.9228) -- (1.48,3.9294) -- (1.5044,3.9359) -- (1.5288,3.9424) -- (1.5533,3.949) -- (1.5777,3.9555) -- (1.6021,3.9621) -- (1.6265,3.9686) -- (1.6509,3.9752) -- (1.6753,3.9817) -- (1.6997,3.9882) -- (1.7241,3.9948) -- (1.7485,4.0013) -- (1.7729,4.0079) -- (1.7973,4.0144) -- (1.8217,4.021) -- (1.8461,4.0275) -- (1.8706,4.0341) -- (1.895,4.0406) -- (1.9194,4.0471) -- (1.9438,4.0537) -- (1.9682,4.0602) -- (1.9926,4.0668) -- (2.017,4.0733) -- (2.0414,4.0799) -- (2.0658,4.0864) -- (2.0902,4.0929) -- (2.1146,4.0995) -- (2.139,4.106) -- (2.1634,4.1126) -- (2.1879,4.1191) -- (2.2123,4.1257) -- (2.2367,4.1322) -- (2.2611,4.1388) -- (2.2855,4.1453) -- (2.3099,4.1518) -- (2.3343,4.1584) -- (2.3587,4.1649) -- (2.3831,4.1715) -- (2.4075,4.178) -- (2.4319,4.1846) -- (2.4563,4.1911) -- (2.4808,4.1976) -- (2.5052,4.2042) -- (2.5296,4.2107) -- (2.554,4.2173) -- (2.5784,4.2238) -- (2.6028,4.2304) -- (2.6272,4.2369) -- (2.6516,4.2435) -- (2.676,4.25) -- (2.7004,4.2565) -- (2.7248,4.2631) -- (2.7492,4.2696) -- (2.7736,4.2762) -- (2.7981,4.2827) -- (2.8225,4.2893) -- (2.8469,4.2958) -- (2.8713,4.3023) -- (2.8957,4.3089) -- (2.9201,4.3154) -- (2.9445,4.322) -- (2.9689,4.3285) -- (2.9933,4.3351) -- (3.0177,4.3416) -- (3.0421,4.3481) -- (3.0665,4.3547) -- (3.0909,4.3612) -- (3.1154,4.3678) -- (3.1398,4.3743) -- (3.1642,4.3809) -- (3.1886,4.3874) -- (3.213,4.394) -- (3.2374,4.4005) -- (3.2618,4.407) -- (3.2862,4.4136) -- (3.3106,4.4201) -- (3.335,4.4267) -- (3.3594,4.4332) -- (3.3838,4.4398) -- (3.4082,4.4463) -- (3.4327,4.4528) -- (3.4571,4.4594) -- (3.4815,4.4659) -- (3.5059,4.4725) -- (3.5303,4.479) -- (3.5547,4.4856) -- (3.5791,4.4921) -- (3.6035,4.4987) -- (3.6279,4.5052) -- (3.6523,4.5117) -- (3.6767,4.5183) -- (3.7011,4.5248) -- (3.7256,4.5314) -- (3.75,4.5379) -- (3.7744,4.5445) -- (3.7988,4.551) -- (3.8232,4.5575) -- (3.8476,4.5641) -- (3.872,4.5706) -- (3.8964,4.5772) -- (3.9208,4.5837) -- (3.9452,4.5903) -- (3.9696,4.5968) -- (3.994,4.6034) -- (4.0184,4.6099) -- (4.0429,4.6164) -- (4.0673,4.623) -- (4.0917,4.6295) -- (4.1161,4.6361) -- (4.1405,4.6426) -- (4.1649,4.6492) -- (4.1893,4.6557) -- (4.2137,4.6622) -- (4.2381,4.6688) -- (4.2625,4.6753) -- (4.2869,4.6819) -- (4.3113,4.6884) -- (4.3357,4.695) -- (4.3602,4.7015) -- (4.3846,4.708) -- (4.409,4.7146) -- (4.4334,4.7211) -- (4.4578,4.7277) -- (4.4822,4.7342) -- (4.5066,4.7408) -- (4.531,4.7473) -- (4.5554,4.7539) -- (4.5798,4.7604) -- (4.6042,4.7669) -- (4.6286,4.7735) -- (4.6531,4.78) -- (4.6775,4.7866) -- (4.7019,4.7931) -- (4.7263,4.7997) -- (4.7507,4.8062) -- (4.7751,4.8127) -- (4.7995,4.8193) -- (4.8239,4.8258) -- (4.8483,4.8324) -- (4.8727,4.8389) -- (4.8971,4.8455) -- (4.9215,4.852) -- (4.9459,4.8586) -- (4.9704,4.8651) -- (4.9948,4.8716) -- (5.0192,4.8782) -- (5.0436,4.8847) -- (5.068,4.8913) -- (5.0924,4.8978) -- (5.1168,4.9044) -- (5.1412,4.9109) -- (5.1656,4.9174) -- (5.19,4.924) -- (5.2144,4.9305) -- (5.2388,4.9371) -- (5.2632,4.9436) -- (5.2877,4.9502) -- (5.3121,4.9567) -- (5.3365,4.9633) -- (5.3609,4.9698) -- (5.3853,4.9763) -- (5.4097,4.9829) -- (5.4341,4.9894) -- (5.4585,4.996) -- (5.4829,5.0025) -- (5.5073,5.0091) -- (5.5317,5.0156) -- (5.5561,5.0221) -- (5.5805,5.0287) -- (5.605,5.0352) -- (5.6294,5.0418) -- (5.6538,5.0483) -- (5.6782,5.0549) -- (5.7026,5.0614) -- (5.727,5.0679) -- (5.7514,5.0745) -- (5.7758,5.081) -- (5.8002,5.0876) -- (5.8246,5.0941) -- (5.849,5.1007) -- (5.8734,5.1072) -- (5.8979,5.1138) -- (5.9223,5.1203) -- (5.9467,5.1268) -- (5.9711,5.1334) -- (5.9955,5.1399) -- (6.0199,5.1465) -- (6.0443,5.153) -- (6.0687,5.1596) -- (6.0931,5.1661) -- (6.1175,5.1726) -- (6.1419,5.1792) -- (6.1663,5.1857) -- (6.1907,5.1923) -- (6.2152,5.1988) -- (6.2396,5.2054) -- (6.264,5.2119) -- (6.2884,5.2185) -- (6.3128,5.225) -- (6.3372,5.2315) -- (6.3616,5.2381) -- (6.386,5.2446) -- (6.4104,5.2512) -- (6.4348,5.2577) -- (6.4592,5.2643) -- (6.4836,5.2708) -- (6.508,5.2773) -- (6.5325,5.2839) -- (6.5569,5.2904) -- (6.5813,5.297) -- (6.6057,5.3035) -- (6.6301,5.3101) -- (6.6545,5.3166) -- (6.6789,5.3232) -- (6.7033,5.3297) -- (6.7277,5.3362) -- (6.7521,5.3428) -- (6.7765,5.3493) -- (6.8009,5.3559) -- (6.8253,5.3624) -- (6.8498,5.369) -- (6.8742,5.3755) -- (6.8986,5.382) -- (6.923,5.3886) -- (6.9474,5.3951) -- (6.9718,5.4017);
\filldraw[fill=MUCream,draw=MUCream,line width=0.42pt] (1.2258,3.9977) -- (1.3078,3.8457) -- (1.1438,3.8457) -- cycle;
\filldraw[fill=MUCream,draw=MUCream,line width=0.42pt] (3.5804,4.4721) -- (3.6624,4.3201) -- (3.4984,4.3201) -- cycle;
\filldraw[fill=MUCream,draw=MUCream,line width=0.42pt] (5.0659,4.9892) -- (5.1479,4.8372) -- (4.9839,4.8372) -- cycle;
\filldraw[fill=MUCream,draw=MUCream,line width=0.42pt] (6.8819,5.491) -- (6.9639,5.339) -- (6.7999,5.339) -- cycle;
\draw[AxisDark,line width=1.55pt] (1,1) -- (7.1,1);
\draw[AxisDark,line width=1.55pt] (1,1) -- (1,6.15);
\draw[AxisDark,line width=1.25pt] (1.2258,1) -- (1.2258,0.89);
\node[anchor=north,font=\large] at (1.2258,0.82) {0.5B};
\draw[AxisDark,line width=1.25pt] (3.5804,1) -- (3.5804,0.89);
\node[anchor=north,font=\large] at (3.5804,0.82) {1.5B};
\draw[AxisDark,line width=1.25pt] (5.0659,1) -- (5.0659,0.89);
\node[anchor=north,font=\large] at (5.0659,0.82) {3B};
\draw[AxisDark,line width=1.25pt] (6.8819,1) -- (6.8819,0.89);
\node[anchor=north,font=\large] at (6.8819,0.82) {7B};
\draw[AxisDark,line width=1.25pt] (1,1) -- (0.89,1);
\draw[AxisDark,line width=1.25pt] (1,2.5642) -- (0.89,2.5642);
\draw[AxisDark,line width=1.25pt] (1,3.6739) -- (0.89,3.6739);
\draw[AxisDark,line width=1.25pt] (1,5.2381) -- (0.89,5.2381);
\node[anchor=north,font=\large] at (4.05,0.38) {Model size};
\node[anchor=south,font=\Large] at (4.05,6.4) {SimNPO};
\draw[GridGrey,line width=0.45pt] (7.82,1) -- (13.92,1);
\draw[GridGrey,line width=0.45pt] (7.82,2.5642) -- (13.92,2.5642);
\draw[GridGrey,line width=0.45pt] (7.82,3.6739) -- (13.92,3.6739);
\draw[GridGrey,line width=0.45pt] (7.82,5.2381) -- (13.92,5.2381);
\draw[GridGrey,line width=0.35pt,opacity=0.65] (8.0458,1) -- (8.0458,6.15);
\draw[GridGrey,line width=0.35pt,opacity=0.65] (10.4004,1) -- (10.4004,6.15);
\draw[GridGrey,line width=0.35pt,opacity=0.65] (11.8859,1) -- (11.8859,6.15);
\draw[GridGrey,line width=0.35pt,opacity=0.65] (13.7019,1) -- (13.7019,6.15);
\draw[ERougeOrange,line width=1.45pt] (7.9583,3.5552) -- (7.9827,3.556) -- (8.0071,3.5568) -- (8.0315,3.5576) -- (8.056,3.5584) -- (8.0804,3.5592) -- (8.1048,3.56) -- (8.1292,3.5608) -- (8.1536,3.5616) -- (8.178,3.5624) -- (8.2024,3.5632) -- (8.2268,3.564) -- (8.2512,3.5648) -- (8.2756,3.5656) -- (8.3,3.5664) -- (8.3244,3.5672) -- (8.3488,3.5681) -- (8.3733,3.5689) -- (8.3977,3.5697) -- (8.4221,3.5705) -- (8.4465,3.5713) -- (8.4709,3.5721) -- (8.4953,3.5729) -- (8.5197,3.5737) -- (8.5441,3.5745) -- (8.5685,3.5753) -- (8.5929,3.5761) -- (8.6173,3.5769) -- (8.6417,3.5777) -- (8.6661,3.5785) -- (8.6906,3.5793) -- (8.715,3.5801) -- (8.7394,3.5809) -- (8.7638,3.5817) -- (8.7882,3.5825) -- (8.8126,3.5833) -- (8.837,3.5841) -- (8.8614,3.5849) -- (8.8858,3.5858) -- (8.9102,3.5866) -- (8.9346,3.5874) -- (8.959,3.5882) -- (8.9834,3.589) -- (9.0079,3.5898) -- (9.0323,3.5906) -- (9.0567,3.5914) -- (9.0811,3.5922) -- (9.1055,3.593) -- (9.1299,3.5938) -- (9.1543,3.5946) -- (9.1787,3.5954) -- (9.2031,3.5962) -- (9.2275,3.597) -- (9.2519,3.5978) -- (9.2763,3.5986) -- (9.3008,3.5994) -- (9.3252,3.6002) -- (9.3496,3.601) -- (9.374,3.6018) -- (9.3984,3.6026) -- (9.4228,3.6035) -- (9.4472,3.6043) -- (9.4716,3.6051) -- (9.496,3.6059) -- (9.5204,3.6067) -- (9.5448,3.6075) -- (9.5692,3.6083) -- (9.5936,3.6091) -- (9.6181,3.6099) -- (9.6425,3.6107) -- (9.6669,3.6115) -- (9.6913,3.6123) -- (9.7157,3.6131) -- (9.7401,3.6139) -- (9.7645,3.6147) -- (9.7889,3.6155) -- (9.8133,3.6163) -- (9.8377,3.6171) -- (9.8621,3.6179) -- (9.8865,3.6187) -- (9.9109,3.6195) -- (9.9354,3.6203) -- (9.9598,3.6212) -- (9.9842,3.622) -- (10.0086,3.6228) -- (10.033,3.6236) -- (10.0574,3.6244) -- (10.0818,3.6252) -- (10.1062,3.626) -- (10.1306,3.6268) -- (10.155,3.6276) -- (10.1794,3.6284) -- (10.2038,3.6292) -- (10.2282,3.63) -- (10.2527,3.6308) -- (10.2771,3.6316) -- (10.3015,3.6324) -- (10.3259,3.6332) -- (10.3503,3.634) -- (10.3747,3.6348) -- (10.3991,3.6356) -- (10.4235,3.6364) -- (10.4479,3.6372) -- (10.4723,3.638) -- (10.4967,3.6389) -- (10.5211,3.6397) -- (10.5456,3.6405) -- (10.57,3.6413) -- (10.5944,3.6421) -- (10.6188,3.6429) -- (10.6432,3.6437) -- (10.6676,3.6445) -- (10.692,3.6453) -- (10.7164,3.6461) -- (10.7408,3.6469) -- (10.7652,3.6477) -- (10.7896,3.6485) -- (10.814,3.6493) -- (10.8384,3.6501) -- (10.8629,3.6509) -- (10.8873,3.6517) -- (10.9117,3.6525) -- (10.9361,3.6533) -- (10.9605,3.6541) -- (10.9849,3.6549) -- (11.0093,3.6557) -- (11.0337,3.6566) -- (11.0581,3.6574) -- (11.0825,3.6582) -- (11.1069,3.659) -- (11.1313,3.6598) -- (11.1557,3.6606) -- (11.1802,3.6614) -- (11.2046,3.6622) -- (11.229,3.663) -- (11.2534,3.6638) -- (11.2778,3.6646) -- (11.3022,3.6654) -- (11.3266,3.6662) -- (11.351,3.667) -- (11.3754,3.6678) -- (11.3998,3.6686) -- (11.4242,3.6694) -- (11.4486,3.6702) -- (11.4731,3.671) -- (11.4975,3.6718) -- (11.5219,3.6726) -- (11.5463,3.6734) -- (11.5707,3.6743) -- (11.5951,3.6751) -- (11.6195,3.6759) -- (11.6439,3.6767) -- (11.6683,3.6775) -- (11.6927,3.6783) -- (11.7171,3.6791) -- (11.7415,3.6799) -- (11.7659,3.6807) -- (11.7904,3.6815) -- (11.8148,3.6823) -- (11.8392,3.6831) -- (11.8636,3.6839) -- (11.888,3.6847) -- (11.9124,3.6855) -- (11.9368,3.6863) -- (11.9612,3.6871) -- (11.9856,3.6879) -- (12.01,3.6887) -- (12.0344,3.6895) -- (12.0588,3.6903) -- (12.0832,3.6911) -- (12.1077,3.692) -- (12.1321,3.6928) -- (12.1565,3.6936) -- (12.1809,3.6944) -- (12.2053,3.6952) -- (12.2297,3.696) -- (12.2541,3.6968) -- (12.2785,3.6976) -- (12.3029,3.6984) -- (12.3273,3.6992) -- (12.3517,3.7) -- (12.3761,3.7008) -- (12.4005,3.7016) -- (12.425,3.7024) -- (12.4494,3.7032) -- (12.4738,3.704) -- (12.4982,3.7048) -- (12.5226,3.7056) -- (12.547,3.7064) -- (12.5714,3.7072) -- (12.5958,3.708) -- (12.6202,3.7088) -- (12.6446,3.7096) -- (12.669,3.7105) -- (12.6934,3.7113) -- (12.7179,3.7121) -- (12.7423,3.7129) -- (12.7667,3.7137) -- (12.7911,3.7145) -- (12.8155,3.7153) -- (12.8399,3.7161) -- (12.8643,3.7169) -- (12.8887,3.7177) -- (12.9131,3.7185) -- (12.9375,3.7193) -- (12.9619,3.7201) -- (12.9863,3.7209) -- (13.0107,3.7217) -- (13.0352,3.7225) -- (13.0596,3.7233) -- (13.084,3.7241) -- (13.1084,3.7249) -- (13.1328,3.7257) -- (13.1572,3.7265) -- (13.1816,3.7273) -- (13.206,3.7282) -- (13.2304,3.729) -- (13.2548,3.7298) -- (13.2792,3.7306) -- (13.3036,3.7314) -- (13.328,3.7322) -- (13.3525,3.733) -- (13.3769,3.7338) -- (13.4013,3.7346) -- (13.4257,3.7354) -- (13.4501,3.7362) -- (13.4745,3.737) -- (13.4989,3.7378) -- (13.5233,3.7386) -- (13.5477,3.7394) -- (13.5721,3.7402) -- (13.5965,3.741) -- (13.6209,3.7418) -- (13.6453,3.7426) -- (13.6698,3.7434) -- (13.6942,3.7442) -- (13.7186,3.745) -- (13.743,3.7459) -- (13.7674,3.7467) -- (13.7918,3.7475);
\filldraw[fill=ERougeOrange,draw=ERougeOrange,line width=0.42pt] (8.0458,3.5229) circle[radius=0.068];
\filldraw[fill=ERougeOrange,draw=ERougeOrange,line width=0.42pt] (10.4004,3.6226) circle[radius=0.068];
\filldraw[fill=ERougeOrange,draw=ERougeOrange,line width=0.42pt] (11.8859,3.818) circle[radius=0.068];
\filldraw[fill=ERougeOrange,draw=ERougeOrange,line width=0.42pt] (13.7019,3.6594) circle[radius=0.068];
\draw[PRougeBlue,line width=1.45pt] (7.9583,2.8628) -- (7.9827,2.861) -- (8.0071,2.8592) -- (8.0315,2.8574) -- (8.056,2.8556) -- (8.0804,2.8538) -- (8.1048,2.852) -- (8.1292,2.8502) -- (8.1536,2.8484) -- (8.178,2.8466) -- (8.2024,2.8448) -- (8.2268,2.843) -- (8.2512,2.8412) -- (8.2756,2.8394) -- (8.3,2.8375) -- (8.3244,2.8357) -- (8.3488,2.8339) -- (8.3733,2.8321) -- (8.3977,2.8303) -- (8.4221,2.8285) -- (8.4465,2.8267) -- (8.4709,2.8249) -- (8.4953,2.8231) -- (8.5197,2.8213) -- (8.5441,2.8195) -- (8.5685,2.8177) -- (8.5929,2.8159) -- (8.6173,2.814) -- (8.6417,2.8122) -- (8.6661,2.8104) -- (8.6906,2.8086) -- (8.715,2.8068) -- (8.7394,2.805) -- (8.7638,2.8032) -- (8.7882,2.8014) -- (8.8126,2.7996) -- (8.837,2.7978) -- (8.8614,2.796) -- (8.8858,2.7942) -- (8.9102,2.7924) -- (8.9346,2.7906) -- (8.959,2.7887) -- (8.9834,2.7869) -- (9.0079,2.7851) -- (9.0323,2.7833) -- (9.0567,2.7815) -- (9.0811,2.7797) -- (9.1055,2.7779) -- (9.1299,2.7761) -- (9.1543,2.7743) -- (9.1787,2.7725) -- (9.2031,2.7707) -- (9.2275,2.7689) -- (9.2519,2.7671) -- (9.2763,2.7652) -- (9.3008,2.7634) -- (9.3252,2.7616) -- (9.3496,2.7598) -- (9.374,2.758) -- (9.3984,2.7562) -- (9.4228,2.7544) -- (9.4472,2.7526) -- (9.4716,2.7508) -- (9.496,2.749) -- (9.5204,2.7472) -- (9.5448,2.7454) -- (9.5692,2.7436) -- (9.5936,2.7417) -- (9.6181,2.7399) -- (9.6425,2.7381) -- (9.6669,2.7363) -- (9.6913,2.7345) -- (9.7157,2.7327) -- (9.7401,2.7309) -- (9.7645,2.7291) -- (9.7889,2.7273) -- (9.8133,2.7255) -- (9.8377,2.7237) -- (9.8621,2.7219) -- (9.8865,2.7201) -- (9.9109,2.7183) -- (9.9354,2.7164) -- (9.9598,2.7146) -- (9.9842,2.7128) -- (10.0086,2.711) -- (10.033,2.7092) -- (10.0574,2.7074) -- (10.0818,2.7056) -- (10.1062,2.7038) -- (10.1306,2.702) -- (10.155,2.7002) -- (10.1794,2.6984) -- (10.2038,2.6966) -- (10.2282,2.6948) -- (10.2527,2.6929) -- (10.2771,2.6911) -- (10.3015,2.6893) -- (10.3259,2.6875) -- (10.3503,2.6857) -- (10.3747,2.6839) -- (10.3991,2.6821) -- (10.4235,2.6803) -- (10.4479,2.6785) -- (10.4723,2.6767) -- (10.4967,2.6749) -- (10.5211,2.6731) -- (10.5456,2.6713) -- (10.57,2.6694) -- (10.5944,2.6676) -- (10.6188,2.6658) -- (10.6432,2.664) -- (10.6676,2.6622) -- (10.692,2.6604) -- (10.7164,2.6586) -- (10.7408,2.6568) -- (10.7652,2.655) -- (10.7896,2.6532) -- (10.814,2.6514) -- (10.8384,2.6496) -- (10.8629,2.6478) -- (10.8873,2.646) -- (10.9117,2.6441) -- (10.9361,2.6423) -- (10.9605,2.6405) -- (10.9849,2.6387) -- (11.0093,2.6369) -- (11.0337,2.6351) -- (11.0581,2.6333) -- (11.0825,2.6315) -- (11.1069,2.6297) -- (11.1313,2.6279) -- (11.1557,2.6261) -- (11.1802,2.6243) -- (11.2046,2.6225) -- (11.229,2.6206) -- (11.2534,2.6188) -- (11.2778,2.617) -- (11.3022,2.6152) -- (11.3266,2.6134) -- (11.351,2.6116) -- (11.3754,2.6098) -- (11.3998,2.608) -- (11.4242,2.6062) -- (11.4486,2.6044) -- (11.4731,2.6026) -- (11.4975,2.6008) -- (11.5219,2.599) -- (11.5463,2.5971) -- (11.5707,2.5953) -- (11.5951,2.5935) -- (11.6195,2.5917) -- (11.6439,2.5899) -- (11.6683,2.5881) -- (11.6927,2.5863) -- (11.7171,2.5845) -- (11.7415,2.5827) -- (11.7659,2.5809) -- (11.7904,2.5791) -- (11.8148,2.5773) -- (11.8392,2.5755) -- (11.8636,2.5737) -- (11.888,2.5718) -- (11.9124,2.57) -- (11.9368,2.5682) -- (11.9612,2.5664) -- (11.9856,2.5646) -- (12.01,2.5628) -- (12.0344,2.561) -- (12.0588,2.5592) -- (12.0832,2.5574) -- (12.1077,2.5556) -- (12.1321,2.5538) -- (12.1565,2.552) -- (12.1809,2.5502) -- (12.2053,2.5483) -- (12.2297,2.5465) -- (12.2541,2.5447) -- (12.2785,2.5429) -- (12.3029,2.5411) -- (12.3273,2.5393) -- (12.3517,2.5375) -- (12.3761,2.5357) -- (12.4005,2.5339) -- (12.425,2.5321) -- (12.4494,2.5303) -- (12.4738,2.5285) -- (12.4982,2.5267) -- (12.5226,2.5249) -- (12.547,2.523) -- (12.5714,2.5212) -- (12.5958,2.5194) -- (12.6202,2.5176) -- (12.6446,2.5158) -- (12.669,2.514) -- (12.6934,2.5122) -- (12.7179,2.5104) -- (12.7423,2.5086) -- (12.7667,2.5068) -- (12.7911,2.505) -- (12.8155,2.5032) -- (12.8399,2.5014) -- (12.8643,2.4995) -- (12.8887,2.4977) -- (12.9131,2.4959) -- (12.9375,2.4941) -- (12.9619,2.4923) -- (12.9863,2.4905) -- (13.0107,2.4887) -- (13.0352,2.4869) -- (13.0596,2.4851) -- (13.084,2.4833) -- (13.1084,2.4815) -- (13.1328,2.4797) -- (13.1572,2.4779) -- (13.1816,2.476) -- (13.206,2.4742) -- (13.2304,2.4724) -- (13.2548,2.4706) -- (13.2792,2.4688) -- (13.3036,2.467) -- (13.328,2.4652) -- (13.3525,2.4634) -- (13.3769,2.4616) -- (13.4013,2.4598) -- (13.4257,2.458) -- (13.4501,2.4562) -- (13.4745,2.4544) -- (13.4989,2.4526) -- (13.5233,2.4507) -- (13.5477,2.4489) -- (13.5721,2.4471) -- (13.5965,2.4453) -- (13.6209,2.4435) -- (13.6453,2.4417) -- (13.6698,2.4399) -- (13.6942,2.4381) -- (13.7186,2.4363) -- (13.743,2.4345) -- (13.7674,2.4327) -- (13.7918,2.4309);
\filldraw[fill=PRougeBlue,draw=PRougeBlue,line width=0.42pt] (7.9808,2.6783) rectangle (8.1108,2.8083);
\filldraw[fill=PRougeBlue,draw=PRougeBlue,line width=0.42pt] (10.3354,2.632) rectangle (10.4654,2.762);
\filldraw[fill=PRougeBlue,draw=PRougeBlue,line width=0.42pt] (11.8209,2.8319) rectangle (11.9509,2.9619);
\filldraw[fill=PRougeBlue,draw=PRougeBlue,line width=0.42pt] (13.6369,2.1457) rectangle (13.7669,2.2757);
\draw[MUCream,line width=1.45pt] (7.9583,3.6414) -- (7.9827,3.6483) -- (8.0071,3.6552) -- (8.0315,3.6621) -- (8.056,3.669) -- (8.0804,3.6759) -- (8.1048,3.6828) -- (8.1292,3.6897) -- (8.1536,3.6966) -- (8.178,3.7035) -- (8.2024,3.7104) -- (8.2268,3.7173) -- (8.2512,3.7242) -- (8.2756,3.7311) -- (8.3,3.738) -- (8.3244,3.745) -- (8.3488,3.7519) -- (8.3733,3.7588) -- (8.3977,3.7657) -- (8.4221,3.7726) -- (8.4465,3.7795) -- (8.4709,3.7864) -- (8.4953,3.7933) -- (8.5197,3.8002) -- (8.5441,3.8071) -- (8.5685,3.814) -- (8.5929,3.8209) -- (8.6173,3.8278) -- (8.6417,3.8347) -- (8.6661,3.8416) -- (8.6906,3.8486) -- (8.715,3.8555) -- (8.7394,3.8624) -- (8.7638,3.8693) -- (8.7882,3.8762) -- (8.8126,3.8831) -- (8.837,3.89) -- (8.8614,3.8969) -- (8.8858,3.9038) -- (8.9102,3.9107) -- (8.9346,3.9176) -- (8.959,3.9245) -- (8.9834,3.9314) -- (9.0079,3.9383) -- (9.0323,3.9452) -- (9.0567,3.9521) -- (9.0811,3.9591) -- (9.1055,3.966) -- (9.1299,3.9729) -- (9.1543,3.9798) -- (9.1787,3.9867) -- (9.2031,3.9936) -- (9.2275,4.0005) -- (9.2519,4.0074) -- (9.2763,4.0143) -- (9.3008,4.0212) -- (9.3252,4.0281) -- (9.3496,4.035) -- (9.374,4.0419) -- (9.3984,4.0488) -- (9.4228,4.0557) -- (9.4472,4.0626) -- (9.4716,4.0696) -- (9.496,4.0765) -- (9.5204,4.0834) -- (9.5448,4.0903) -- (9.5692,4.0972) -- (9.5936,4.1041) -- (9.6181,4.111) -- (9.6425,4.1179) -- (9.6669,4.1248) -- (9.6913,4.1317) -- (9.7157,4.1386) -- (9.7401,4.1455) -- (9.7645,4.1524) -- (9.7889,4.1593) -- (9.8133,4.1662) -- (9.8377,4.1731) -- (9.8621,4.1801) -- (9.8865,4.187) -- (9.9109,4.1939) -- (9.9354,4.2008) -- (9.9598,4.2077) -- (9.9842,4.2146) -- (10.0086,4.2215) -- (10.033,4.2284) -- (10.0574,4.2353) -- (10.0818,4.2422) -- (10.1062,4.2491) -- (10.1306,4.256) -- (10.155,4.2629) -- (10.1794,4.2698) -- (10.2038,4.2767) -- (10.2282,4.2837) -- (10.2527,4.2906) -- (10.2771,4.2975) -- (10.3015,4.3044) -- (10.3259,4.3113) -- (10.3503,4.3182) -- (10.3747,4.3251) -- (10.3991,4.332) -- (10.4235,4.3389) -- (10.4479,4.3458) -- (10.4723,4.3527) -- (10.4967,4.3596) -- (10.5211,4.3665) -- (10.5456,4.3734) -- (10.57,4.3803) -- (10.5944,4.3872) -- (10.6188,4.3942) -- (10.6432,4.4011) -- (10.6676,4.408) -- (10.692,4.4149) -- (10.7164,4.4218) -- (10.7408,4.4287) -- (10.7652,4.4356) -- (10.7896,4.4425) -- (10.814,4.4494) -- (10.8384,4.4563) -- (10.8629,4.4632) -- (10.8873,4.4701) -- (10.9117,4.477) -- (10.9361,4.4839) -- (10.9605,4.4908) -- (10.9849,4.4977) -- (11.0093,4.5047) -- (11.0337,4.5116) -- (11.0581,4.5185) -- (11.0825,4.5254) -- (11.1069,4.5323) -- (11.1313,4.5392) -- (11.1557,4.5461) -- (11.1802,4.553) -- (11.2046,4.5599) -- (11.229,4.5668) -- (11.2534,4.5737) -- (11.2778,4.5806) -- (11.3022,4.5875) -- (11.3266,4.5944) -- (11.351,4.6013) -- (11.3754,4.6082) -- (11.3998,4.6152) -- (11.4242,4.6221) -- (11.4486,4.629) -- (11.4731,4.6359) -- (11.4975,4.6428) -- (11.5219,4.6497) -- (11.5463,4.6566) -- (11.5707,4.6635) -- (11.5951,4.6704) -- (11.6195,4.6773) -- (11.6439,4.6842) -- (11.6683,4.6911) -- (11.6927,4.698) -- (11.7171,4.7049) -- (11.7415,4.7118) -- (11.7659,4.7188) -- (11.7904,4.7257) -- (11.8148,4.7326) -- (11.8392,4.7395) -- (11.8636,4.7464) -- (11.888,4.7533) -- (11.9124,4.7602) -- (11.9368,4.7671) -- (11.9612,4.774) -- (11.9856,4.7809) -- (12.01,4.7878) -- (12.0344,4.7947) -- (12.0588,4.8016) -- (12.0832,4.8085) -- (12.1077,4.8154) -- (12.1321,4.8223) -- (12.1565,4.8293) -- (12.1809,4.8362) -- (12.2053,4.8431) -- (12.2297,4.85) -- (12.2541,4.8569) -- (12.2785,4.8638) -- (12.3029,4.8707) -- (12.3273,4.8776) -- (12.3517,4.8845) -- (12.3761,4.8914) -- (12.4005,4.8983) -- (12.425,4.9052) -- (12.4494,4.9121) -- (12.4738,4.919) -- (12.4982,4.9259) -- (12.5226,4.9328) -- (12.547,4.9398) -- (12.5714,4.9467) -- (12.5958,4.9536) -- (12.6202,4.9605) -- (12.6446,4.9674) -- (12.669,4.9743) -- (12.6934,4.9812) -- (12.7179,4.9881) -- (12.7423,4.995) -- (12.7667,5.0019) -- (12.7911,5.0088) -- (12.8155,5.0157) -- (12.8399,5.0226) -- (12.8643,5.0295) -- (12.8887,5.0364) -- (12.9131,5.0433) -- (12.9375,5.0503) -- (12.9619,5.0572) -- (12.9863,5.0641) -- (13.0107,5.071) -- (13.0352,5.0779) -- (13.0596,5.0848) -- (13.084,5.0917) -- (13.1084,5.0986) -- (13.1328,5.1055) -- (13.1572,5.1124) -- (13.1816,5.1193) -- (13.206,5.1262) -- (13.2304,5.1331) -- (13.2548,5.14) -- (13.2792,5.1469) -- (13.3036,5.1539) -- (13.328,5.1608) -- (13.3525,5.1677) -- (13.3769,5.1746) -- (13.4013,5.1815) -- (13.4257,5.1884) -- (13.4501,5.1953) -- (13.4745,5.2022) -- (13.4989,5.2091) -- (13.5233,5.216) -- (13.5477,5.2229) -- (13.5721,5.2298) -- (13.5965,5.2367) -- (13.6209,5.2436) -- (13.6453,5.2505) -- (13.6698,5.2574) -- (13.6942,5.2644) -- (13.7186,5.2713) -- (13.743,5.2782) -- (13.7674,5.2851) -- (13.7918,5.292);
\filldraw[fill=MUCream,draw=MUCream,line width=0.42pt] (8.0458,3.8575) -- (8.1278,3.7055) -- (7.9638,3.7055) -- cycle;
\filldraw[fill=MUCream,draw=MUCream,line width=0.42pt] (10.4004,4.2419) -- (10.4824,4.0899) -- (10.3184,4.0899) -- cycle;
\filldraw[fill=MUCream,draw=MUCream,line width=0.42pt] (11.8859,4.8074) -- (11.9679,4.6554) -- (11.8039,4.6554) -- cycle;
\filldraw[fill=MUCream,draw=MUCream,line width=0.42pt] (13.7019,5.4388) -- (13.7839,5.2868) -- (13.6199,5.2868) -- cycle;
\draw[AxisDark,line width=1.55pt] (7.82,1) -- (13.92,1);
\draw[AxisDark,line width=1.55pt] (7.82,1) -- (7.82,6.15);
\draw[AxisDark,line width=1.25pt] (8.0458,1) -- (8.0458,0.89);
\node[anchor=north,font=\large] at (8.0458,0.82) {0.5B};
\draw[AxisDark,line width=1.25pt] (10.4004,1) -- (10.4004,0.89);
\node[anchor=north,font=\large] at (10.4004,0.82) {1.5B};
\draw[AxisDark,line width=1.25pt] (11.8859,1) -- (11.8859,0.89);
\node[anchor=north,font=\large] at (11.8859,0.82) {3B};
\draw[AxisDark,line width=1.25pt] (13.7019,1) -- (13.7019,0.89);
\node[anchor=north,font=\large] at (13.7019,0.82) {7B};
\draw[AxisDark,line width=1.25pt] (7.82,1) -- (7.71,1);
\draw[AxisDark,line width=1.25pt] (7.82,2.5642) -- (7.71,2.5642);
\draw[AxisDark,line width=1.25pt] (7.82,3.6739) -- (7.71,3.6739);
\draw[AxisDark,line width=1.25pt] (7.82,5.2381) -- (7.71,5.2381);
\node[anchor=north,font=\large] at (10.87,0.38) {Model size};
\node[anchor=south,font=\Large] at (10.87,6.4) {MASC};
\node[rotate=90,anchor=south,font=\large] at (0.22,3.575) {Raw score};
\node[anchor=center,font=\large] at (7.46,-0.42) {\begin{tabular}{@{}c@{\hspace{0.86cm}}c@{\hspace{0.86cm}}c@{}}\textcolor{ERougeOrange}{\rule[0.45ex]{0.58cm}{1.45pt}}\;E-ROUGE & \textcolor{PRougeBlue}{\rule[0.45ex]{0.58cm}{1.45pt}}\;P-ROUGE & \textcolor{MUCream}{\rule[0.45ex]{0.58cm}{1.45pt}}\;MU on $\mathcal{D}_r$\end{tabular}};
\end{tikzpicture}